%% file: neurips_2019.tex
\documentclass{article}

     \PassOptionsToPackage{numbers}{natbib}

	\usepackage[preprint]{neurips_2019}

\usepackage[utf8]{inputenc} 
\usepackage[T1]{fontenc}    
\usepackage{hyperref}       
\usepackage{url}            
\usepackage{booktabs}       
\usepackage{amsfonts}       
\usepackage{nicefrac}       
\usepackage{microtype}      
\usepackage{amsmath}
\usepackage{amssymb}		
\usepackage{amsthm}
\usepackage[tight-spacing=true]{siunitx}
\usepackage{enumitem}
\usepackage{multirow}
\usepackage{tabularx}
\usepackage{cancel}
\usepackage{thmtools}
\newtheorem{theorem}{Theorem}
\declaretheoremstyle[%
  spaceabove=-4pt,%
  spacebelow=2pt,%
  headfont=\normalfont\itshape,%
  postheadspace=1em,%
  qed=\qedsymbol%
]{mystyle}
\declaretheoremstyle[%
  spaceabove=2ex,%
  spacebelow=4pt,%
  headfont=\normalfont\itshape,%
  postheadspace=1em,
  qed=\qedsymbol%
]{tstyle}
\declaretheorem[name={Proof},style=mystyle,unnumbered,
]{Proof}

\declaretheorem[name={sketch},style=mystyle,unnumbered,
]{Sketch of Proof}

\usepackage{graphicx}
\usepackage{xcolor, soul}

\newcommand\etal[1]{~\textit{et al.}~\citep{#1}}
\makeatletter
\newif\ifarxiv\arxivfalse
\if@preprint\arxivtrue\fi
\makeatother

\title{Improved Adversarial Robustness by Reducing \\ Open Space Risk via Tent Activations}

\author{%
  Andras Rozsa \\
  Verisk AI, Verisk Analytics\\
  \texttt{andras.rozsa@verisk.com} \\
   \And
   Terrance E. Boult \\
   Vision and Security Technology Lab, UCCS \\
   \texttt{tboult@vast.uccs.edu} \\
}

\begin{document}

\maketitle

\begin{abstract}

Adversarial examples contain small perturbations that can remain imperceptible to human observers but alter the behavior of even the best performing deep learning models and yield incorrect outputs.
Since their discovery, adversarial examples have drawn significant attention in machine learning: researchers try to reveal the reasons for their existence and improve the robustness of machine learning models to adversarial perturbations.
The state-of-the-art defense is the computationally expensive and very time consuming adversarial training via projected gradient descent (PGD).
We hypothesize that adversarial attacks exploit the open space risk of classic monotonic activation functions.
This paper introduces the tent activation function with bounded open space risk and shows that tents make deep learning models more robust to adversarial attacks.
We demonstrate on the MNIST dataset that a classifier with tents yields an average accuracy of 91.8\% against six white-box adversarial attacks, which is more than 15 percentage points above the state of the art.
On the CIFAR-10 dataset, our approach improves the average accuracy against the six white-box adversarial attacks to 73.5\% from 41.8\% achieved by adversarial training via PGD.

\end{abstract}

\input{introduction}
\input{approach}
\input{motivation_and_formalization}
\input{experiments}
\input{conclusion}

{\small
\bibliographystyle{unsrt}
\bibliography{neurips_2019}
}

\ifarxiv
\newpage
\input{appendix1}
\input{appendix2}
\fi

\end{document}

%% file: introduction.tex
\section{Introduction}

Recent advances in computer vision and speech recognition are rapidly and increasingly bringing real-world applications of deep learning to life.
As these developments impact our lives, the security aspects of machine learning are critical.
While the best performing deep learning classifiers can achieve human-level performance on normal examples, Szegedy\etal{szegedy2013intriguing} demonstrated that an adversary can slightly perturb those examples to make deep learning models produce incorrect outputs.
These adversarial examples are often transferable between deep learning models, which means that adversaries can perform black-box or transfer attacks and do not need to access the targeted classifier in order to form such perturbations.
The phenomenon of adversarial examples has received rapidly growing attention from researchers.
A broad range of adversarial attacks have been proposed \citep{szegedy2013intriguing,goodfellow2014explaining, moosavi2016deepfool,carlini2017towards}.
In parallel, there is significant work on how to make deep learning models robust against such adversaries via model hardening or adversary detection\citep{szegedy2013intriguing,goodfellow2014explaining,kurakin2017adversarial,madry2018towards,tao2018attacks}.

As of today, among the several defense mechanisms proposed in the literature, adversarial training produces the best performance.
As the name suggests, this approach relies on the use of adversarially perturbed examples during training.  
Madry\etal{madry2018towards} demonstrated that adversarial training via their PGD produces classifiers that are robust against a wide range of adversarial attacks.
They argue that in order to withstand adversaries, robust classifiers require a significantly increased capacity. 

Handling unknown/unwanted inputs -- other than adversarially generated ones -- has become an important research topic known by many names including open set recognition \cite{openset14}, and out-of-distribution samples \cite{vyas2018out,lee2018simple,shalev2018out}.
It is also intimately related to anomaly detection \cite{racah2017extremeweather,menon2018loss,pidhorskyi2018generative}
and uncertainty estimation\cite{gal2016dropout,lakshminarayanan2017simple,sensoy2018evidential,malinin2018predictive}, since one can use such detectors or uncertainty estimates to reject unknown samples.
Each subarea has its own notation and related theories. 
The idea of managing open set risk has gained traction as a way of improving performance, see \cite{rattani2015open,zhang2016sparse,scherreik2016open,shi2018odn,dang2019open,coletta2019combining}. 
We follow \cite{openset14} as our goal is to obtain robust classifiers rather than detect or quantify unknown or unwanted inputs.
The intuition for open space risk is that when an input is far from the training samples, there is an increased risk that it is from an unknown distribution.
A key component of their theories is that recognition in the open world is risky and empirical, and open space risk need to be balanced.
We define activation functions that reduce open space risk.
By training with these tent activations, the network becomes more robust to adversarial perturbations. 

\begin{figure}
  \centering\includegraphics[width=1.0\linewidth]{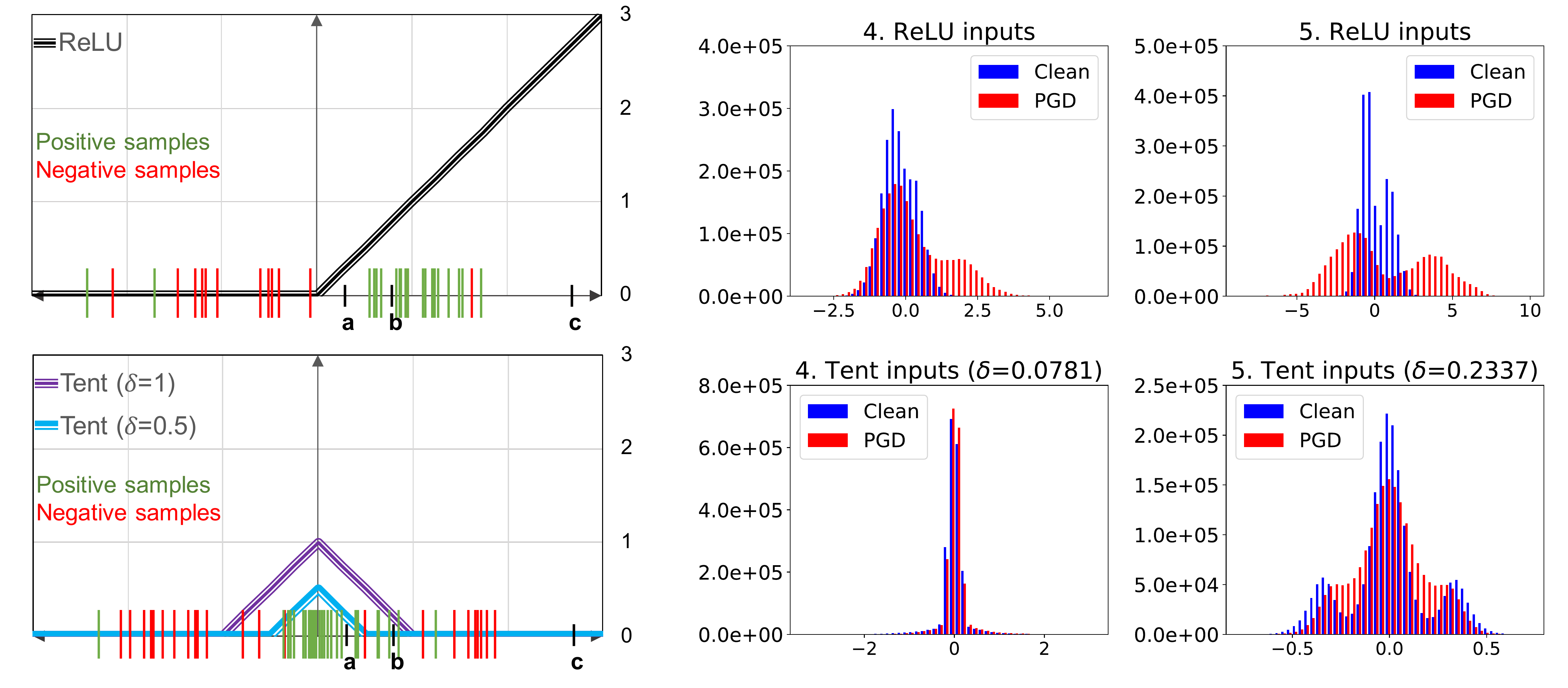}

\caption{The core hypothesis of this paper is that adversaries leverage the open space risk of activation functions to attack deep learning models.
While the rectified linear unit (ReLU\citep{nair2010rectified} upper left) has infinite open space risk, our novel tent activation (lower left) bounds open space risk.
Consider the visualization of hypothetical training samples shown as red (negative) or green (positive) lines at that layer; these represent layer specific features rather than classes, e.g., positive could be smooth curved strokes and negative could be anything else.
If inputs, such as adversarial examples, yield values that map to the right from the training data, e.g., point {\bf c}, the responses with ReLUs can be large even when inputs are far from training data.
With the tent activation, the same adversarial input {\bf c} would produce a zero response.
At the top on the right, using the 10k clean test examples and their adversarially perturbed counterparts via PGD we show the distributions of activation's inputs at the top two activation layers of a classifier trained on the MNIST dataset.
As we can see, for the regularly trained model with ReLUs, there is an increase in the absolute values of activation's inputs due to adversarial attacks. 
Below, we can see that with the classifier having tent activation functions there are smaller differences between these distributions and the range of inputs and outputs remains stable, which supports our open space risk hypothesis.
In Section~\ref{exp_section}, we show that tent activation functions increase robustness against adversaries.
}
  \label{fig:tent}
\end{figure}

One key insight of this paper is that non-linear activation functions used in deep neural networks (DNNs) -- e.g., ReLU, leaky ReLU, sigmoid, etc. -- act as local classifiers with some inputs being mapped to zero or below and others providing a positive response.
While people often think of the final layer of the network as ``the classifier," we argue that each activation is acting as a weak classifier and, therefore, the whole network acts as a cascade/ensemble of weak classifiers.
From that point of view, we can see that each activation function carries a potential open space risk and, hence, is a candidate for improvement in order to minimize open space risk.

We focus on training classifiers that learn deep features from normal inputs and extract internal feature representations that have bounded open space risk.
We do this with {\em tent activation functions}, see Figure~\ref{fig:tent}, which are computationally efficient, have bounded open space risk, and lead to classifiers that perform well on benign examples and against adversaries.

Since adversarial examples are formed by small, even imperceptible perturbations, at first glance, they might appear to be near the training samples and, hence, not in the open space.
But viewed from the actual classification results, we argue that they are in open space because the perturbed input is classified as an incorrect class whose training samples are far from the adversarial examples.

In this paper, we make the following contributions:
\begin{itemize}[noitemsep,topsep=0pt,parsep=0pt,partopsep=0pt]
\item We prove that standard monotonic activation functions -- including ReLU, leaky ReLU, sigmoid, or tanh -- all have unbounded open space risk.
\item We introduce the constraining tent activation function and prove that it has bounded open space risk. 
\item On the MNIST and CIFAR-10 dataset, we experimentally demonstrate that classifiers with tent activation functions possess significantly improved robustness against a broad range of adversaries.
Considering the overall performance against six adversarial white-box attacks, we present the new state of the art on the two dataset with classifiers having constraining tent activation functions.
\item This paper provides new insights and supporting experiments to suggest that reducing open space risk is one potential factor contributing to the improved robustness of classifiers.
\end{itemize}

%% file: approach.tex
\section{Approach}

There are several things to consider when we design our new element-wise activation function.
First and foremost, the activation function needs to have finite open space risk, which means limited range.
Second, we prefer all calculations associated with the implementation of the activation function to be simple, computationally efficient, and fast.
Last but not least, we must avoid the problem of exploding and vanishing gradients during gradient propagation.

\subsection{The tent activation function}

Consistent with the above requirements, we introduce the constraining tent activation function as 
\begin{equation}
  \label{eq:tent}
  f(x;\delta) = \max(0,\delta-\lvert x \rvert) .
\end{equation}
Our tent activation function is derived from the widely used ReLU: it is basically built from two ReLUs where a ReLU and a horizontally flipped ReLU share the same distance ($\delta$) from the origin. This distance, which also defines the height of the tent, is a learnable parameter. Tent activations are visualized in Figure~\ref{fig:tent}.

Similar to the ReLU, the tent activation function is differentiable almost everywhere -- it is simply not differentiable at a finite set of points. 

We define its partial derivative with respect to the input $x$, and parameter $\delta$ to be
\begin{equation}
{  \label{eq:tent_part_der_x}
  { \frac{\partial f}{\partial x}}=
    \begin{cases}
     -\text{sgn}(x) & \text{if~} 0<\lvert x \rvert \leq \delta \\
     0  & \text{otherwise,}\\
    \end{cases}
    }
     \ \hbox{ \quad and \quad  }  \ 
{
  { \frac{\partial f}{\partial \delta}}=
    \begin{cases}
     1 & \text{if~} \lvert x \rvert \leq \delta \\
     0  & \text{otherwise, }\\
    \end{cases}
    }
\          \hbox{ respectively.}  
\end{equation}

In summary, the introduced tent activation function is bounded with respect to both its inputs and outputs.
It is computationally efficient and its gradient can be defined everywhere.
However, it has zero response on both sides and, therefore, can easily become inactive to inputs depending on the size of the tent defined by $\delta$.
Consequently, this activation function is sensitive to initialization.

\subsection{Initialization and training}

The element-wise tent activation function has a single parameter that defines its size and, eventually, its active (non-zero) response range.
Obviously, if tents are initialized to be too small with respect to their inputs, most of them will fall into saturated regions and the machine learning model will not be able to learn due to the limited gradient propagation.
Similarly, if the distribution of inputs radically changes between iterations, training becomes more difficult, or even impossible, due to the saturated non-linearities.
Fortunately, we can take advantage of batch normalization \citep{ioffe2015batch} to normalize the inputs of our tent activation functions and keep the distributions of inputs more even during training.

We apply this constraining activation function on both convolutional and linear layers -- simply replacing each ReLU in the model with a combination of batch normalization feeding into a tent.
Based on heuristics we initialize tents with $\delta=1$.
Our implementation allows tents to be shared among channels and applied independently for each channel.
Optionally, sizes of tents can be limited by defining lower and/or upper bounds for the $\delta$ parameters. 

%% file: motivation_and_formalization.tex
\newcommand{\argmin}{\operatornamewithlimits{argmin}}
\section{Formal Open Set Risk Analysis}

First, we review some definitions from~\cite{openset14}.
Let $f$ be a measurable recognition function over an input space $x\in X$ where $f_y(x)>0$ for recognition of the class $y$ of interest and $f_y(x)=0$ when $y$ is not recognized.
To formalize open space, \cite{openset14} defines $r$ to be a problem-specific parameter and defines open space as the space greater than a distance $r$ from any known training sample $x_i\in {\cal K}, i = 1 \ldots N$.
The known space ${B_r(x_i)}$ is a closed ball with radius $r$ centered around training sample $x_i$. 
They define $S_o$ to be the smallest ball that includes all ${B_r(x_i)}$.
Finally, the authors define open space ${\cal O} = X - \bigcup_{i \in N}B_r(x_i)$, {\em Open Space Risk} $R_{\cal O}(f)$ and {\em Open Set Risk} ${\cal R}(f)$ as
\begin{equation}
R_{{\cal O}}(f) = \frac{\int_{{\cal O}} f_y(x) dx}{\int_{S_o} f_y(x) dx}
\quad \hbox { and } \quad
{\cal R}(f) = R_{\cal O}(f) + \lambda_r R_{\cal E}(f({\cal V; \cal K})) ,
\label{eq:opensetrisk}
\end{equation}
where $R_{\cal E}$ is the empirical error, $\cal V$ is validation data, and $\lambda_r$ is a constant balancing empirical and open space risk.
They define optimal open set recognition as finding $f$ that minimizes ${\cal R}$, which requires $f$ to have finite risk.

We can now analyze the risk of classic monotonic activations and our tent activation.
Without loss of generality, we consider non-decreasing activation function $a(x)$ with $a(x)\le 0$ for $x \le 0$ and $a(x) > 0$ when $x > 0$. 
These are the activations widely being used today including hard and noisy versions of sigmoid and tanh activations \cite{gulcehre2016noisy}, the more recently developed ReLU \cite{nair2010rectified} and its variants such as leaky \cite{maas2013rectifier} and parametric forms \cite{he2015deep}, as well as the self-normalizing network activation function \cite{klambauer2017self}. 
There are, of course, other activations that are not monotonic, such as Gaussian (aka RBF), cosine, or even conic activations.
While there are some specific applications of these, they are uncommon because they are difficult to train, expensive to use, and in many experimental settings, have been shown not to be as effective as tanh or simple ReLUs \cite{karlik2011performance,goodfellow2014explaining,bircanouglu2018comparison}.
For simplicity, we let $a^+(x) = max(0,a(x))$ be our measurable recognition function for the monotonic activation. 

\begin{theorem}
Classic monotonic activation functions act as weak classifiers with unbounded open space risk and unbounded open set risk. 
\end{theorem}
\begin{proof}
Sketch of the many proofs: Without loss of generality, let us assume the activation is monotonically non-decreasing as above.
Looking at Eq.~\ref{eq:opensetrisk}, we see that for non-degenerate training, the denominator is non-zero, so the risk is dominated by the numerator. In this case letting $v$ be the maximum value seen in training and $r$ is the ball size in the definition of open space,
we have $\lim_{z\rightarrow\infty} \int_{v+r}^{z} a^+(x)dx \rightarrow\infty $.
Since this integral in the numerator of Eq.~\ref{eq:opensetrisk} is not finite, the open space risk $R_{\cal O}$ is not bounded, and it follows that the open set risk $\cal R$ in Eq.~\ref{eq:opensetrisk} is also unbounded.
\end{proof}

We now formalize some of the properties of the tent activation function.   
\begin{theorem}
Tent activation functions act as weak classifiers which manage open space risk.
\end{theorem}

\begin{Proof}
To simplify the proof, initially consider a normalized tent activation, ${\hat f}(x;\delta)= f(x;\delta)/\delta^2$, which has unit area and, hence, can be viewed as a probability distribution.
Normalized tent activations are compact probability abating models in the sense of \cite{openset14} with $x_i=0$ and $\tau=\delta$.
Hence, by Theorem 1 of that paper, the tent activation manages open space risk.
To complete the proof, note that scaling back to original tent does not impact if a value is zero, finite or not.
\end{Proof}

The astute reader might be thinking that the actual open set risk of any network activation is practically bounded.
All layers are finite combinations with finite network weights.
With actual input data, one can estimate a more precise open space risk bound using interval analysis similar to the provable robustness analysis in \cite{wang2018efficient}.
Even without such interval analysis, our experiments show that learned $\delta$ parameters are usually small with ($\delta < 1$).
In contrast, the maximum potential value for any monotonic activation, given even finite inputs, is still quite large.
Thus, our tent activation functions provide a very significant reduction in open space risk.

%% file: experiments.tex
\newcommand\p{\phantom{0}}
\newcommand\z[1]{\textbf{#1}}

\section{Experiments}
\label{exp_section}

We evaluate DNNs with tent activation functions and test their impact on robustness.
To improve reproducibility, we start with a detailed description of our experiments.
Researchers can extend this work by building upon our publicly available code base. \footnote{\url{anonymized_url}} 

\subsection{Experimental Setup}

We perform experiments on both the MNIST \citep{lecun1998mnist} and CIFAR-10 \citep{krizhevsky2009learning} dataset using the tent activation function.
For comparison, we train regular deep learning models utilizing ReLUs, and we perform the adversarial training of Madry\etal{madry2018towards} which, as of today, produces the most robust classifiers.
We use training and validation sets (55k/5k images for MNIST and 45k/5k for CIFAR-10), and after the corresponding training completes, we pick the best model based on its performance on the validation set.
All evaluations occur on the test sets (10k examples for both MNIST and CIFAR-10).

For each classifier, we report the accuracies on clean test samples and on test images that were perturbed by the different adversaries. 
To evaluate and compare the robustness of the trained models, we use a broad range of adversaries
\ifarxiv
(for visualized examples see Appendix~A)
\hspace{-4pt}
\fi.
The selected white-box adversarial example generation techniques are the following:

\begin{itemize}[noitemsep,topsep=0pt,parsep=0pt,partopsep=0pt]
\item the fast gradient sign method (FGSM) of Goodfellow\etal{goodfellow2014explaining},
\item the basic iterative method (BIM) by Kurakin\etal{kurakin2017adversarial}, 
\item the projected gradient descent (PGD) approach of Madry\etal{madry2018towards},
\item the DeepFool (DF) method by Moosavi-Dezfooli\etal{moosavi2016deepfool},
\item the $l_2$ and $l_\infty$ versions of Carlini and Wagner attacks \citep{carlini2017towards}, in short, CW $l_2$ and CW $l_\infty$.
\end{itemize}

We also conduct black-box attacks by testing the transferability of adversarial examples formed on regular models to classifiers that we obtain via adversarial training or by applying the tent activation function.
To perform these adversarial attacks on trained models and obtain robust classifiers via adversarial training, we utilize the adversarial robustness toolbox (ART v0.10.0) \citep{art2018}.
Unless otherwise specified, we use the default parameters of ART.

\textbf{MNIST.}
For the MNIST dataset, we used the architecture from Carlini and Wagner in \citep{carlini2017towards}, which was introduced by Papernot\etal{papernot2016distillation}.
For simplicity, we refer to this classifier as MNIST-Net.
This model consists of four convolutional layers -- they have $3\times3$ kernels with 32, 32, 64, and 64 filters, respectively -- and two linear layers with a width of $200$.
Each of these six layers is followed by ReLUs; spatial dimensions are reduced after the second and fourth convolutional layers via max-pooling.
For regularization, dropout is applied after both linear layers with dropout ratio of $0.5$.
We seek to perform a fair comparison between regular classifiers utilizing ReLUs and the modified architectures replacing the six ReLUs with our constraining tent activation functions. Since we use tents coupled with batch normalization, we modify the architecture and apply batch normalization to the inputs of ReLUs as well.
Throughout our experiments on this architecture, we use the Adam optimizer with a fixed $0.001$ learning rate, the batch size is set to $100$, and we train classifiers for 40 epochs.
Image pixels are scaled to be in the range $\left[0,1\right]$; no further data augmentation is used.
To squeeze the size of tents and, thus, reduce open space risk, we apply weight-decay to $\delta$ parameters.
For adversarial training, we use the same settings as Madry\etal{madry2018towards}.
Each training sample in every mini-batch is perturbed by the PGD approach with the following parameters: the number of iterations is limited to 40, the $l_\infty$ of perturbations is maximized to $0.3$, and the step-size for each iteration is set to $0.1$ in pixel values.
For evaluation, we use the default settings of ART for applicable adversaries: $l_\infty$ limit for perturbations is $0.3$, and step-size is $0.1$, by default.
Note that only four adversaries (FGSM, BIM, PGD, and CW $l_\infty$) have $l_\infty$ upper bounds for perturbations. Similarly, the step-size can only be defined for two adversarial example generation algorithms (BIM and PGD).
Last but not least, note that the default settings of ART v0.10.0 define $100$ iterations for BIM, PGD, and DF adversaries.
 
\textbf{CIFAR-10.} For our experiments on the CIFAR-10 dataset, we use two variants of the wide residual networks (WRNs) \citep{zagoruyko2016wide} that were designed to improve over the original residual networks (ResNets) \citep{he2015deep}.
While ResNets require incrementally increased depths for every performance improvement, wide residual networks can achieve the same level of performance with shallower but wider architectures.
Wide residual networks are denoted as WRN-$d$-$k$ where $d$ defines their depth, and $k$ specifies their width.
With $k=1$, the WRN has the same width as a standard ResNet.
To evaluate and compare the effect of tent activation functions and adversarial training on classifiers having different capacities, we experiment with WRN-$28$-$1$ and WRN-$28$-$10$ models.
We obtain WRNs with tent activation functions by simply replacing the ReLUs of residual units with tents.
To train regular WRNs and for adversarial training, we use stochastic gradient descent (SGD) with momentum of $0.9$ for 200 epochs with a weight-decay of $0.0005$.
We decrease the initial learning rate of $0.1$ by a factor of $0.2$ after 60, 120, and 160 epochs, respectively.
For WRNs with tents, we use the Adam optimizer, which is generally inferior to SGD on residual networks but yields smoother training experience with our constraining activation functions.
Furthermore, for regularization, we only use weight-decay on $\delta$ parameters of tents.
We use an initial learning rate of $0.001$ and apply the same, previously described policy to adjust it.
The batch size is $100$ for all experiments on WRNs, and we apply random cropping and horizontal flipping on images having pixels in $\left[0,1\right]$.
Again, for adversarial training, we follow Madry\etal{madry2018towards} as closely as possible.
Each training sample is perturbed by PGD using the following parameters: the adversary is limited to 7 iterations, the maximum $l_\infty$ of perturbations is $0.032$ ($=8.16$ in $\left[0,255\right]$ pixel-value range) while the step-size for iterations is set to $0.008$.
As before, for evaluation, we use the default settings of ART, however, we change the $l_\infty$ limit for perturbations to $0.032$ and the step-size to $0.008$ for applicable adversaries.

The following settings for the constraining tent activation functions are shared throughout all our experiments.
To avoid significantly increasing the capacity of classifiers compared to models having ReLUs, we use tents with shared $\delta$ parameters among input channels.
Tents are initialized with $\delta=1$, and we apply both a lower and upper bound to $\delta$ parameters with $0.05$ and $1.0$, respectively.
While the lower bound is useful to prevent tents from becoming inactive, with the upper bound we aim at avoiding inflating tents that would compromise our goals to constrain inputs and outputs of these activation functions.

\subsection{Results}

\begin{table}
  \caption{MNIST: Performance of MNIST-Net classifiers on clean test samples and against different adversaries conducting white-box and black-box attacks. We show accuracies (\%) for the original architecture (org) \citep{carlini2017towards}, its modified version with added batch normalization layers (bn), the adversarially trained network (adv) with batch normalization via PGD \citep{madry2018towards}, and several classifiers with tent activation functions trained with different weight-decays on $\delta$ parameters as indicated by the numbers in parentheses. At the bottom, we show the performance of black-box attacks where the perturbed examples originate from the regularly trained network (bn).}
  \label{results_mnist}
  \centering
  \setlength{\tabcolsep}{7pt}
  \begin{tabular}{p{0pt}lccccccc}
    \toprule
 	&Model				 	& Clean		& FGSM		& BIM		& PGD		& DF			& CW $l_2$	& CW $l_\infty$ \\
    \midrule
	\multirow{8}{*}{\rotatebox[origin=c]{90}{\small{White-box}}} &    
    
	MNIST-Net org		 	& 99.38		& 27.89		&\p0.50		&\p0.50		&\p5.60		&\p5.28		& 37.19 \\
	&MNIST-Net bn			&\z{99.50}	& 11.82		&\p0.39		&\p0.39		&\p7.49		&\p5.82		& 40.00 \\
	&MNIST-Net adv		 	& 99.36		&\z{96.13}	&\z{91.46}		&\z{91.46}	&\p2.09		& 85.17		& 94.58 \\
	&MNIST-Net tent (0)	  	& 99.29		& 68.13		&\p2.39		&\p2.39		&11.78		& 41.25		& 57.43 \\
	&MNIST-Net tent (0.01)	& 99.46		& 55.52		&\p1.97		&\p1.97		& 29.78		& 61.93		& 75.80 \\
	&MNIST-Net tent (0.10)	& 99.15		& 80.34		& 75.20	    & 75.20		& 88.57		& 94.81		& 97.09 \\
	&MNIST-Net tent (0.12)	& 99.20		& 87.56		& 88.38	    & 88.37		& 92.92		&\z{95.81}	&\z{97.77} \\
	&MNIST-Net tent (0.14)	& 98.95		& 64.05		& 67.59	    & 67.59		&\z{95.02}	& 94.90		& 96.89 \\
    \midrule
	\multirow{5}{*}{\rotatebox[origin=c]{90}{\small{Black-box}}} &    
	MNIST-Net org			&\z{99.38}	& 58.26		& 67.37		& 67.48		& 98.77		& 94.15		& 97.01 \\
	&MNIST-Net adv			& 99.36		&\z{97.86}	& \z{98.55}		&\z{98.57}	&\z{99.12}	&\z{98.32}	&\z{98.66} \\
	&MNIST-Net tent (0)		& 99.29		& 70.23		& 72.37		& 72.25		& 98.36		& 94.64		& 96.48 \\
	&MNIST-Net tent (0.10)	& 99.15		& 70.94		& 80.06		& 79.95		& 98.74		& 98.04		& 98.24 \\
	&MNIST-Net tent (0.12)	& 99.20		& 73.90		& 83.82	& 83.88		& 98.82		& 98.11		& 98.25 \\
    \bottomrule
  \end{tabular}
\end{table}

On both the MNIST and CIFAR-10 dataset, we focus on discovering and analyzing the effects of tent activation functions on robustness rather than trying to find the optimal hyper-parameters to achieve the best possible performance.
Of course, improving robustness of classifiers while, in parallel, dramatically decreasing their performance on clean examples is undesirable.
Therefore, we present trained models where the application of our constraining activation functions yields comparable accuracies on the clean test dataset with the adversarial training of Madry\etal{madry2018towards}.
After discussing our experimental results on both datasets separately, we highlight our key findings related to tent activation functions.

\textbf{MNIST -- WHITE-BOX.}
The results of white-box attacks are summarized in Table~\ref{results_mnist}.
To address the potential impact of our added batch normalization layers, we evaluate and compare with both the original one (MNIST-Net org) and the modified classifier (MNIST-Net bn).
Considering their overall lack of robustness to different adversaries, there is no relevant difference between the two.
Next, we analyze the adversarially trained classifier via PGD \citep{madry2018towards}, (MNIST-Net adv), which is our benchmark.
This computationally expensive approach yields significantly improved robustness against five of the six adversaries.
Interestingly, the accuracy on test samples that we perturbed via DeepFool (DF) is lower than the regular training produces.
Finally, we can evaluate the impact of tent activation functions on robustness.
As a vanilla approach, we have a classifier MNIST-Net tent (0) with tents where we do not penalize the model for having larger tents, which means greater open space risk. 
To achieve this, we do not apply weight-decay to the $\delta$ parameters.
While MNIST-Net tent (0) shows some signs of improvement over regular training, it is not comparable to adversarial training.
However, as we increasingly squeeze tents and constrain both the input and output regions of these activation functions via weight-decay, the performance of the trained classifiers against all adversaries improve.
While we cannot match the benchmark on FGSM, BIM, and PGD, our tent activation functions yield more general robustness.
The trained model obtained with weight-decay of $0.12$ (MNIST-Net tent (0.12) in Table~\ref{results_mnist}) achieves the best overall performance among classifiers having tents.
We find that further increased penalties on sizes of tents lead to performance degradation on benign examples and against some adversaries as well.

\textbf{MNIST -- BLACK-BOX.}
We evaluate the robustness of classifiers against black-box attacks and compare their performance with the original network and the benchmark in Table~\ref{results_mnist}.
For these attacks, we use the perturbed examples formed on the regular classifier (denoted as bn).
Note that while adversaries do not have direct access to the targeted model, the source network shares the architecture and training set with the targeted classifiers (one might call these ``grey-box'' attacks).
Similar to white-box attacks, we find that adversarial training clearly outperforms the classifiers we have trained with tent activation functions against FGSM, BIM, and PGD adversaries.
Considering the other adversaries, networks with tents can reach the same level as adversarial training.  

\begin{table}
  \caption{CIFAR-10: Performance of the 28-layer Wide-ResNets on clean test samples and against different adversaries conducting white-box and black-box attacks. We present accuracies (\%) for two variants of the original architecture \citep{zagoruyko2016wide} with widen-factor 1 and 10, the adversarially trained models (adv) via PGD \citep{madry2018towards}, and several classifiers with tent activation functions that we trained with different weight-decays on $\delta$ parameters as indicated by the numbers in parentheses. We show the performance of black-box attacks where the perturbed examples originate from the regularly trained networks denoted as WRN-28-1 and WRN-28-10, respectively.}
  \label{results_wrn}
  \centering
  \setlength{\tabcolsep}{6.5pt}
  \begin{tabular}{p{0pt}lccccccc}
    \toprule
 	&Model				 	& Clean		& FGSM		& BIM		& PGD		& DF			& CW $l_2$	& CW $l_\infty$ \\
    \midrule
	\multirow{6}{*}{\rotatebox[origin=c]{90}{\small{White-box}}} &    
	WRN-28-1					&\z{93.05}	& 17.92		&\p4.39		&\p4.39		& 40.05		&\p4.53		&\p4.59 \\
	&WRN-28-1 adv			& 73.26		& 59.09		& 53.38		& 53.38		& 17.26		& 13.91		& 55.80 \\
	&WRN-28-1 tent (0.01)	& 85.16		& 66.36		& 17.97		& 17.97		&\z{83.71}	& 13.25		& 62.49 \\
	&WRN-28-1 tent (0.02)	& 81.71		&\z{79.11}	& 74.18		& 74.18		& 81.40		& 44.77		&\z{80.78} \\
	&WRN-28-1 tent (0.03)	& 81.26		& 78.39		&\z{75.04}	&\z{75.04}	& 80.88		&\z{46.68}	& 79.67 \\
	&WRN-28-1 tent (0.04)	& 78.71		& 75.04		& 71.87		& 71.87		& 78.19		& 44.22		& 76.39 \\
    \midrule
	\multirow{5}{*}{\rotatebox[origin=c]{90}{\small{Black-box}}} &WRN-28-1 adv				& 73.26		& 72.27		& 72.67		& 72.67		& 73.17		& 73.15		& 73.14 \\
	&WRN-28-1 tent (0.01)	&\z{85.16}	& 65.88		& 73.86		& 73.86		&\z{84.44}	&\z{84.53}	&\z{84.08} \\
	&WRN-28-1 tent (0.02)	& 81.71		& 77.04		& 79.42		& 79.42		& 81.39		& 81.54		& 81.45 \\
	&WRN-28-1 tent (0.03)	& 81.26		&\z{78.03}	&\z{79.47}	&\z{79.47}	& 81.01		& 81.04		& 81.08 \\
	&WRN-28-1 tent (0.04)	& 78.71		& 75.35		& 76.91		& 76.91		& 78.62		& 78.68		& 78.42 \\
    \midrule
	\multirow{6}{*}{\rotatebox[origin=c]{90}{\small{White-box}}} &WRN-28-10				&\z{95.58}	& 42.37		&\p2.98		&\p2.98		& 25.52		&\p3.58		& \p4.11 \\
	&WRN-28-10 adv			& 86.72		& 64.95		& 51.56		& 51.56		&\p9.78		& 17.91		& 54.91 \\
	&WRN-28-10 tent (0.001)	& 87.87		& 74.84		& 22.98		& 23.06		& 65.39		&\p9.03		& 72.36 \\
	&WRN-28-10 tent (0.002)	& 87.52		& 77.22		& 54.78		& 54.73		& 69.62		& 12.44		& 71.38 \\
	&WRN-28-10 tent (0.003)	& 86.04		&\z{84.58}	& 83.27		& 83.27		&\z{83.59}	&\z{24.13}	&\z{82.73} \\
	&WRN-28-10 tent (0.004)	& 85.27		& 83.62		&\z{84.70}	&\z{84.63}	& 83.41		& 23.65		& 81.23 \\
    \midrule
	\multirow{5}{*}{\rotatebox[origin=c]{90}{\small{Black-box}}} &WRN-28-10 adv			& 86.72		&\z{85.03}	&\z{85.87}	&\z{85.86}	& 86.63		& 86.63		&\z{86.53} \\
	&WRN-28-10 tent (0.001)	&\z{87.87}	& 74.15		& 80.77		& 80.76		&\z{86.72}	&\z{86.96}	& 85.99 \\
	&WRN-28-10 tent (0.002)	& 87.52		& 73.26		& 80.70		& 80.72		& 86.50		& 86.47		& 85.80 \\
	&WRN-28-10 tent (0.003)	& 86.04		& 82.66		& 84.53		& 84.52		& 85.70		& 85.63		& 85.55 \\
	&WRN-28-10 tent (0.004)	& 85.27		& 82.48		& 83.86		& 83.86		& 85.09		& 84.92		& 84.84 \\
    \bottomrule
  \end{tabular}
\end{table}

\textbf{CIFAR-10 -- WHITE-BOX.}
The results of our experiments using two (WRNs) architectures on the CIFAR-10 dataset are displayed in Table~\ref{results_wrn}.
We train regular classifiers, adversarially trained networks via PGD \citep{madry2018towards} to obtain the benchmark, and a few models with tent activation functions using different weight-decays on $\delta$ parameters.
For both WRN architectures, we observe that adversarial training is less effective compared to the results achieved on the MNIST dataset.
Against white-box DeepFool attacks, compared to the regularly trained model, the robustness is reduced by adversarial training. 
Also, the accuracy on benign test examples for the lower capacity network (WRN-28-1) is dramatically decreased by adversarial training. 
Madry\etal{madry2018towards} argued that this behavior is driven by the fact that there is a trade-off between accuracy and robustness.
Contrarily, the narrow classifiers with tents (WRN-28-1 tent) better maintain their performance on benign examples and, in parallel, outperform the benchmark by a large margin.
On the large capacity architecture (WRN-28-10) adversarial training delivers approximately the same level of robustness to adversaries as on the narrow network, but it performs better on clean examples.
Compared to the benchmark, classifiers with tents demonstrate increased robustness on this higher capacity architecture except for the $l_2$ version of the Carlini and Wagner attack (CW $l_2$).

\textbf{CIFAR-10 -- BLACK-BOX.}
On this dataset, we adapt the black-box protocol we described for the MNIST dataset.
We simply take the examples adversaries perturbed on the regular classifier while conducting white-box attacks and test those images on the targeted trained models. 
Considering the results of black-box attacks listed in Table~\ref{results_wrn}, we can observe that classifiers with tent activation functions outperform the benchmark on the low capacity classifier and match its performance on the higher capacity model.

\textbf{SUMMARY.}
Training DNNs with our novel tent activation functions provides greatly improved adversarial robustness with no additional computational cost compared to regular training.
Note that the adversarial training via PGD of Madry\etal{madry2018towards} requires extra forward/backward passes over regular training.
On the MNIST and CIFAR-10 dataset there are $40$ and $7$ additional passes, respectively, to obtain the perturbed examples for training.
While experimenting with low and high capacity WRNs, we have found that low capacity DNNs with tent activation functions better maintain their performance on clean test examples than adversarial training. 

%% file: conclusion.tex
\section{Conclusion}

One of our primary hypothesis, strongly supported by the theory and experimental results of this paper, is that open space risk is a major contributing factor to the success of adversarial attacks.
We proved that all monotonic activation functions have unbounded risk.
Our experimental results are consistent with our hypothesis; the standard activation functions allow adversaries to exploit the open space risk with adversarially perturbed images, see the right side of Figure~\ref{fig:tent}
\ifarxiv \ and Appendix~B\fi.
We revealed that the prior state of the art, the PGD-based adversarial training of Madry\etal{madry2018towards}, is effective against all but one of the selected adversaries. 
Notably, the computationally expensive approach yields decreased robustness against the DeepFool attack compared to regular training.

The open space risk model herein explains many issues for adversarial example generation.
Gradients -- real or approximated -- provide an understanding of the influential pixels in the image with respect to both the output of the classifier and the response of each activation.
For each activation, adversaries can hill climb to push those pixels farther into their high response regions.
With a monotonic activation, there is always a direction in which one can move to increase the response.
With a tent activation that is no longer the case as there is a limit beyond which increasing the input to the layer will cause the activation output to decrease and eventually become zero.

This paper introduced the tent activation function that is efficiently computed and which we proved has bounded open space risk. 
We emphasize that, unlike adversarial training, the application of our novel tent activation function does not increase the computational cost over regular training.
We showed that networks using tents yield significantly improved adversarial robustness compared to the state of the art.
On MNIST, the network having tent activation functions -- trained with a weight-decay of 0.12 on $\delta$ parameters of tents -- produce an average accuracy of 91.8\% on examples perturbed via six adversarial white-box attacks, compared to 76.8\% achieved by the state-of-the-art adversarial training via PGD.
On the CIFAR-10 dataset, our approach -- WRN-28-10 having tents and trained with a weight-decay of 0.004 on $\delta$ parameters -- yields an average accuracy of 73.5\% against the six adversarial white-box attacks compared to 41.8\% of the state of the art.
Combined, these significant improvements further support our hypothesis about the importance of open space risk, which we believe has great potential in helping to address the problem of adversarial robustness.


%% file: appendix1.tex
\section*{APPENDIX A: Visualization of Adversarial Examples}
\label{sec:appendixa}

In the paper, similar to Madry\etal{madry2018towards}, we focus on measuring adversarial robustness quantitatively as we evaluate the robustness of classifiers by calculating accuracies on test samples that have been perturbed by different adversaries.
Here, we present a few of those examples to highlight the qualitative properties of various adversaries.
Also, we show the difference between regular classifiers and the best models that we have trained with the novel tent activation function with respect to adversarial perturbations.

We use the same six adversaries using the adversarial robustness toolbox (ART v0.10.0) \citep{art2018} as before to conduct white-box attacks.
These adversarial example generation techniques are:
\begin{itemize}[noitemsep,topsep=0pt,parsep=0pt,partopsep=0pt]
\item the fast gradient sign method (FGSM) of Goodfellow\etal{goodfellow2014explaining},
\item the basic iterative method (BIM) by Kurakin\etal{kurakin2017adversarial},
\item the projected gradient descent (PGD) approach of Madry\etal{madry2018towards},
\item DeepFool (DF) method by Moosavi-Dezfooli\etal{moosavi2016deepfool},
\item the $l_2$ and $l_\infty$ versions of Carlini and Wagner attacks \citep{carlini2017towards}, in short, CW $l_2$ and CW $l_\infty$.
\end{itemize}

Note that after presenting the clean test image, we follow the same order for adversaries when we display the adversarially perturbed examples in Figures~\ref{fig:mnist_imgs} and \ref{fig:cifar10_imgs}. 

\begin{figure}

    \centering \includegraphics[width=.475\linewidth]{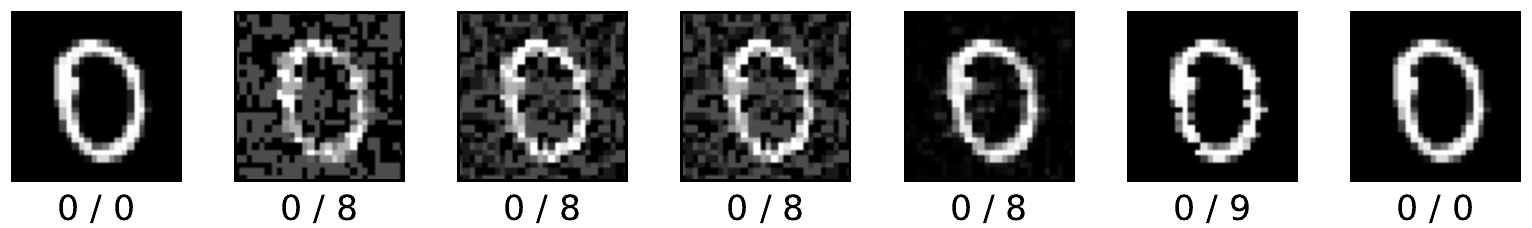} \hfill 	\hfill 			{\includegraphics[width=.475\linewidth]{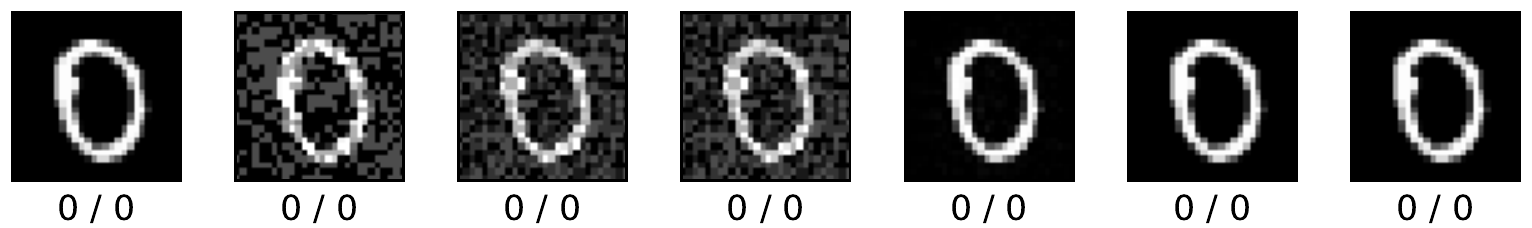}} \\

    \centering \includegraphics[width=.475\linewidth]{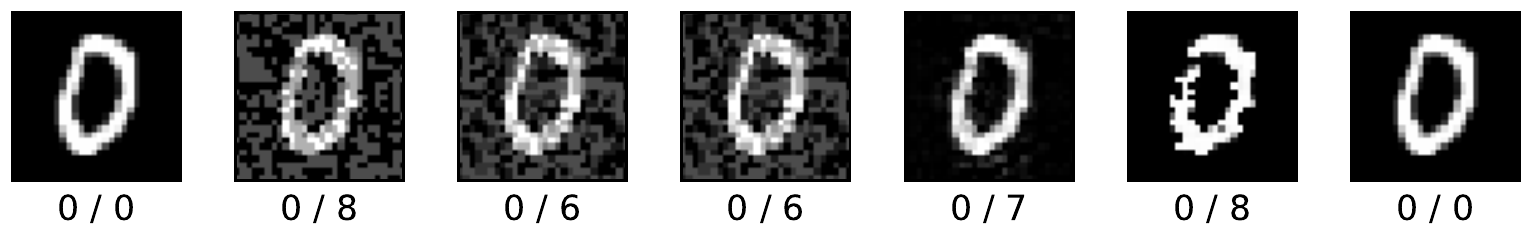} \hfill 	\hfill 			{\includegraphics[width=.475\linewidth]{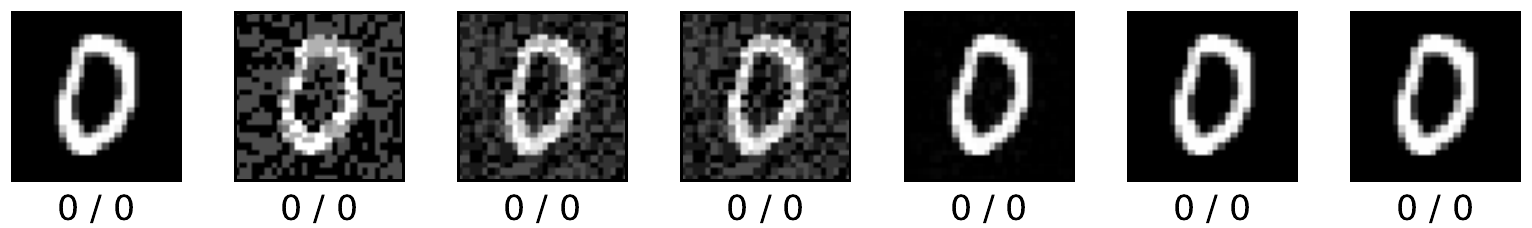}} \\

    \centering \includegraphics[width=.475\linewidth]{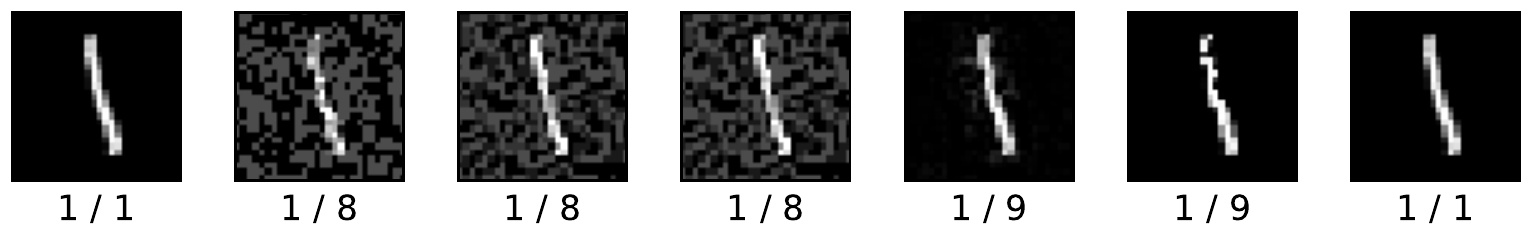} \hfill 	\hfill 			{\includegraphics[width=.475\linewidth]{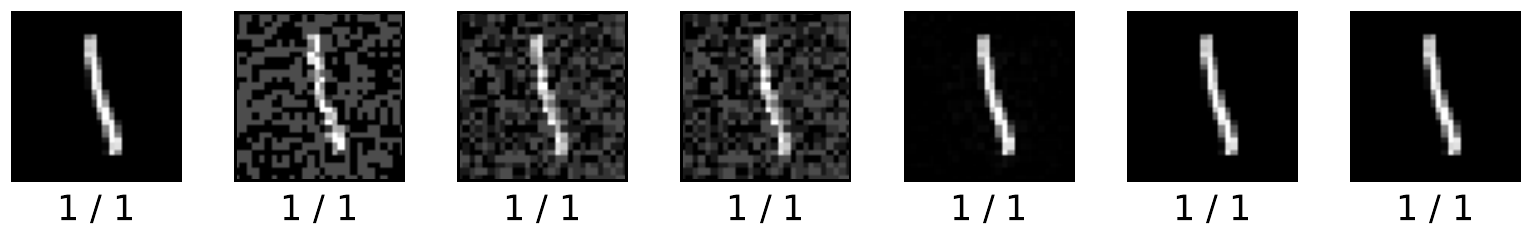}} \\

    \centering \includegraphics[width=.475\linewidth]{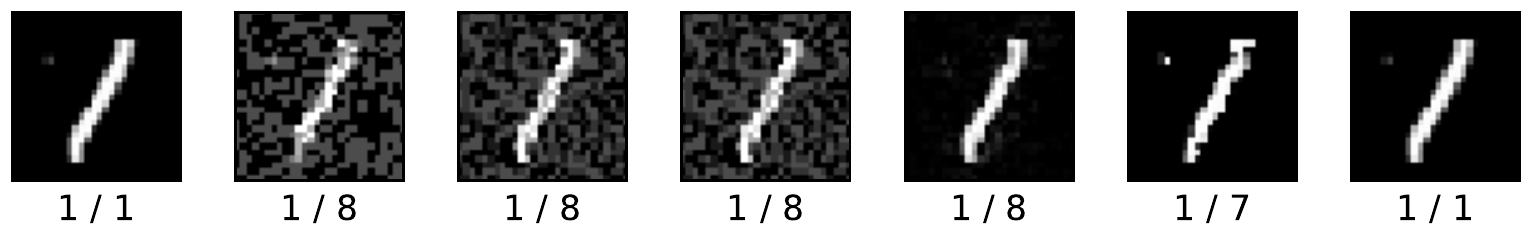} \hfill 	\hfill 			{\includegraphics[width=.475\linewidth]{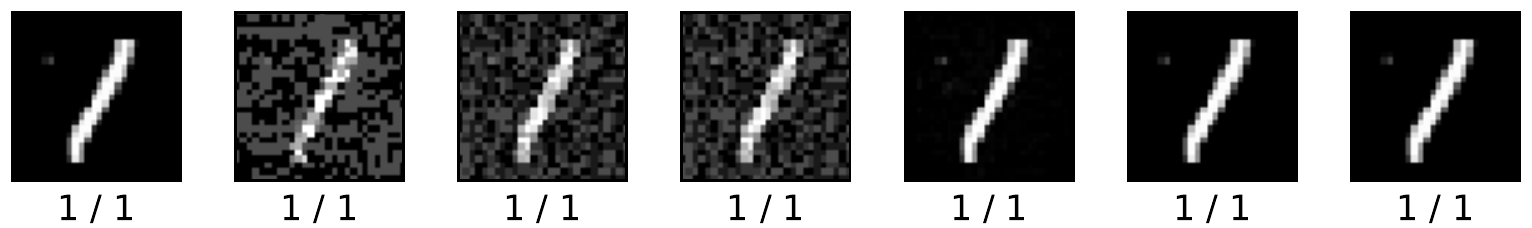}} \\

    \centering \includegraphics[width=.475\linewidth]{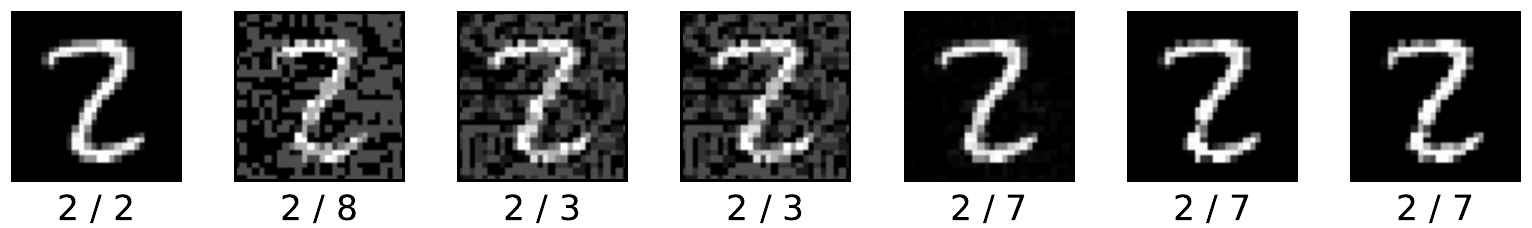} \hfill 	\hfill 			{\includegraphics[width=.475\linewidth]{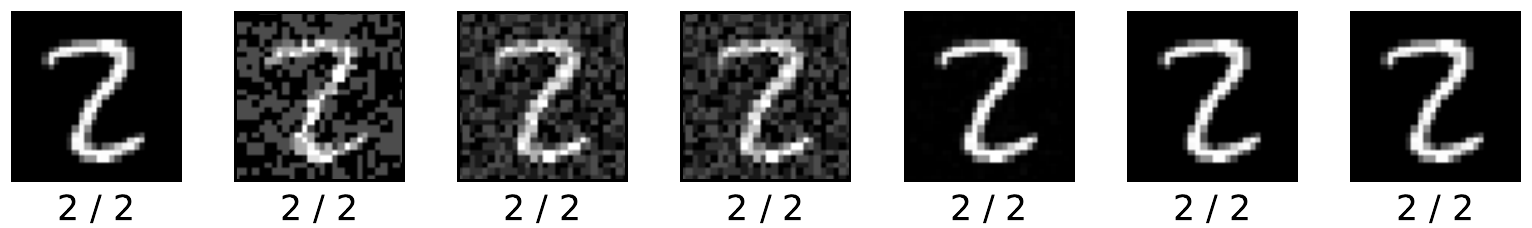}} \\

    \centering \includegraphics[width=.475\linewidth]{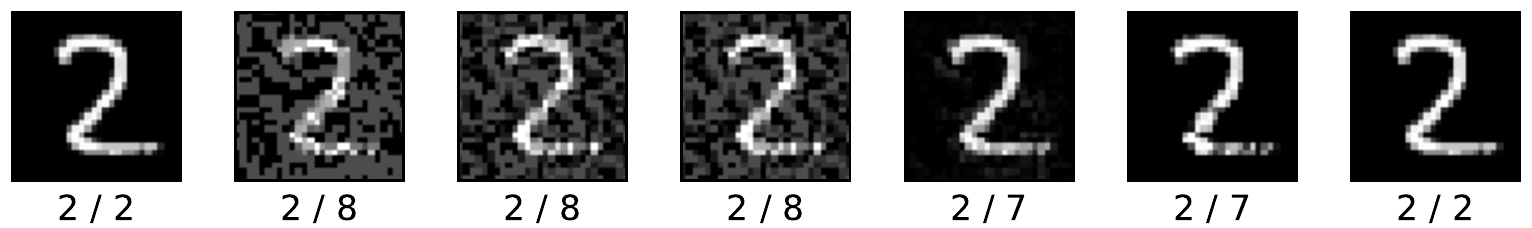} \hfill 	\hfill 			{\includegraphics[width=.475\linewidth]{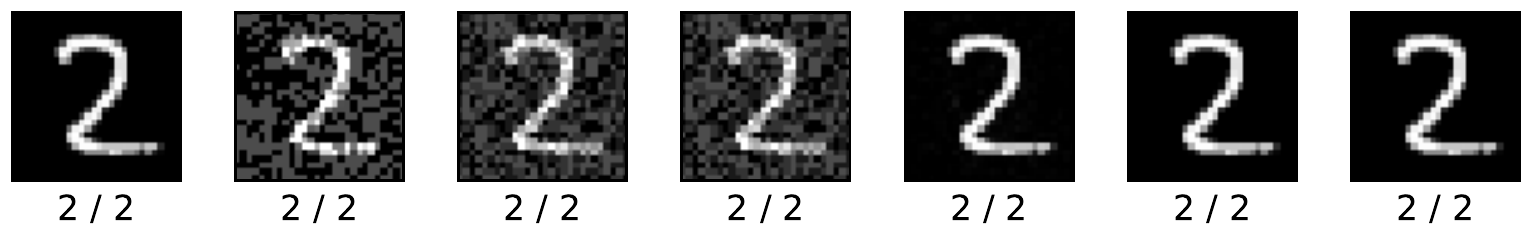}} \\

    \centering \includegraphics[width=.475\linewidth]{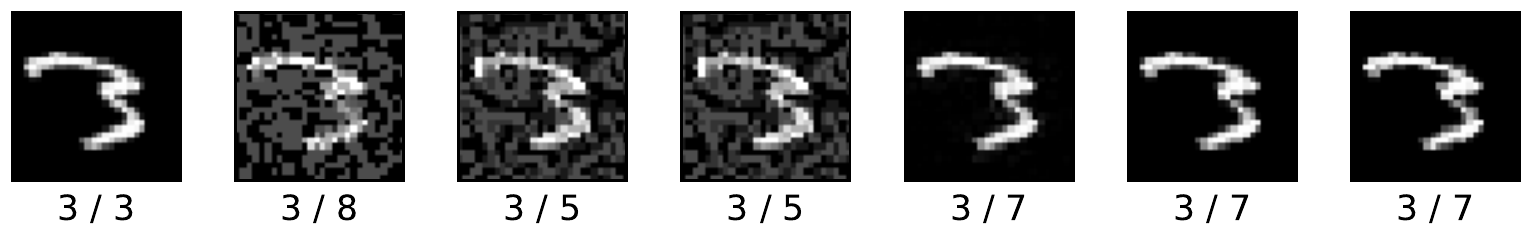} \hfill 	\hfill 			{\includegraphics[width=.475\linewidth]{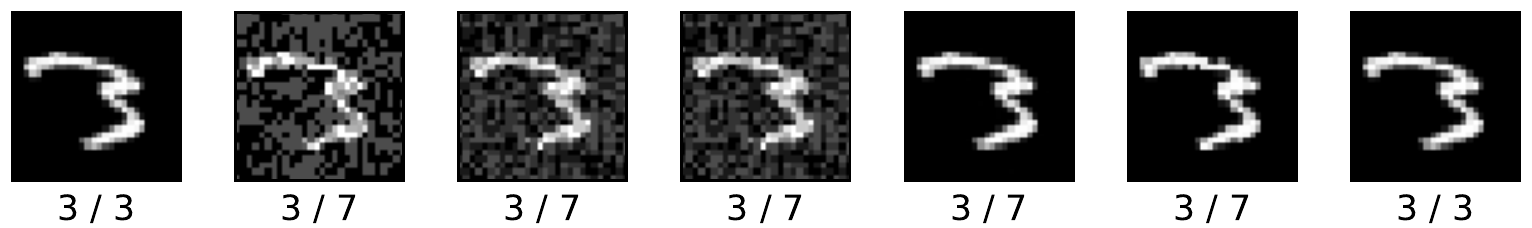}} \\
    
    \centering \includegraphics[width=.475\linewidth]{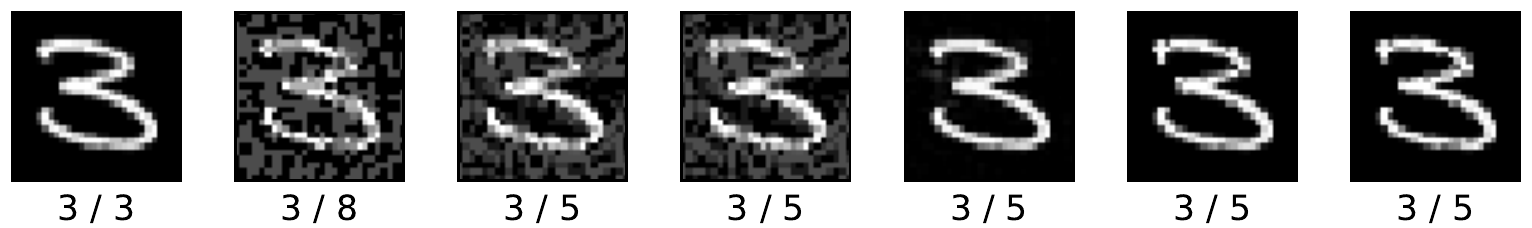} \hfill 	\hfill 			{\includegraphics[width=.475\linewidth]{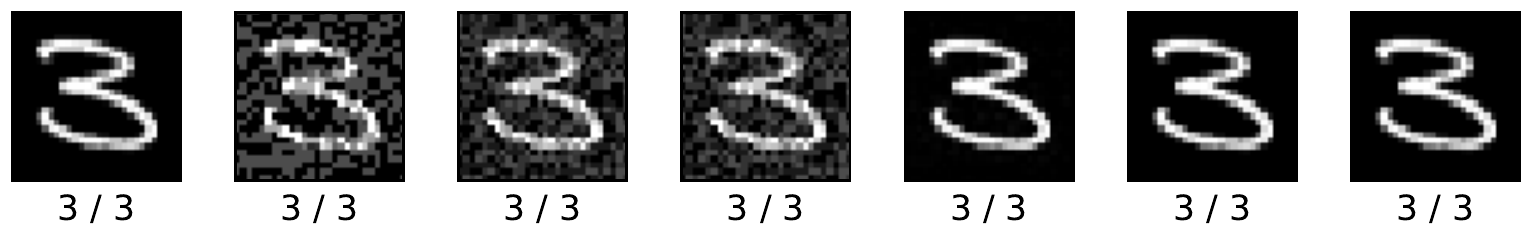}} \\

    \centering \includegraphics[width=.475\linewidth]{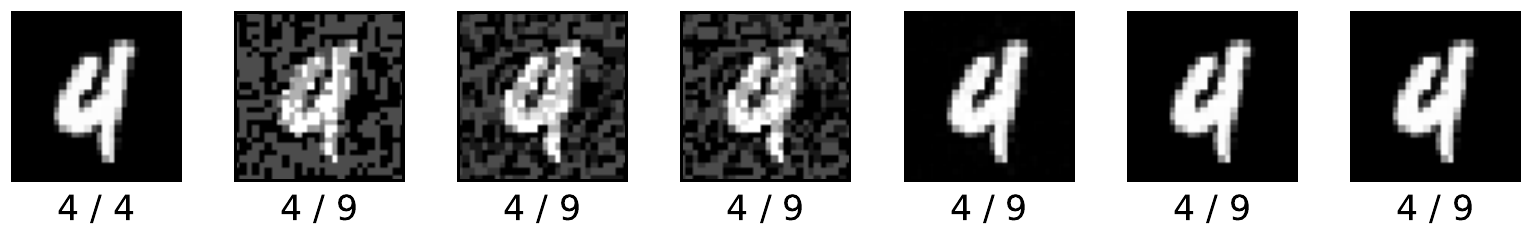} \hfill 	\hfill 			{\includegraphics[width=.475\linewidth]{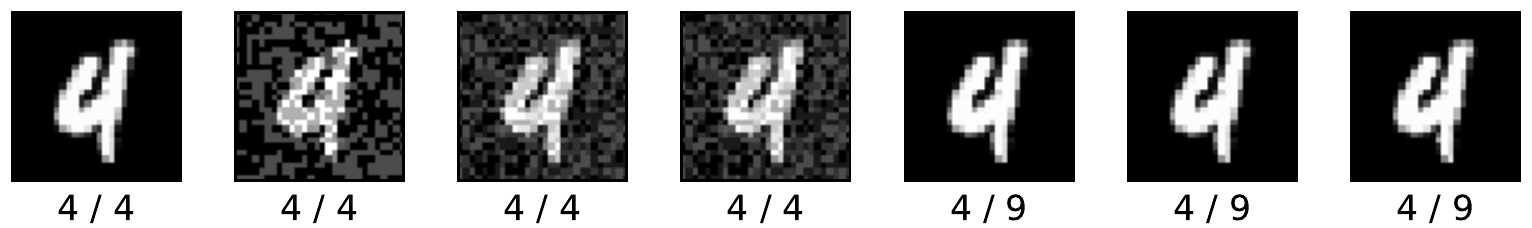}} \\

    \centering \includegraphics[width=.475\linewidth]{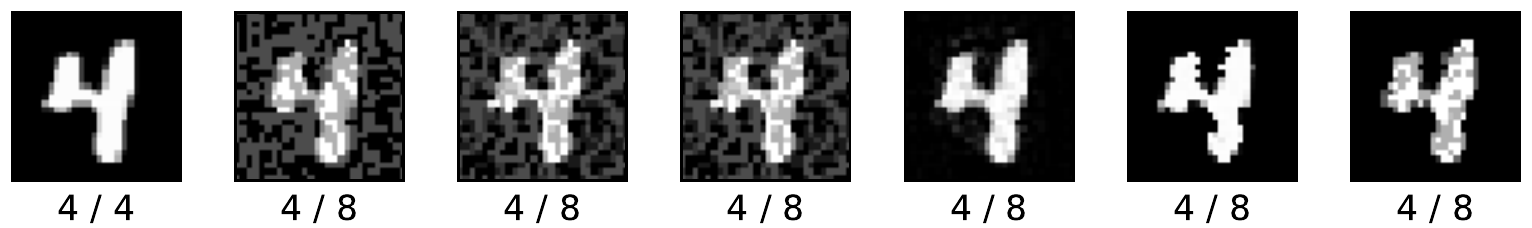} \hfill 	\hfill 			{\includegraphics[width=.475\linewidth]{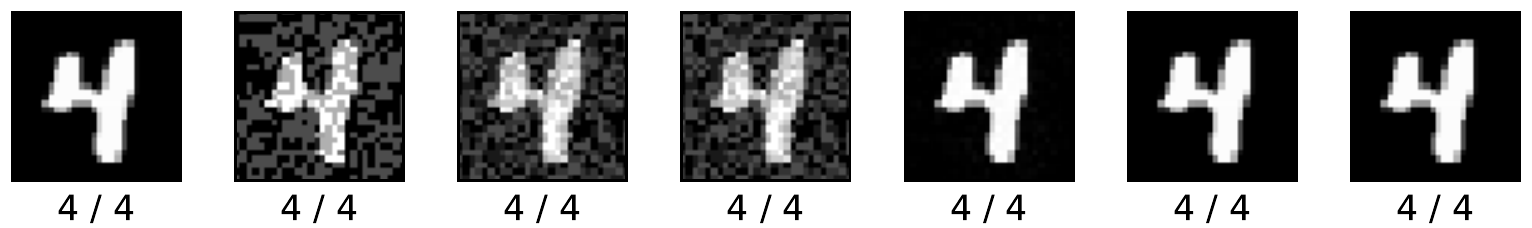}} \\

    \centering \includegraphics[width=.475\linewidth]{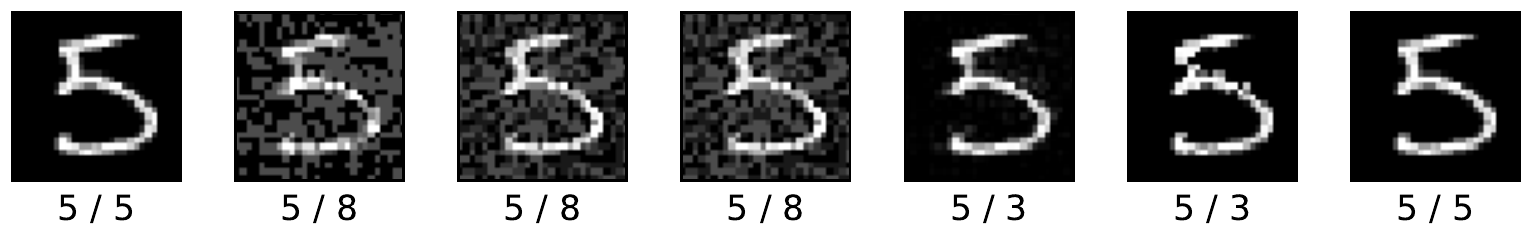} \hfill 	\hfill 			{\includegraphics[width=.475\linewidth]{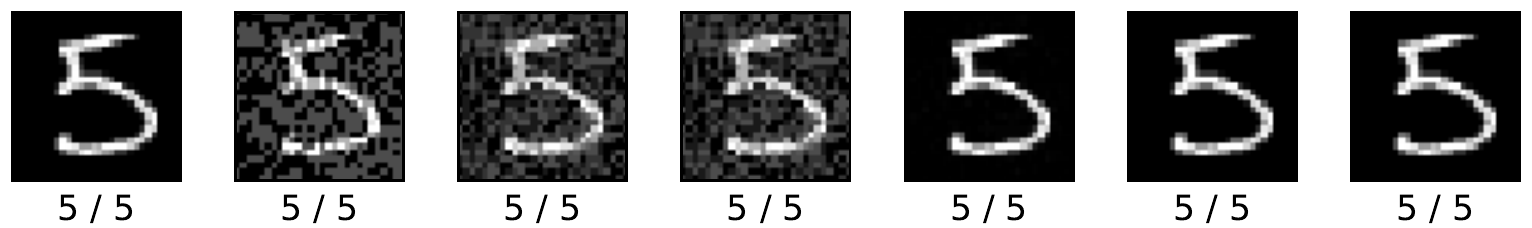}} \\

    \centering \includegraphics[width=.475\linewidth]{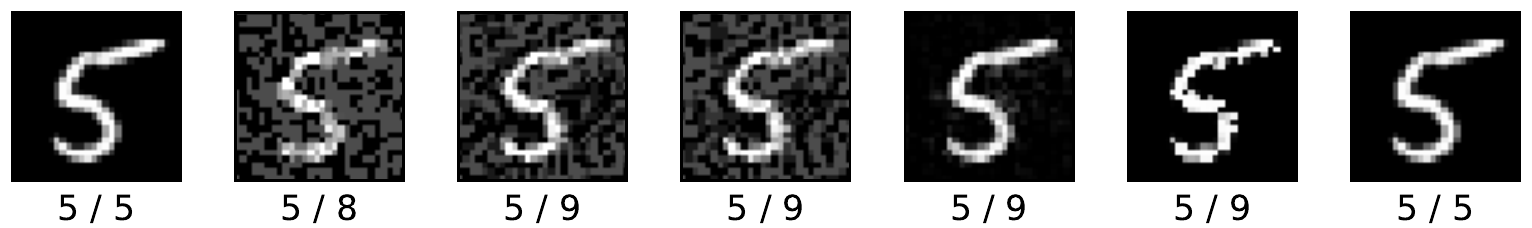} \hfill 	\hfill 			{\includegraphics[width=.475\linewidth]{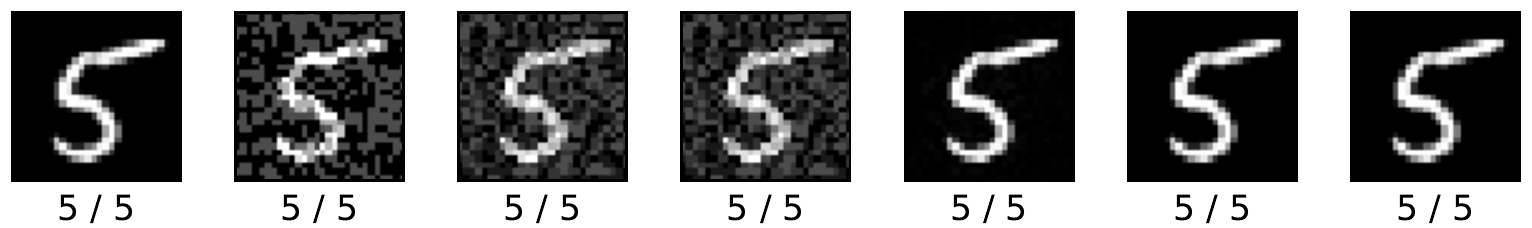}} \\

    \centering \includegraphics[width=.475\linewidth]{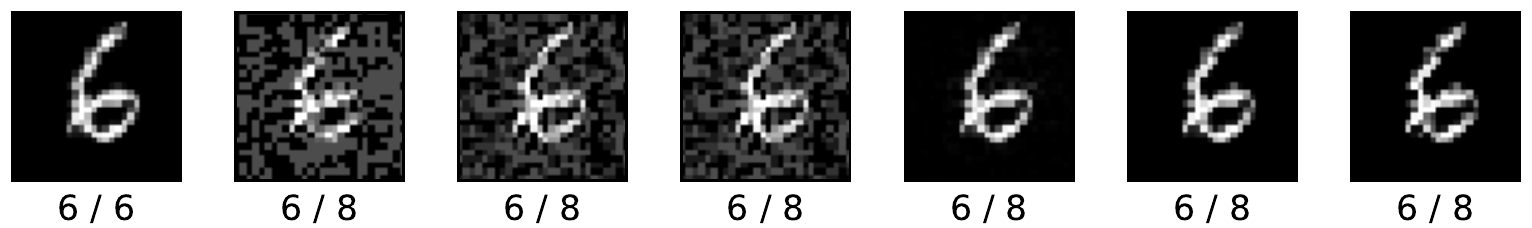} \hfill 	\hfill 			{\includegraphics[width=.475\linewidth]{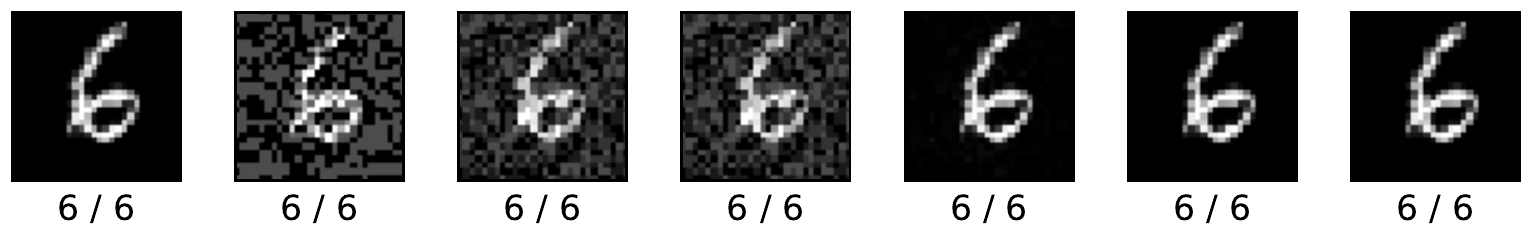}} \\

    \centering \includegraphics[width=.475\linewidth]{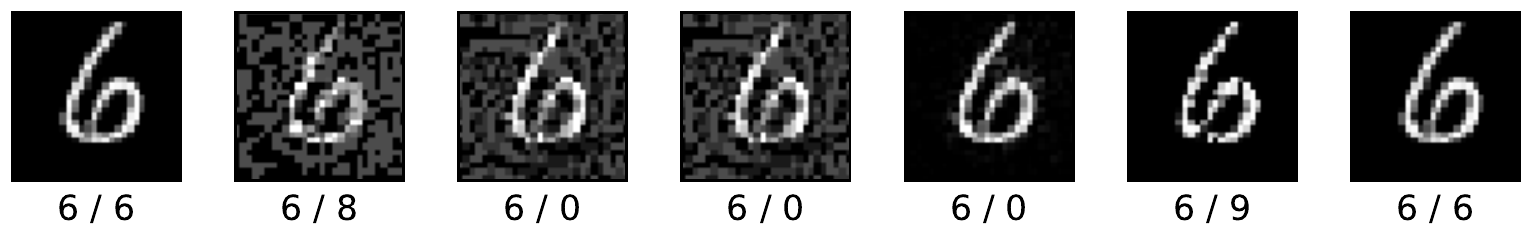} \hfill 	\hfill 			{\includegraphics[width=.475\linewidth]{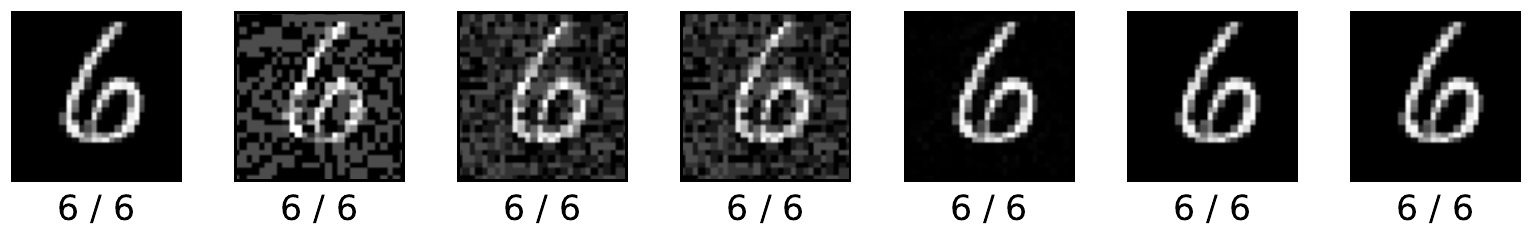}} \\

    \centering \includegraphics[width=.475\linewidth]{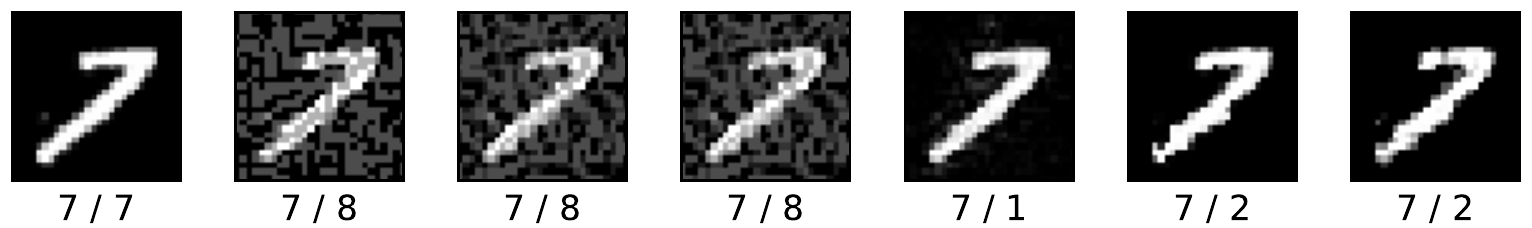} \hfill 	\hfill 		{\includegraphics[width=.475\linewidth]{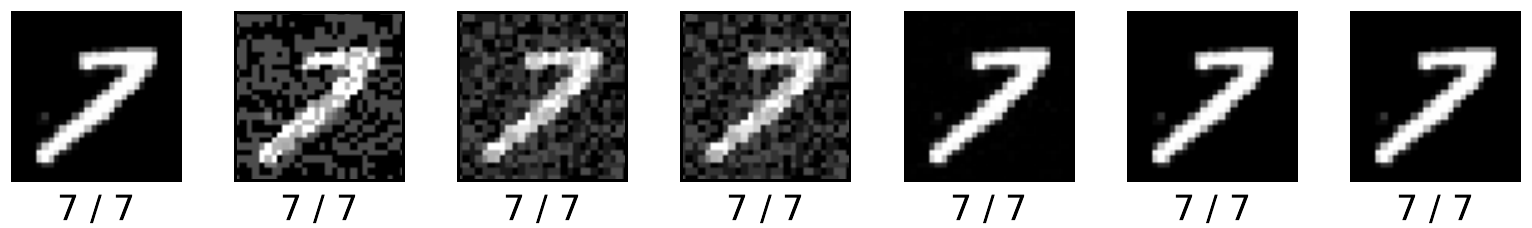}} \\

    \centering \includegraphics[width=.475\linewidth]{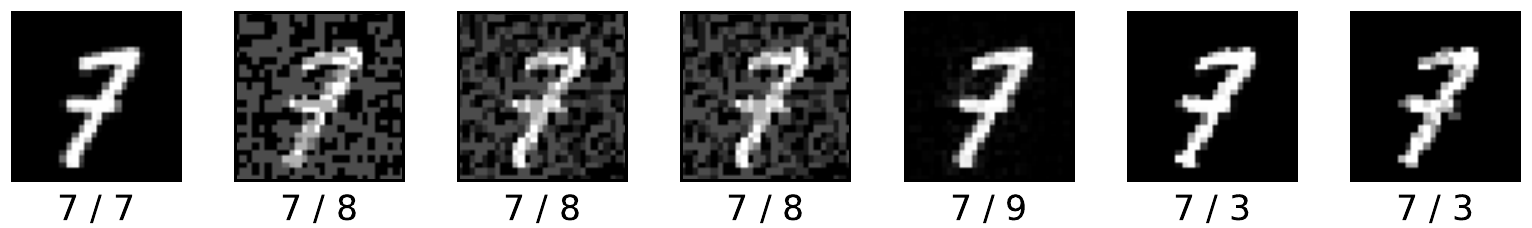} \hfill 	\hfill 			{\includegraphics[width=.475\linewidth]{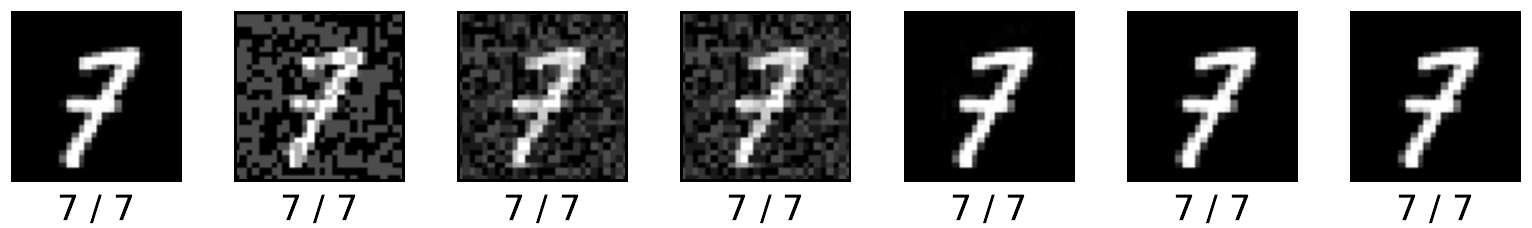}} \\

    \centering \includegraphics[width=.475\linewidth]{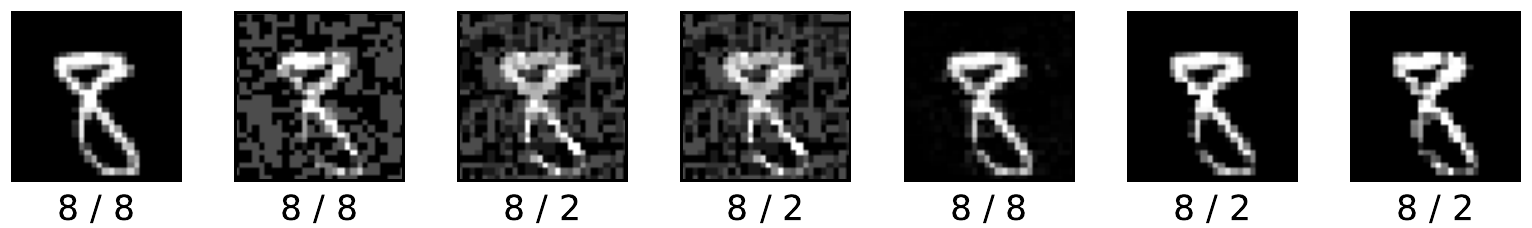} \hfill 	\hfill 			{\includegraphics[width=.475\linewidth]{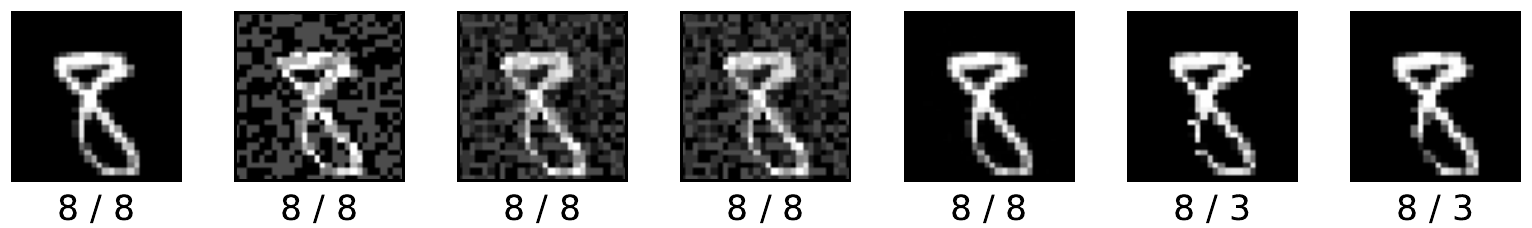}} \\
    
    \centering \includegraphics[width=.475\linewidth]{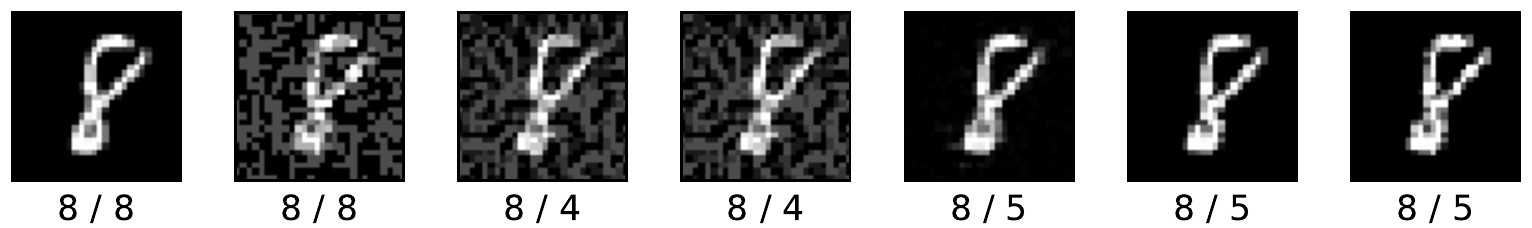} \hfill 	\hfill 			{\includegraphics[width=.475\linewidth]{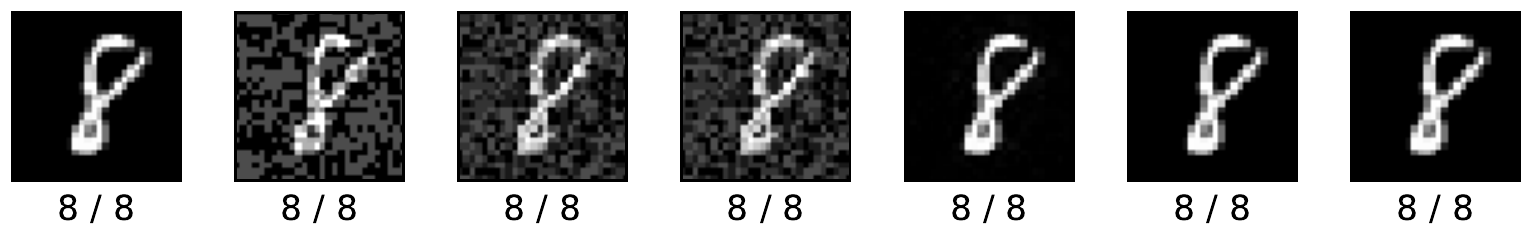}} \\

    \centering \includegraphics[width=.475\linewidth]{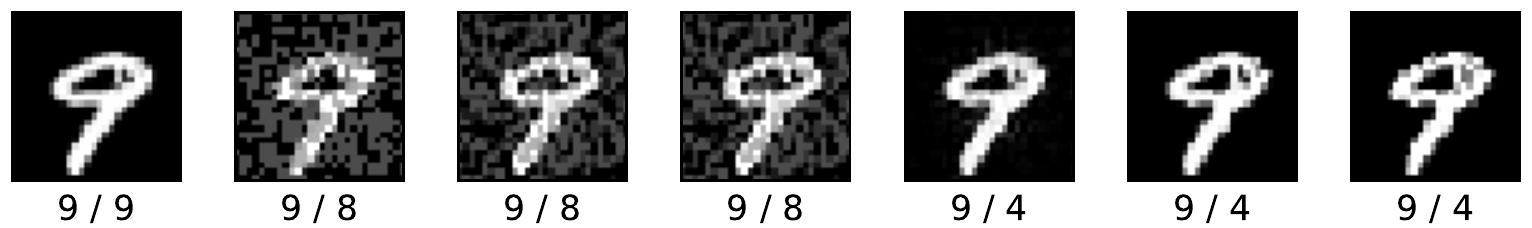} \hfill 	\hfill 			{\includegraphics[width=.475\linewidth]{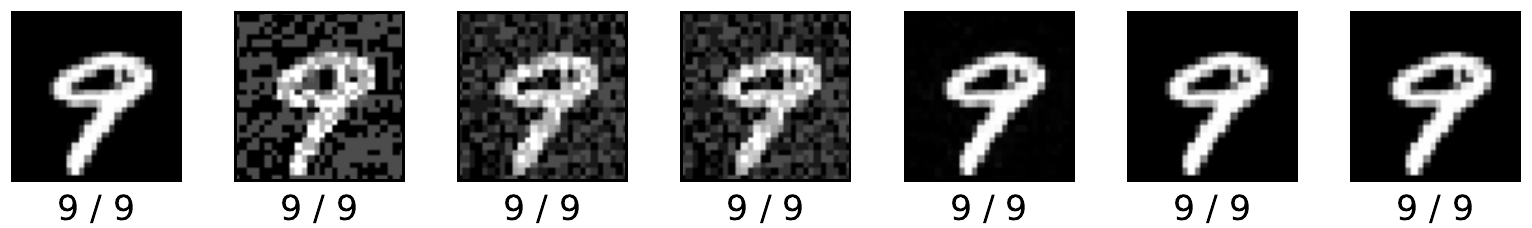}} \\

    \centering \includegraphics[width=.475\linewidth]{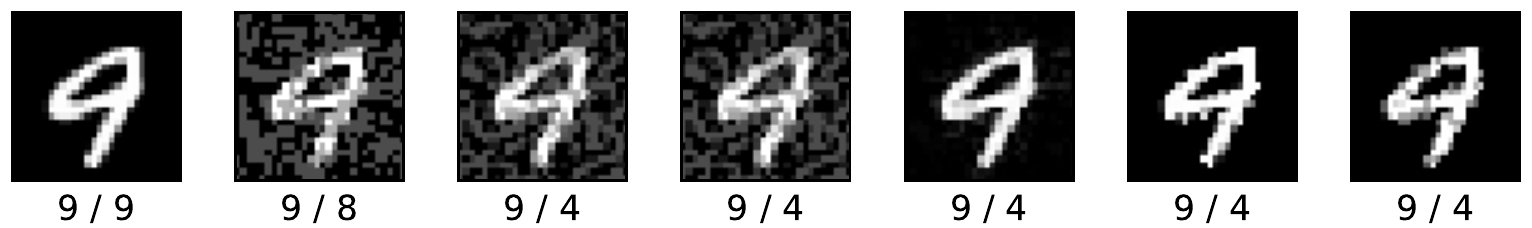} \hfill 	\hfill 		{\includegraphics[width=.475\linewidth]{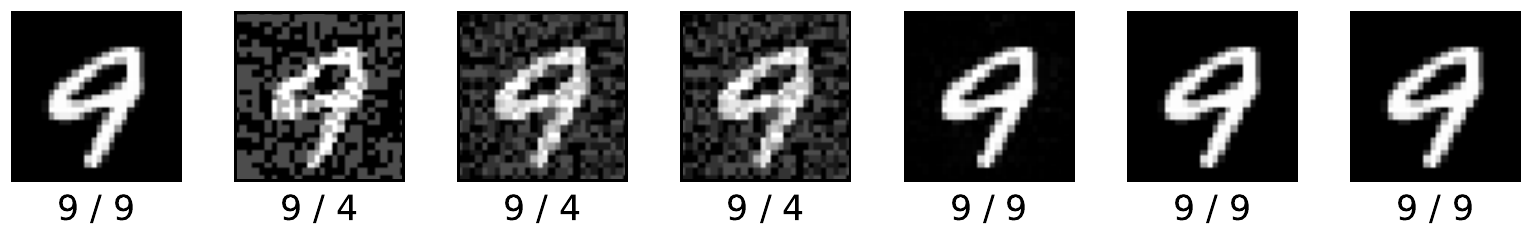}}

  \caption{Adversaries on MNIST. We show clean samples followed by their perturbed versions that we obtained with ART using FGSM, BIM, PGD, DeepFool, and the two Carlini and Wagner adversaries. For each example, we show the true and predicted class labels. In the left column, we display samples that originate from the regular network with ReLUs, denoted as MNIST-Net bn. On the right, we present images generated on the classifier with tent activations, MNIST-Net tent (0.12).}
  \label{fig:mnist_imgs}
\end{figure}

\subsection*{MNIST}

For the MNIST dataset, we present some examples for the regular classifier containing ReLUs and batch normalization -- we referred to this as MNIST-Net bn in Table 1 of the paper -- and for the model we have trained with a weight-decay of $0.12$ on $\delta$ parameters of tent activations that we denoted as MNIST-Net tent (0.12).

The clean test samples and their corresponding perturbed variants are shown in Figure~\ref{fig:mnist_imgs}.
We can observe that, in general, FGSM, BIM, and PGD produce very strong perturbations.
In other words, these methods appear to produce the worst adversarial perturbations qualitatively.
Since these perturbations are highly visible, one might not even consider them adversarial, by nature.

\subsection*{CIFAR-10}

Similar to MNIST, we compare adversarial examples qualitatively on two classifiers that we have trained on the CIFAR-10 dataset using the wide residual network (WRN) architecture: the regular WRN-28-10 with ReLUs and WRN-28-10 tent (0.004) from Table 2 of the paper.
The latter one has tent activations; a weight-decay of $0.004$ was applied to $\delta$ parameters to reduce open space risk via the constraining tent activation functions and, as the quantitative evaluation highlighted, improve the adversarial robustness of the classifier. 

The original test images and their corresponding perturbed variants are shown in Figure~\ref{fig:cifar10_imgs}.
We can observe that the quantitatively best performing adversary -- $l_2$-optimized version of Carlini and Wagner (CW $l_2$) attack -- produces some strong, quite visible perturbations on the classifier having tent activations.
This phenomenon demonstrates that the quantitatively improved adversarial robustness is further extended by qualitative enhancements; the adversary can still form perturbations that yield incorrect decisions by the classifier, but such perturbations need to be stronger and, hence, become more perceptible.

\section*{Summary}

The visual evaluation of adversarial examples presented in this section highlights that quantitative evaluation of adversarial robustness solely based upon accuracies may not be sufficient.
As we have seen, the application of tent activation functions yields classifiers that require adversaries to form stronger perturbations which are more perceptible to human observers.
Note that measuring perceptible/imperceptible changes with respect to the human visual system is not trivial; $l_p$ norms are not applicable \citep{sabour2016adversarial}.

\begin{figure}

    \centering \includegraphics[width=.475\linewidth]{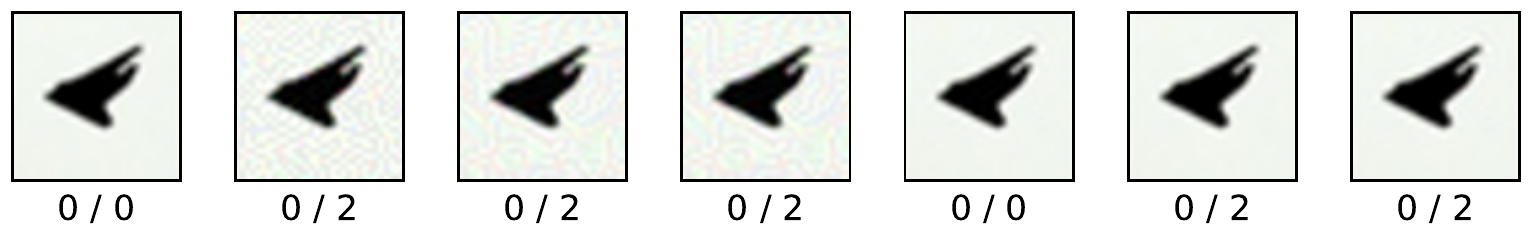} \hfill 	\hfill 			{\includegraphics[width=.475\linewidth]{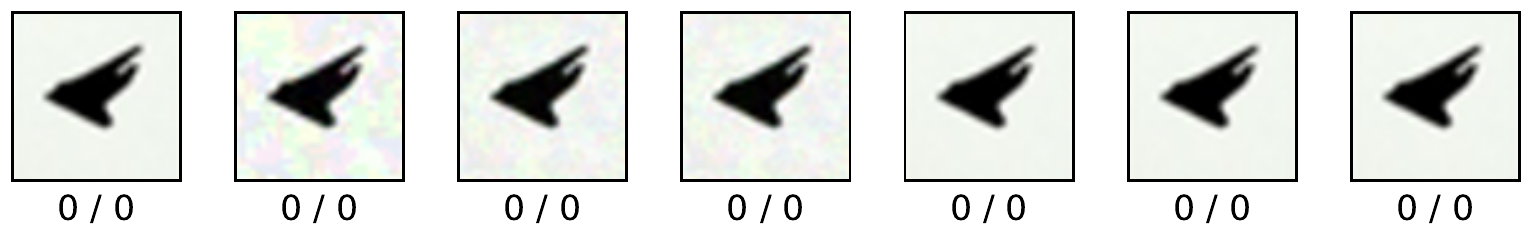}} \\

    \centering \includegraphics[width=.475\linewidth]{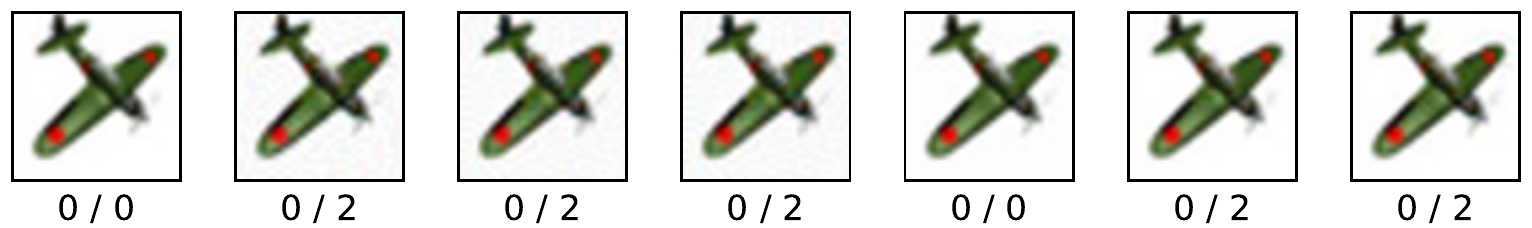} \hfill 	\hfill 			{\includegraphics[width=.475\linewidth]{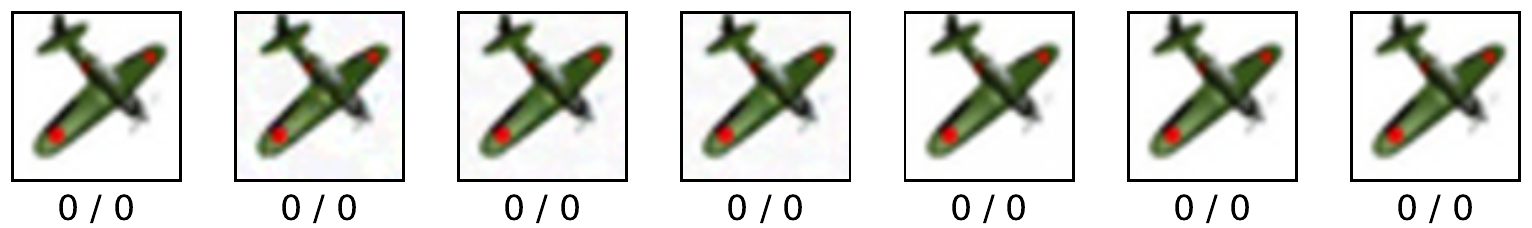}} \\

    \centering \includegraphics[width=.475\linewidth]{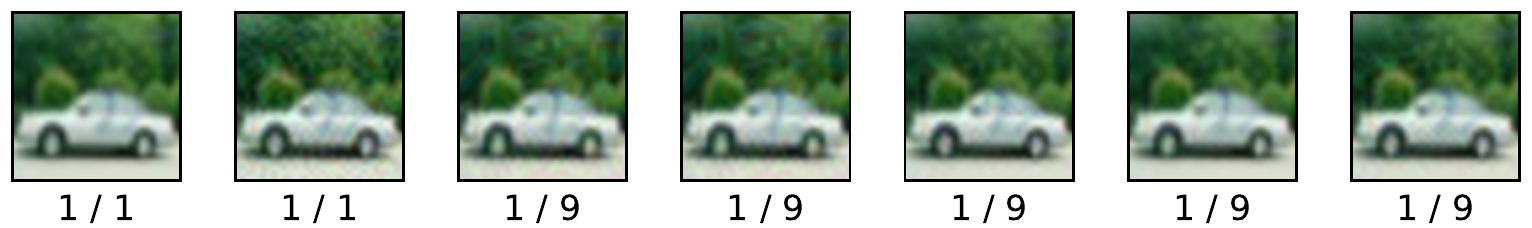} \hfill 	\hfill 			{\includegraphics[width=.475\linewidth]{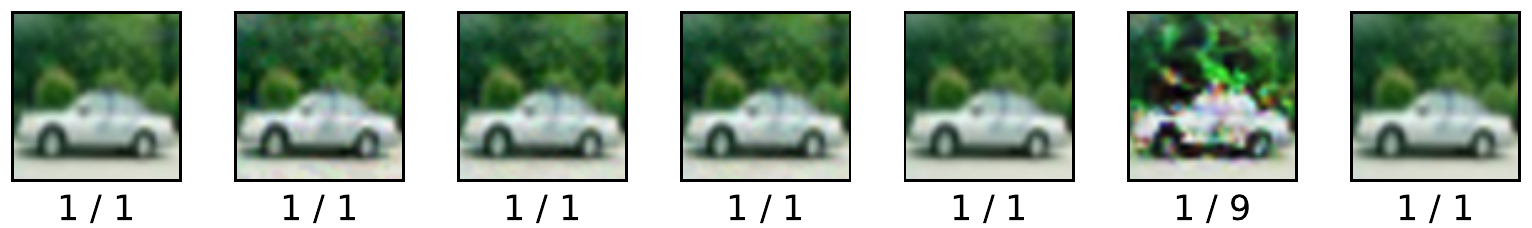}} \\

    \centering \includegraphics[width=.475\linewidth]{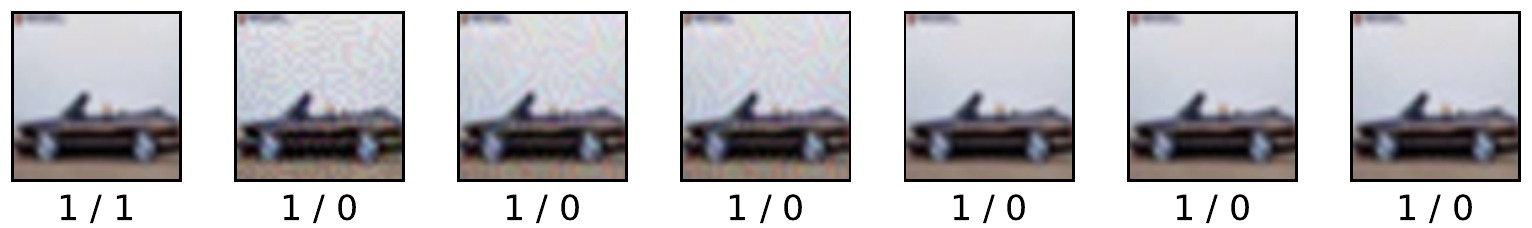} \hfill 	\hfill 			{\includegraphics[width=.475\linewidth]{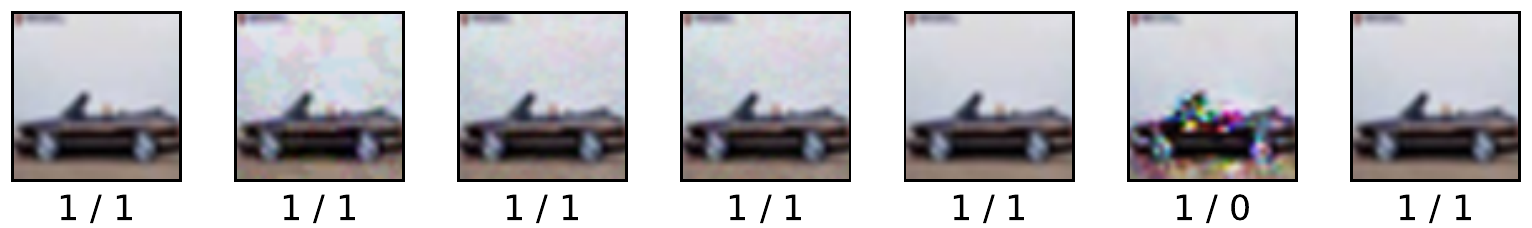}} \\

    \centering \includegraphics[width=.475\linewidth]{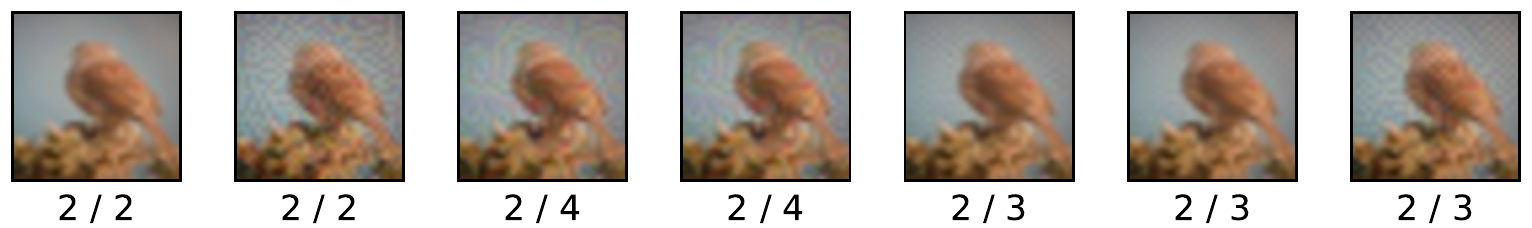} \hfill 	\hfill 			{\includegraphics[width=.475\linewidth]{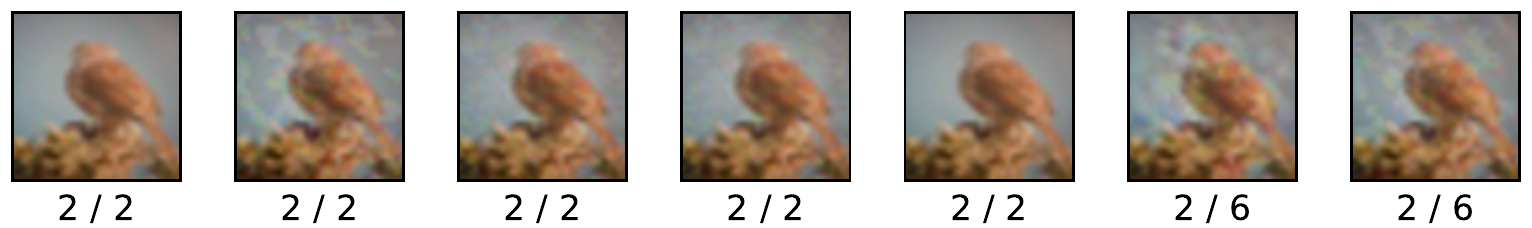}} \\

    \centering \includegraphics[width=.475\linewidth]{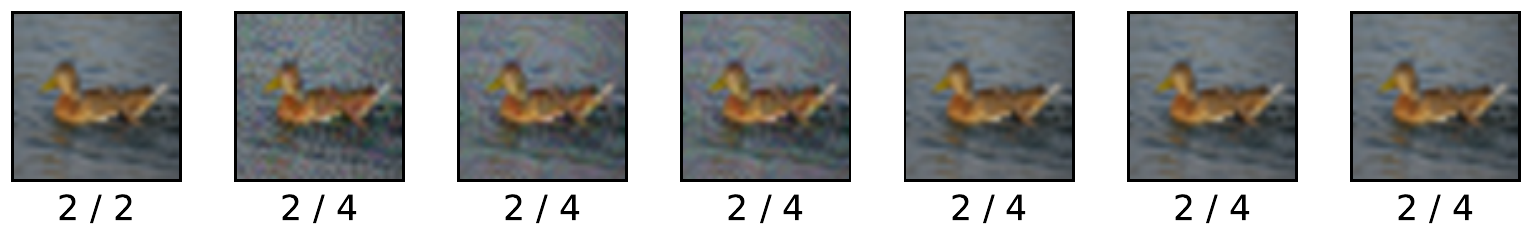} \hfill 	\hfill 			{\includegraphics[width=.475\linewidth]{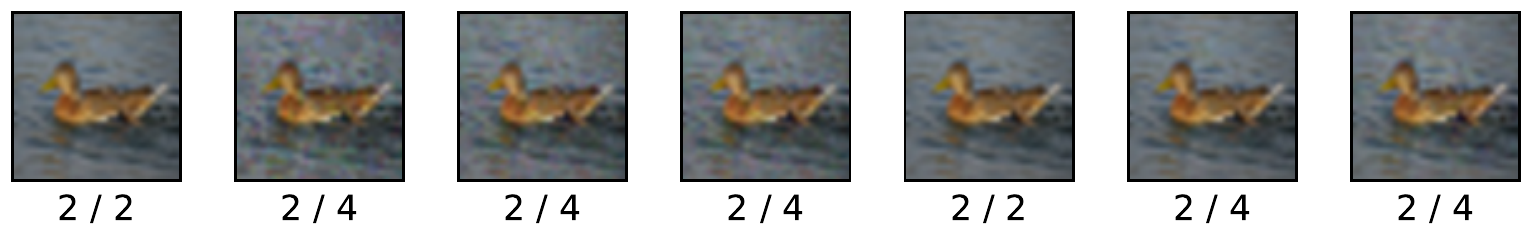}} \\

    \centering \includegraphics[width=.475\linewidth]{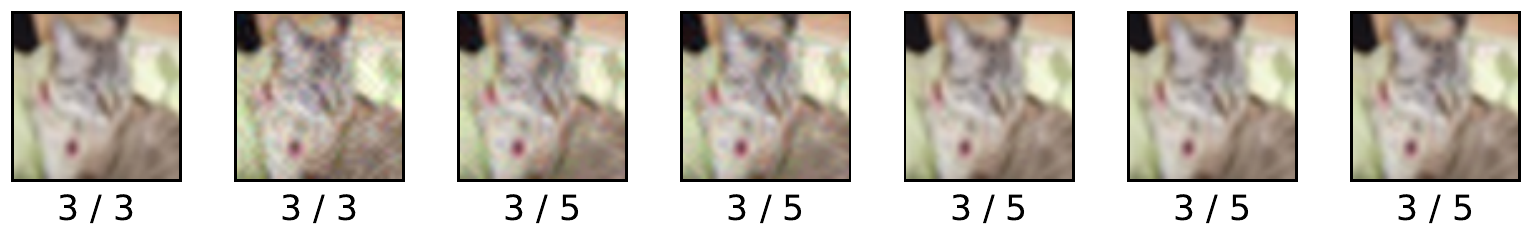} \hfill 	\hfill 			{\includegraphics[width=.475\linewidth]{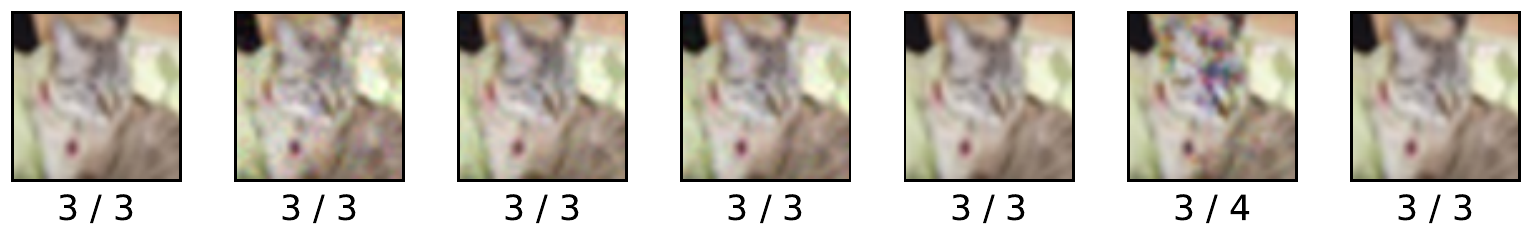}} \\

    \centering \includegraphics[width=.475\linewidth]{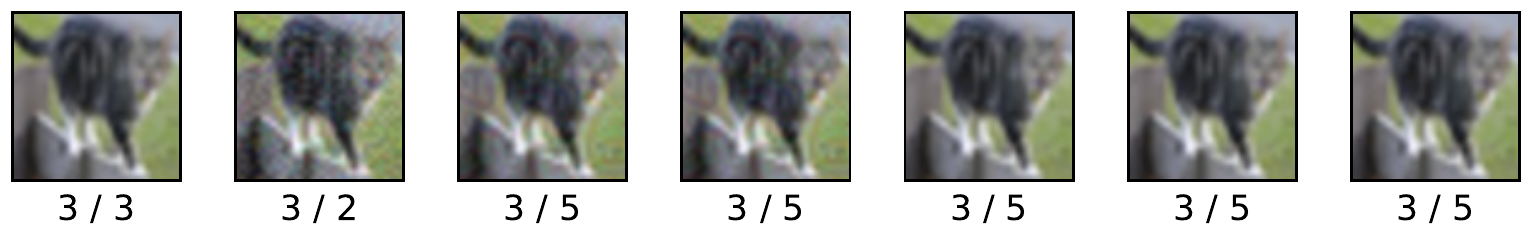} \hfill 	\hfill 			{\includegraphics[width=.475\linewidth]{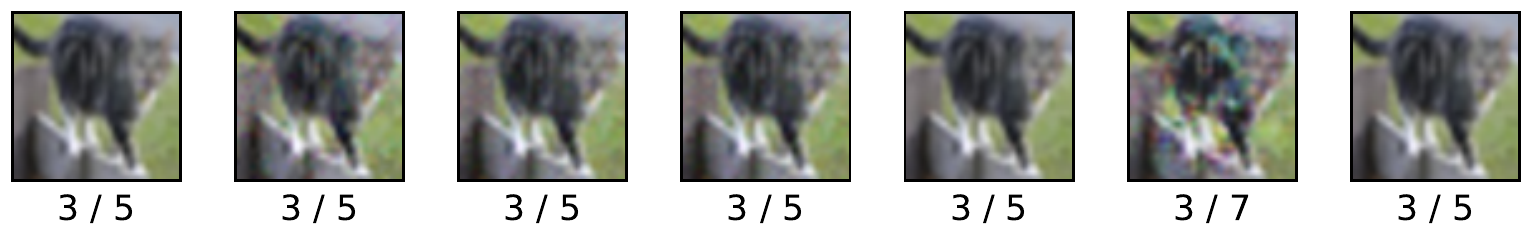}} \\

    \centering \includegraphics[width=.475\linewidth]{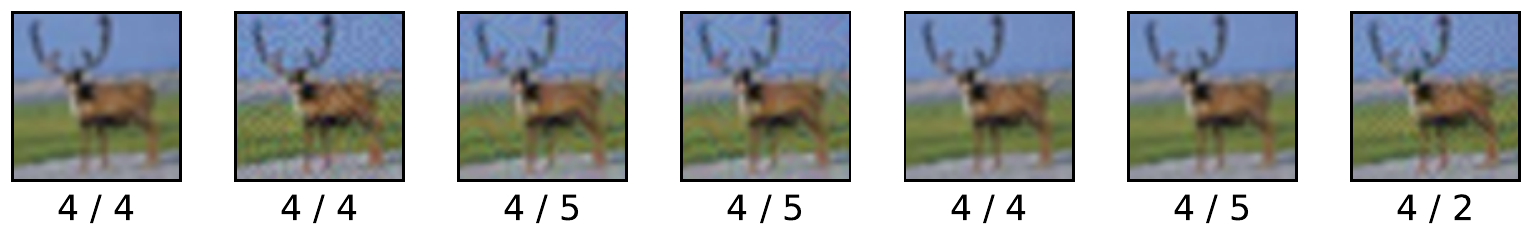} \hfill 	\hfill 			{\includegraphics[width=.475\linewidth]{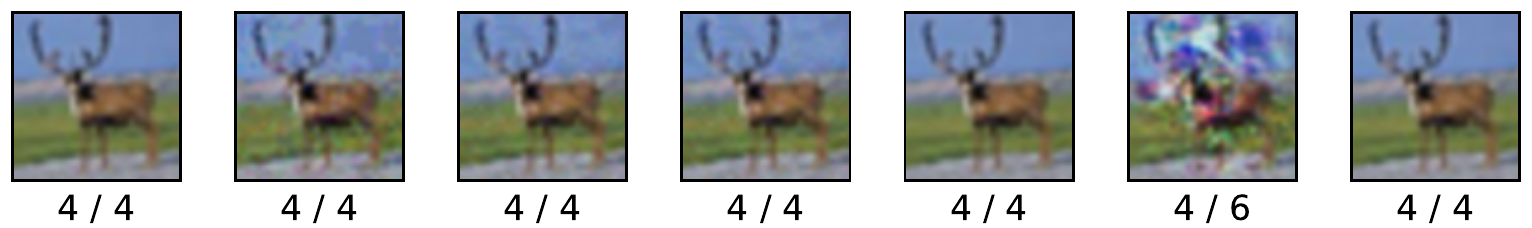}} \\

    \centering \includegraphics[width=.475\linewidth]{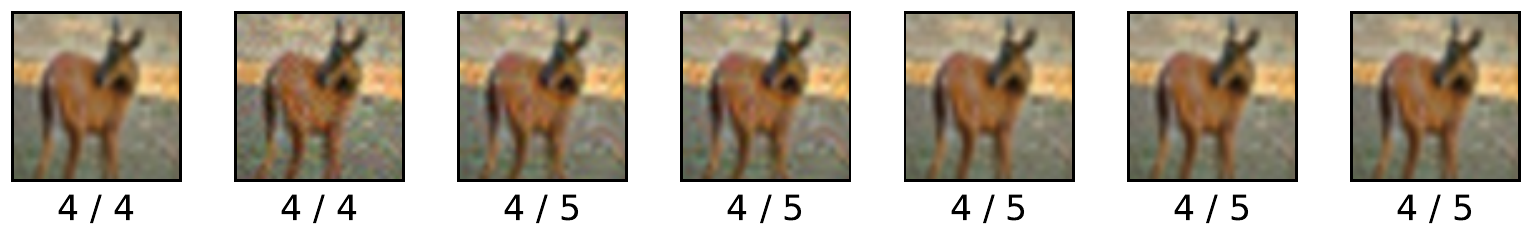} \hfill 	\hfill 			{\includegraphics[width=.475\linewidth]{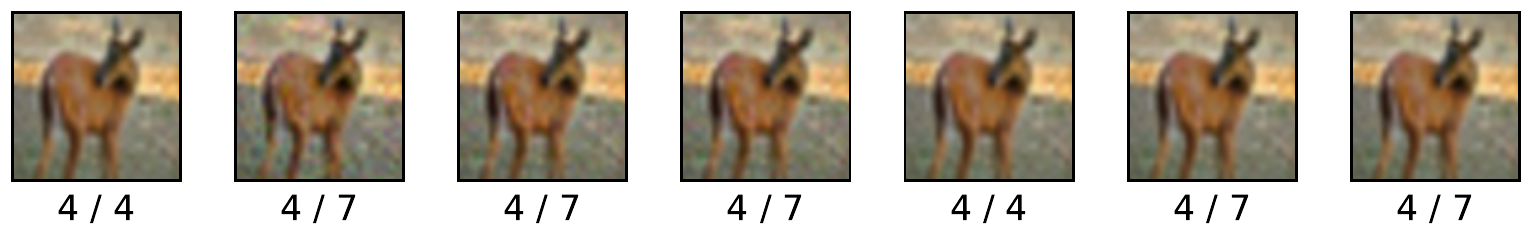}} \\

    \centering \includegraphics[width=.475\linewidth]{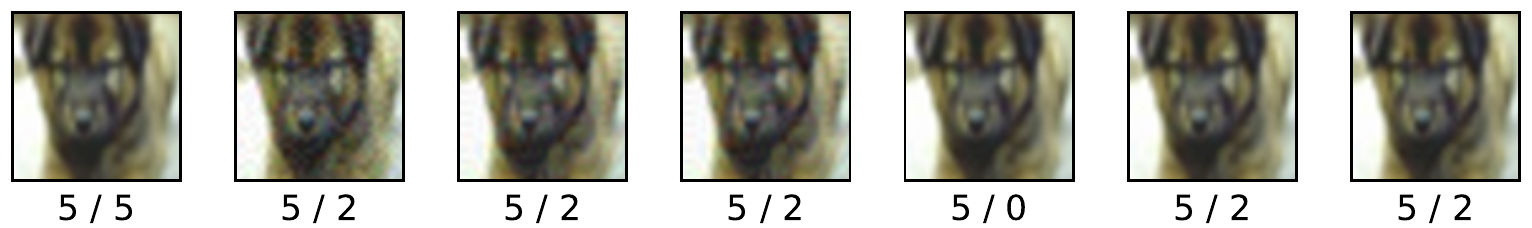} \hfill 	\hfill 			{\includegraphics[width=.475\linewidth]{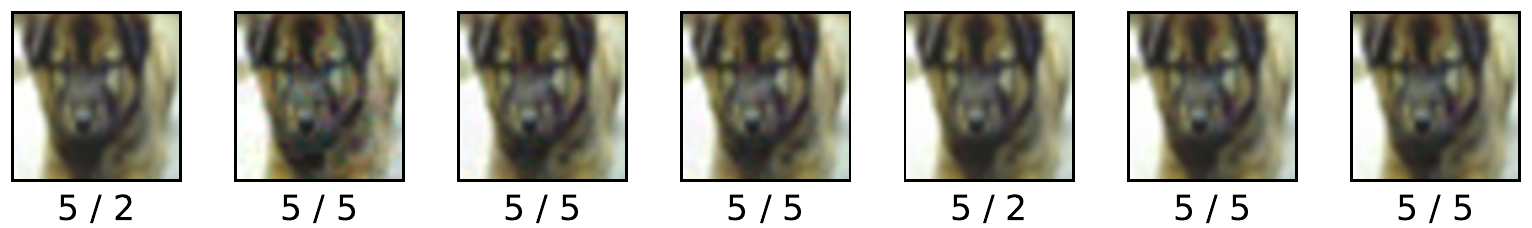}} \\

    \centering \includegraphics[width=.475\linewidth]{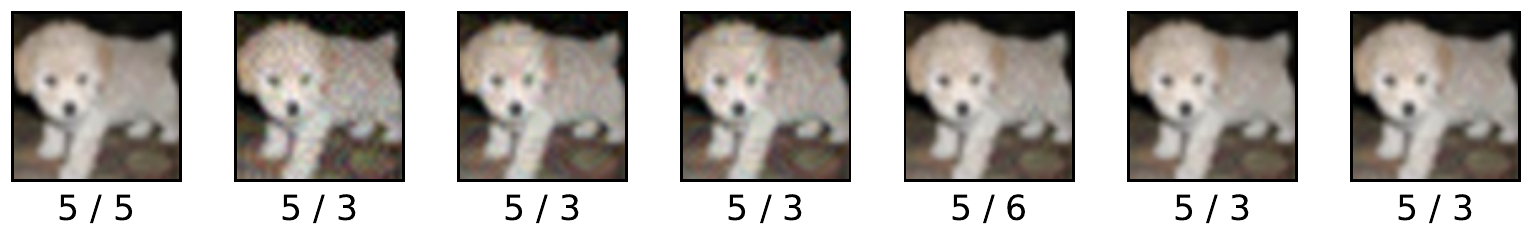} \hfill 	\hfill 			{\includegraphics[width=.475\linewidth]{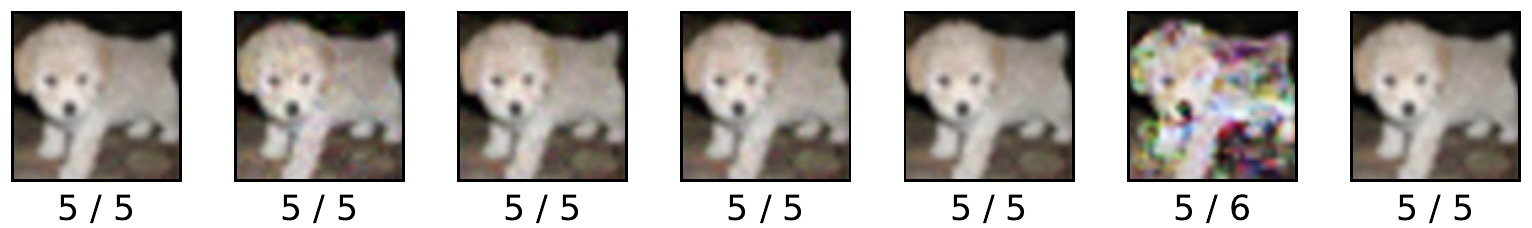}} \\

    \centering \includegraphics[width=.475\linewidth]{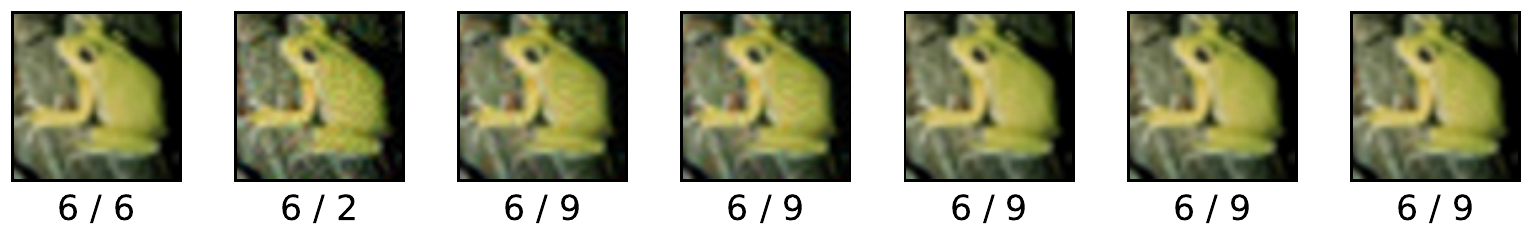} \hfill 	\hfill 		{\includegraphics[width=.475\linewidth]{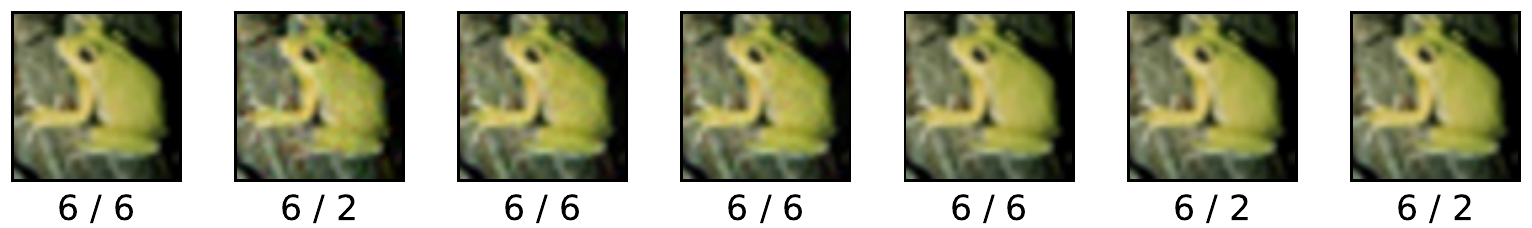}} \\

    \centering \includegraphics[width=.475\linewidth]{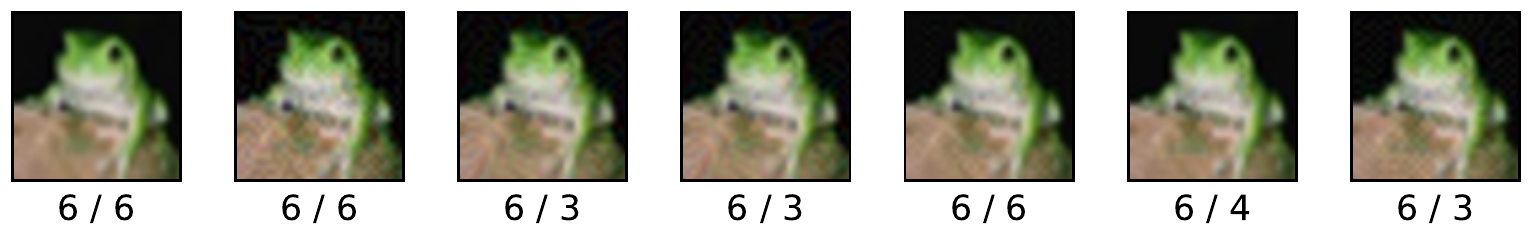} \hfill 	\hfill 			{\includegraphics[width=.475\linewidth]{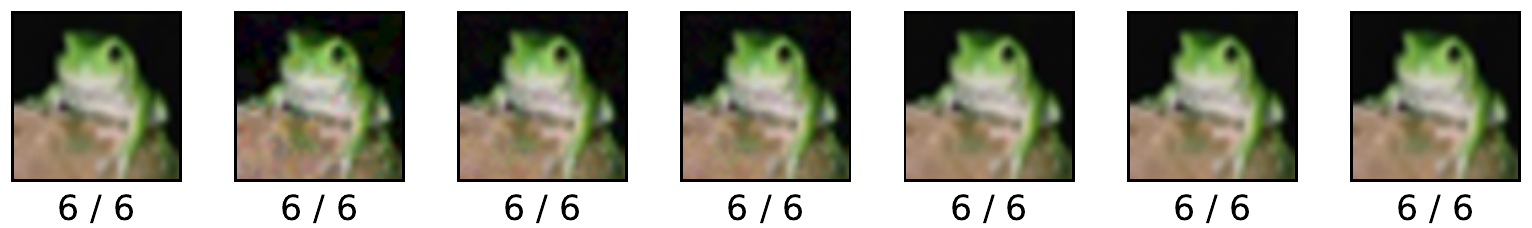}} \\

    \centering \includegraphics[width=.475\linewidth]{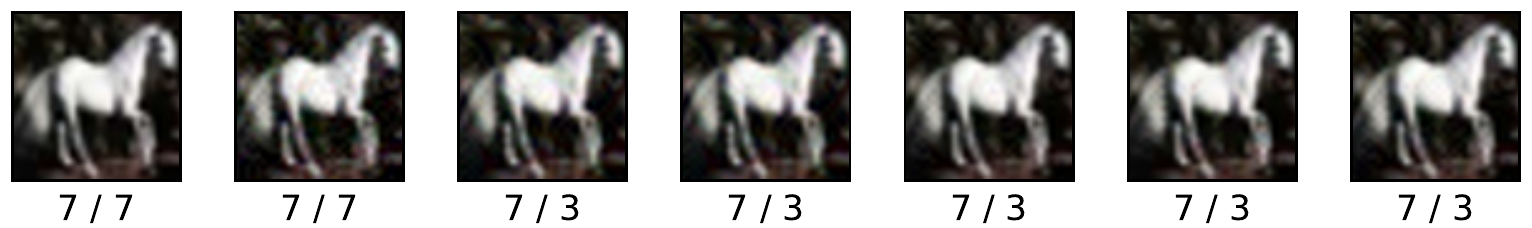} \hfill 	\hfill 			{\includegraphics[width=.475\linewidth]{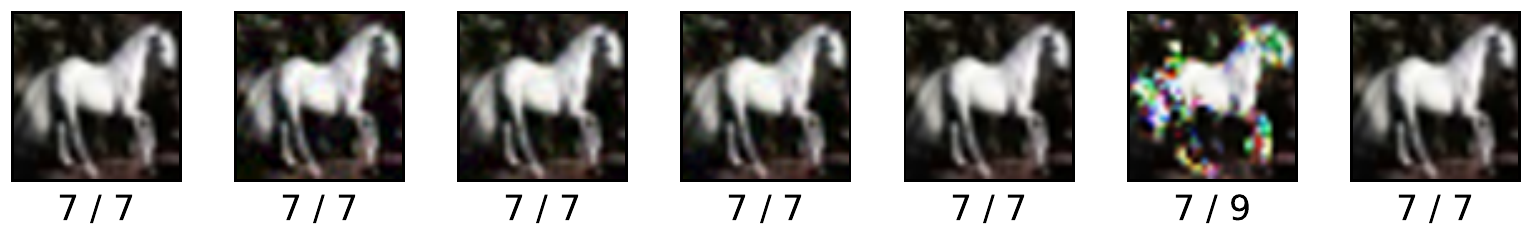}} \\

    \centering \includegraphics[width=.475\linewidth]{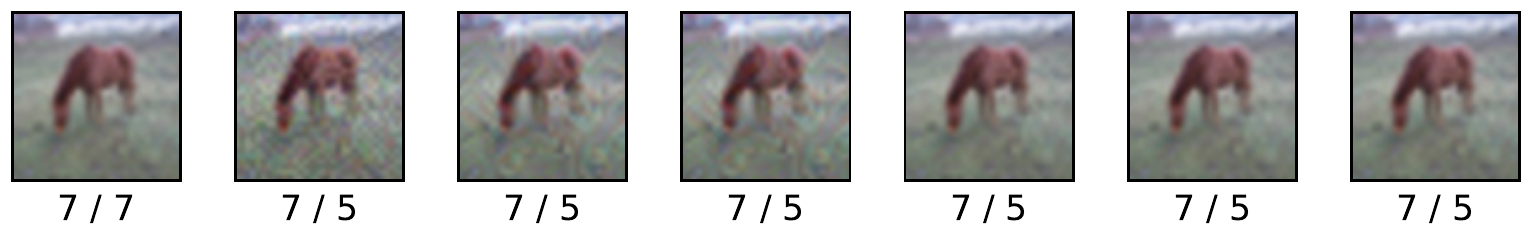} \hfill 	\hfill 		{\includegraphics[width=.475\linewidth]{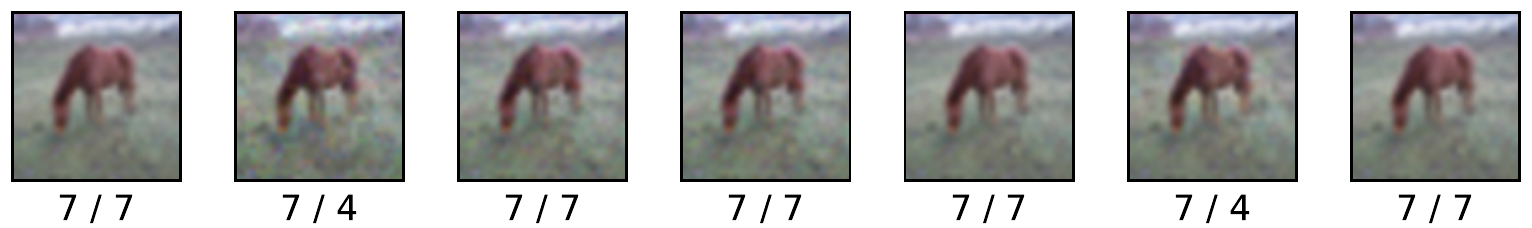}} \\

    \centering \includegraphics[width=.475\linewidth]{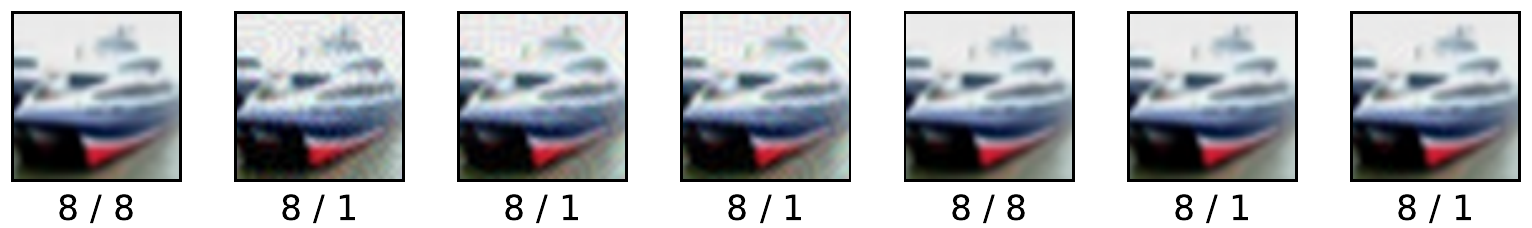} \hfill 	\hfill 			{\includegraphics[width=.475\linewidth]{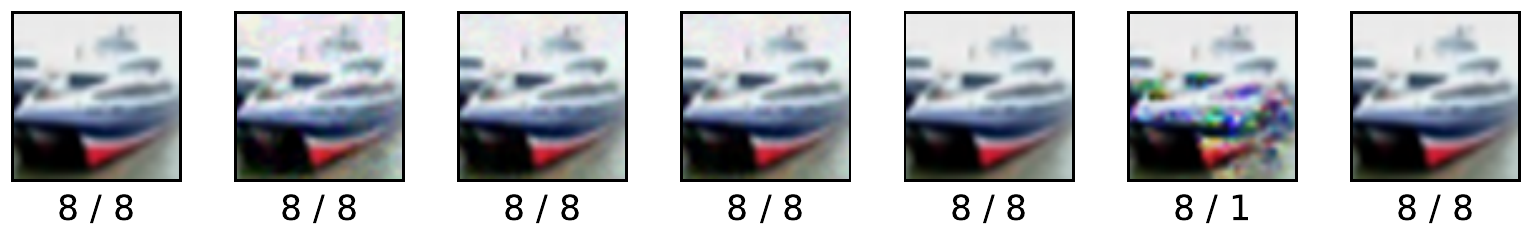}} \\

    \centering \includegraphics[width=.475\linewidth]{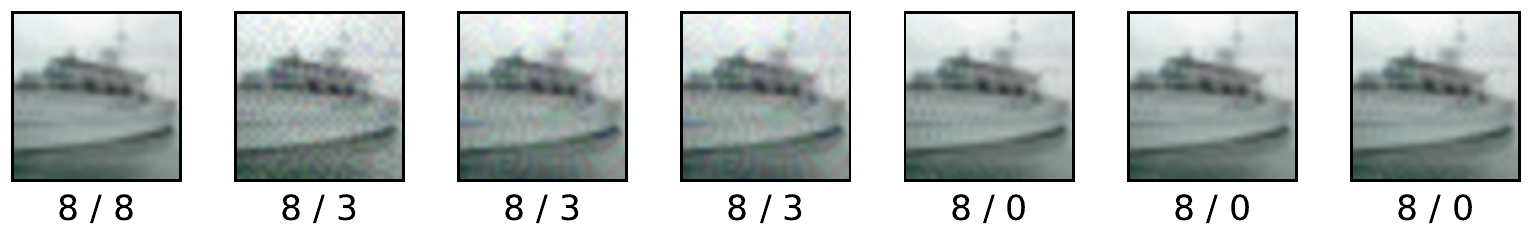} \hfill 	\hfill 			{\includegraphics[width=.475\linewidth]{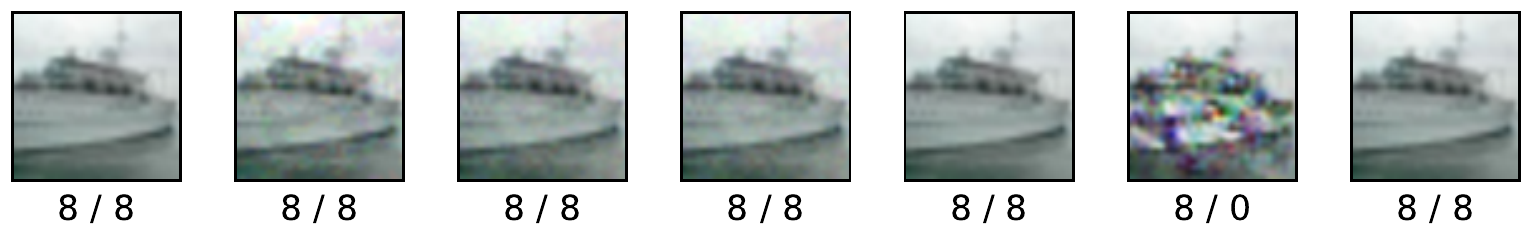}} \\

    \centering \includegraphics[width=.475\linewidth]{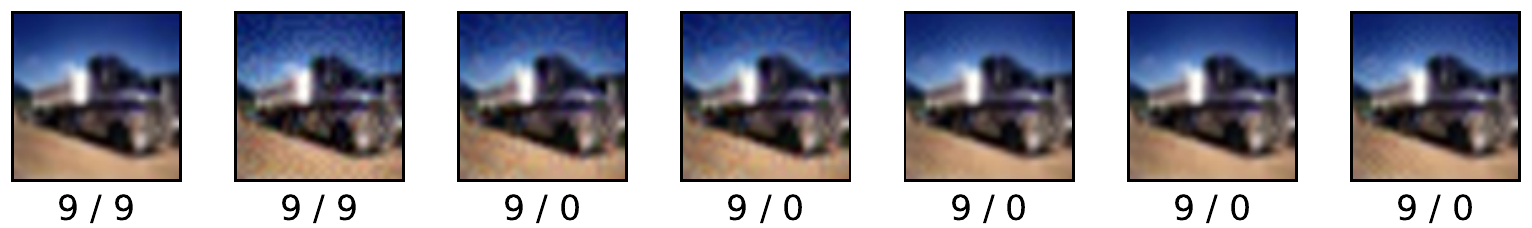} \hfill 	\hfill 			{\includegraphics[width=.475\linewidth]{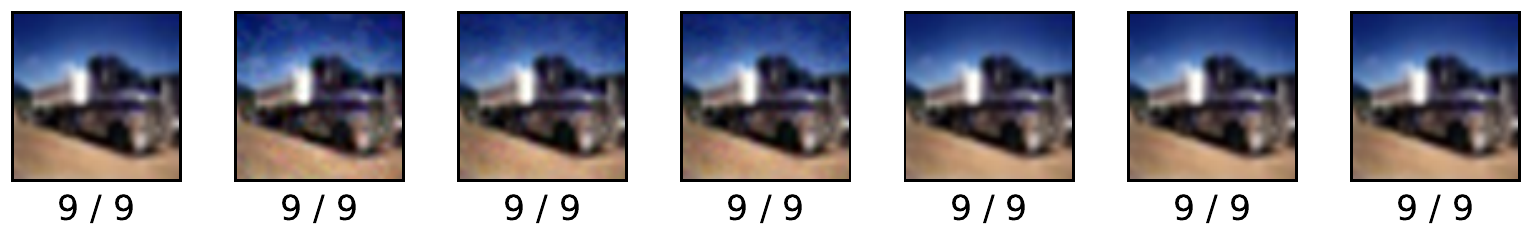}} \\

    \centering \includegraphics[width=.475\linewidth]{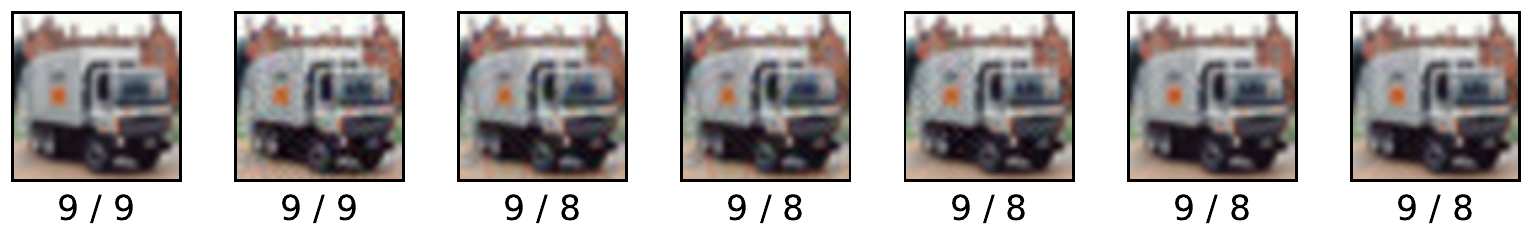} \hfill 	\hfill 			{\includegraphics[width=.475\linewidth]{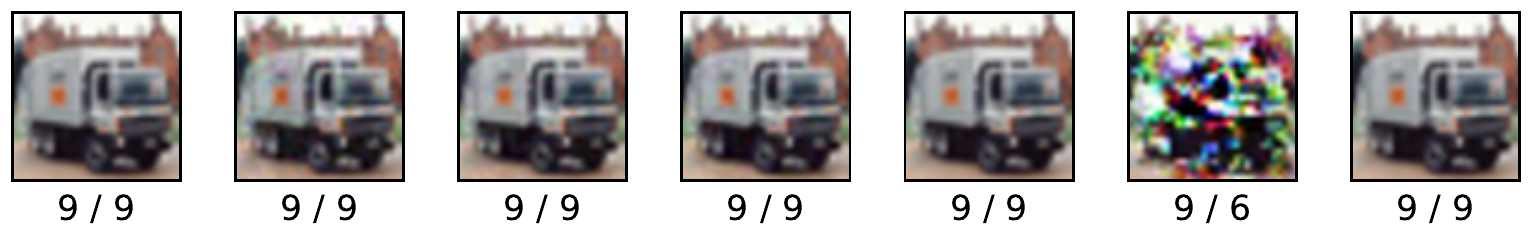}} \\

  \caption{Adversaries on CIFAR-10. We show clean samples followed by their perturbed versions that we obtained with ART using FGSM, BIM, PGD, DeepFool, and the two Carlini and Wagner adversaries. For each example, we show the true and predicted class indeces. In the left column, we display samples that originate from the regular network with ReLUs, denoted as WRN-28-10. On the right, we present the resulting images on the classifier with tent activations, WRN-28-10 tent (0.004).}
  \label{fig:cifar10_imgs}
\end{figure}

%% file: appendix2.tex
\section*{APPENDIX B: Adversarial Perturbations and the Open Space}
\label{sec:appendixb}

The core idea of our paper is that adversaries take advantage of deep neural networks (DNNs) by pushing feature representations into the open space via small perturbations.
In the right side of Figure \ref{fig:tent}, we have shown the distributions of activation's inputs for clean test examples and their perturbed counterparts obtained by using projected gradient descent (PGD) \cite{madry2018towards}.
We have highlighted that distributions of the top two activation layers' inputs of the DNN containing the widely used ReLUs significantly differ under the adversarial attack.
Here, we present additional results.

\subsection*{Distribution of Activation Inputs}

In Figure \ref{fig:activ_in}, we show the mixed distributions of activation inputs at each layer of three DNNs trained on the MNIST dataset:
\begin{itemize}[noitemsep,topsep=0pt,parsep=0pt,partopsep=0pt]
\item left column: the regularly trained model having ReLU activation functions, denoted as MNIST-Net bn in Table \ref{results_mnist},
\item middle column: the DNN with ReLUs we obtained using adversarial training via PGD as described in \cite{madry2018towards}, listed as MNIST-Net adv in the paper,
\item right column: the DNN containing tent activation functions that we trained with a weight-decay of $0.12$ on $\delta$ parameters of tents, denoted as MNIST-Net tent (0.12).
\end{itemize}

Note that these distributions are mixed regarding both dimensions and classes of the dataset.
Ideally, one would consider analyzing the feature representations of different classes separately to see whether adversarial perturbations drive those into the open space at different dimensions.
However, even these mixed distributions demonstrate with one of the strongest adversary, PGD, that the formed adversarial perturbations significantly change the range of activation inputs.

Considering the distributions of the regularly trained DNN represented in the left column of Figure \ref{fig:activ_in}, we can see  that, in general, there is an increase in activation's inputs due to the adversarial attack via PGD.
Note that this adversary is very successful on this trained model as only 39 out of the 10k perturbed test examples remain correctly recognized by the classifier.
For the top two activation layers, compared to the clean test examples the adversarial perturbations yield feature representations that were clearly ``unseen'' by the classifier before as the visualized distributions for PGD significantly widen.
Consequently, the examined perturbations push feature representations into the open space.

Since adversarial training improves the robustness of DNNs to adversarial attacks, it is natural to ask how activation's inputs evolve during attacks.
To answer that question we show the same distributions in the middle column of Figure \ref{fig:activ_in} for the classifier that we obtained with adversarial training via PGD.
Note that this classifier achieves 91.46\% accuracy on test examples that were perturbed via PGD.
While this DNN also contains ReLU activation functions, the obtained distributions for clean and PGD samples remain similar.
Taking a closer look at those histograms, we can notice changes but the range of activation's inputs remain unchanged.  

Finally, we show the distributions for the classifier having tent activation functions in the right column of Figure \ref{fig:activ_in}.
This DNN achieves 88.37\% accuracy on the 10k test examples that were perturbed via PGD.
For each tent activation layer, we present the learned $\delta$ parameter which specifies the size of the activation function and, eventually, represents the maximum output the particular tent activation can produce with a zero input.
Looking at the visualized distributions of clean and PGD examples, we can only notice a bigger difference at the top layer, the others remain fairly stable.
Interestingly, this classifier has the largest tent activation functions at the top layer; the size (0.2337) is nearly twice the size of the second largest tent (0.1279).

\subsection*{Summary}

In this section, we have demonstrated how activation's inputs evolve under adversarial attacks via PGD compared to clean test examples.
While our results presented in this paper are more than promising, we believe that the full potential of tent activation functions has not been revealed yet.
The naive approach that we used in our experiments -- namely, the application of the same weight-decay on all tent activation layers -- probably yields sub-optimal results.
Future work can explore other ways to limit the size of tents and apply different policies at the different activation layers, for example, focusing more on top layers.

\begin{figure}

    \centering
    \includegraphics[width=.3\linewidth]{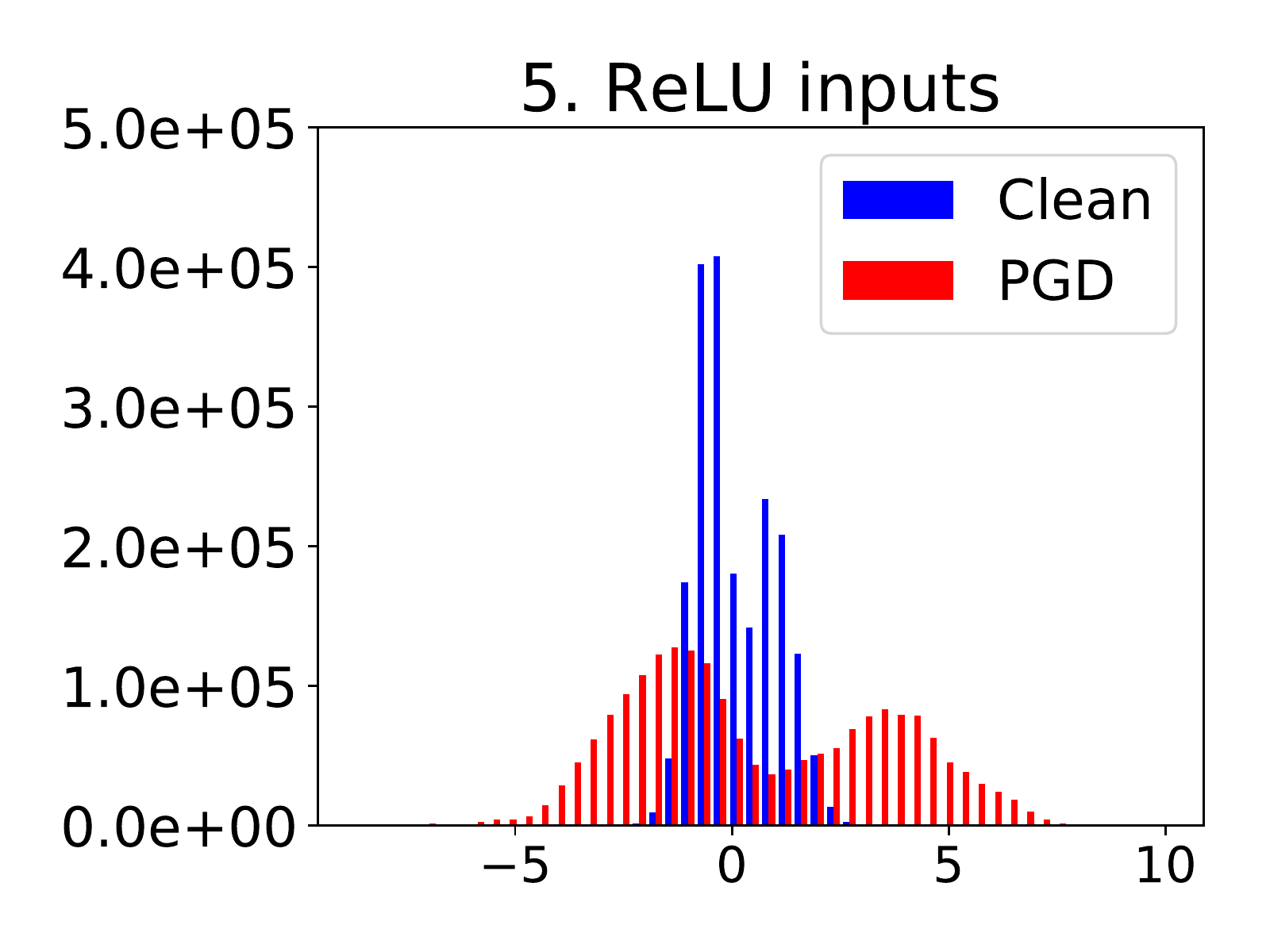} \hfill
    \includegraphics[width=.3\linewidth]{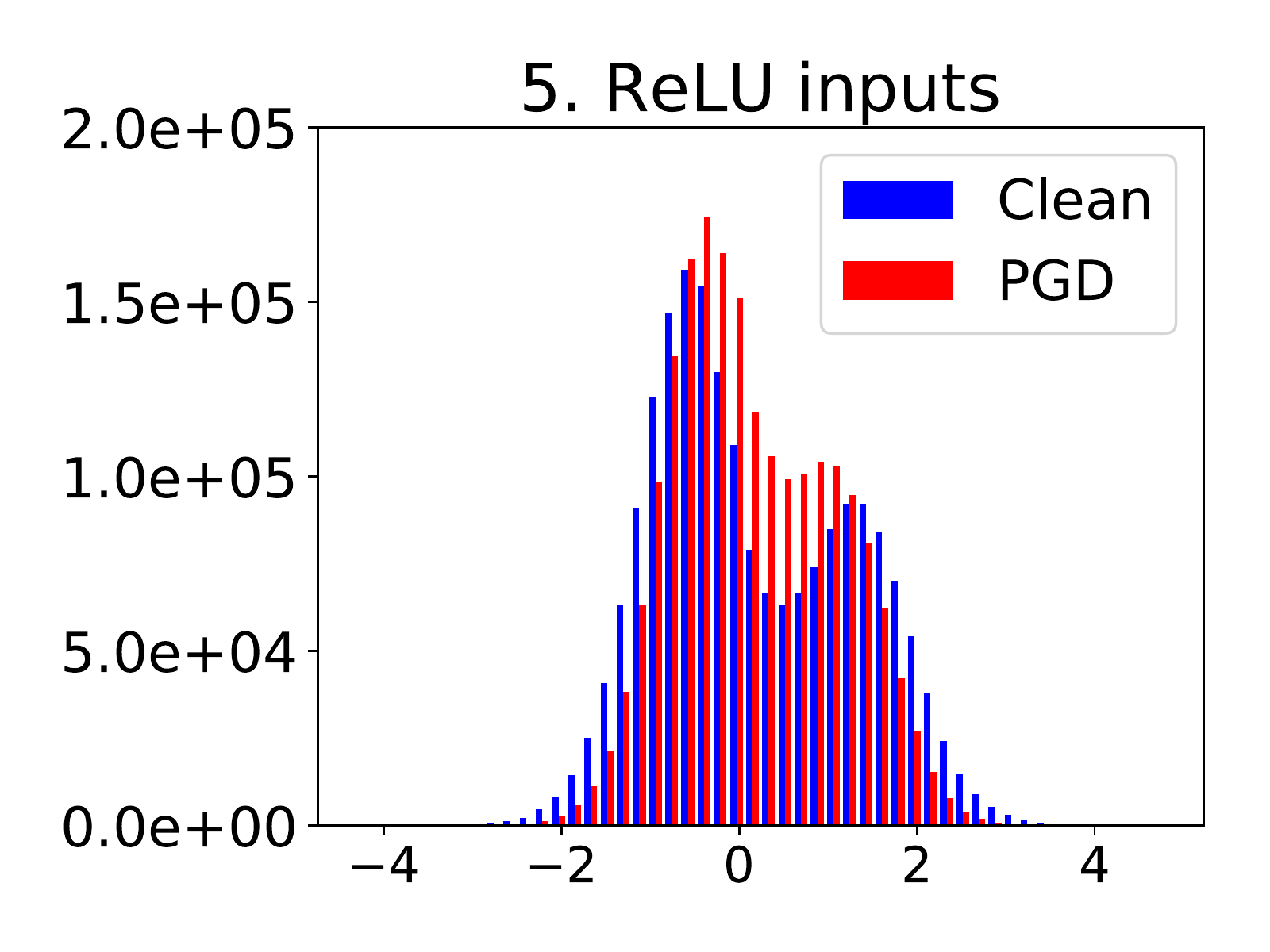} \hfill
    \includegraphics[width=.3\linewidth]{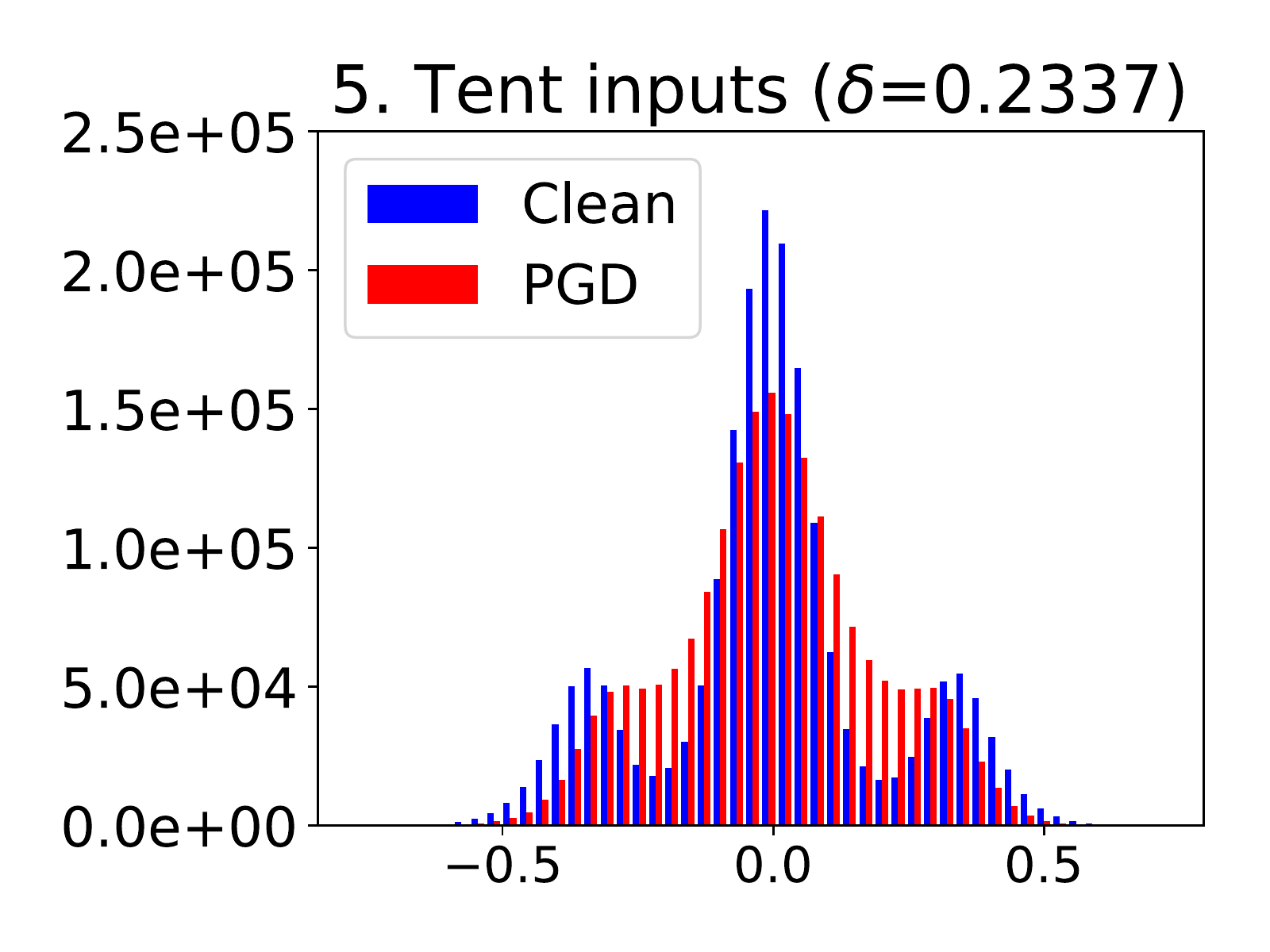} \\

    \centering
    \includegraphics[width=.3\linewidth]{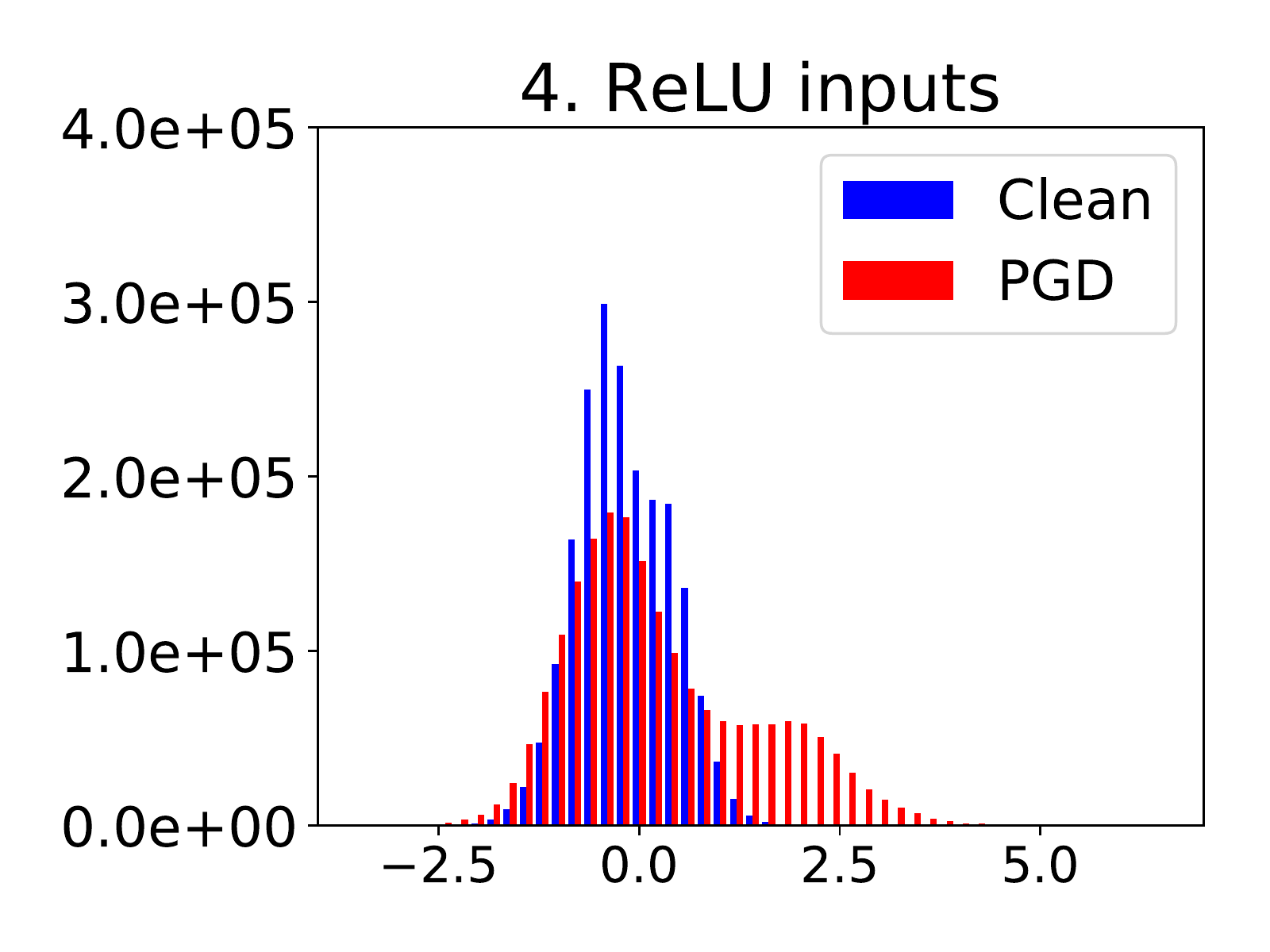} \hfill
    \includegraphics[width=.3\linewidth]{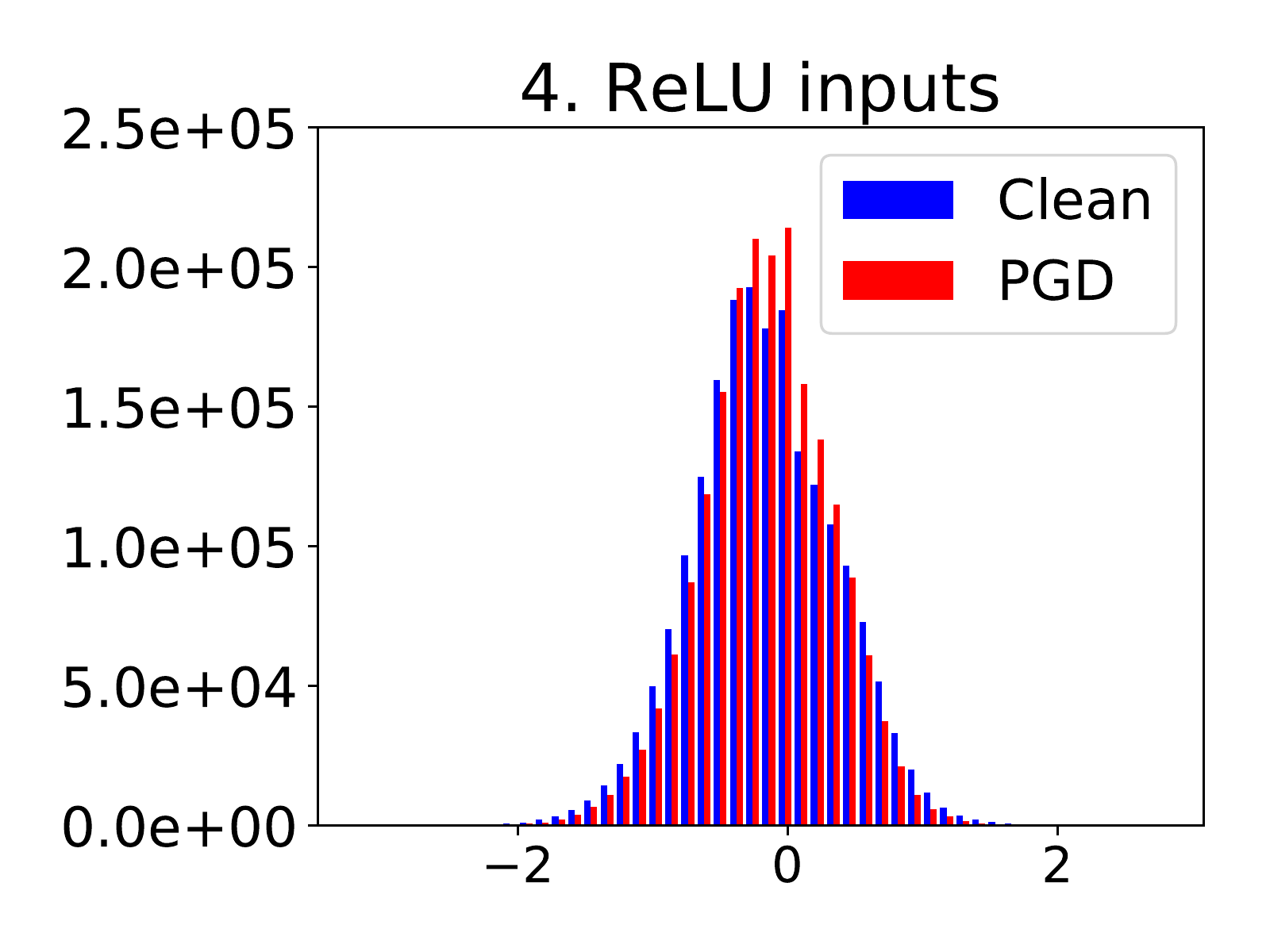} \hfill 
    \includegraphics[width=.3\linewidth]{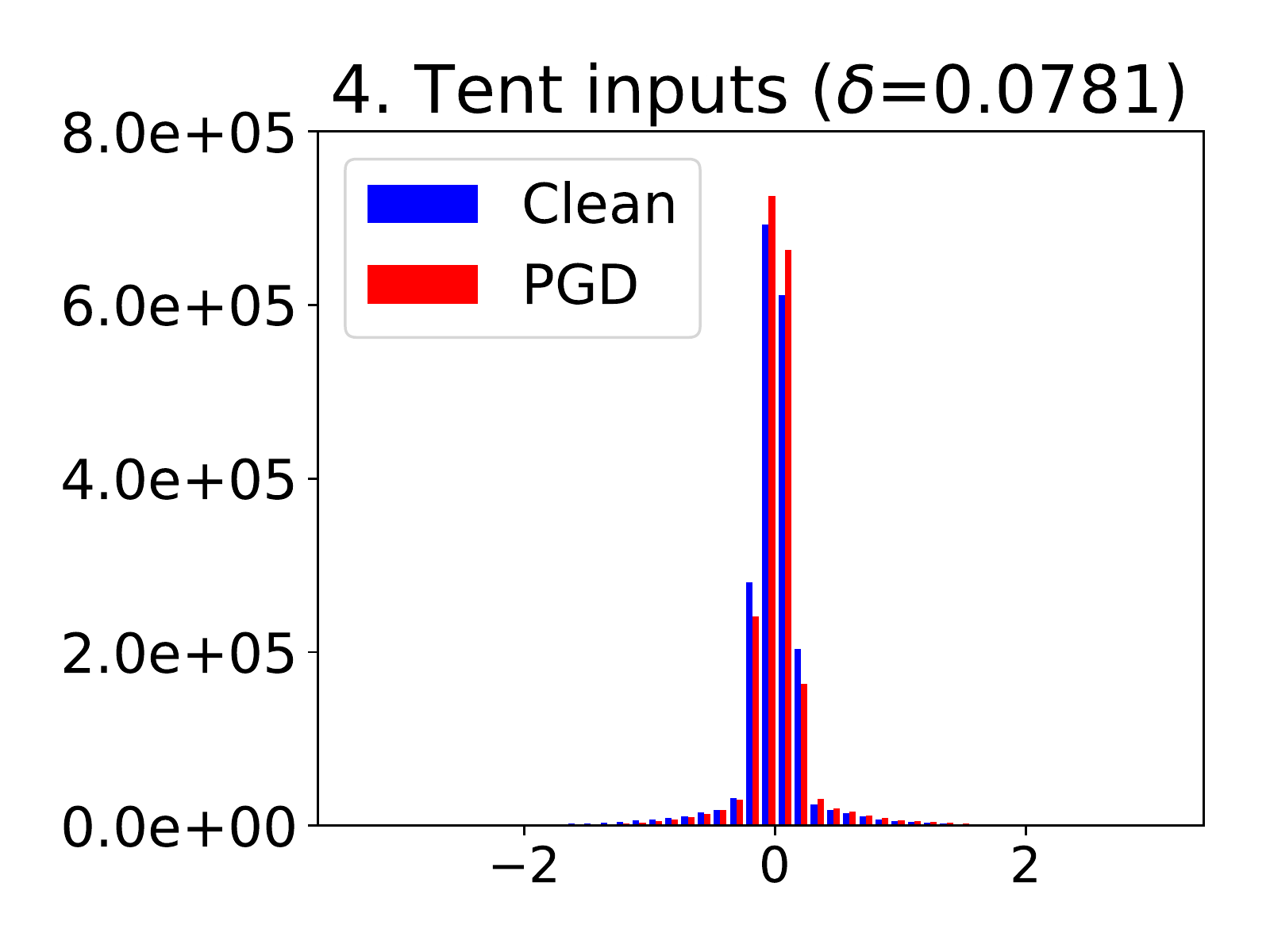} \\

    \centering
    \includegraphics[width=.3\linewidth]{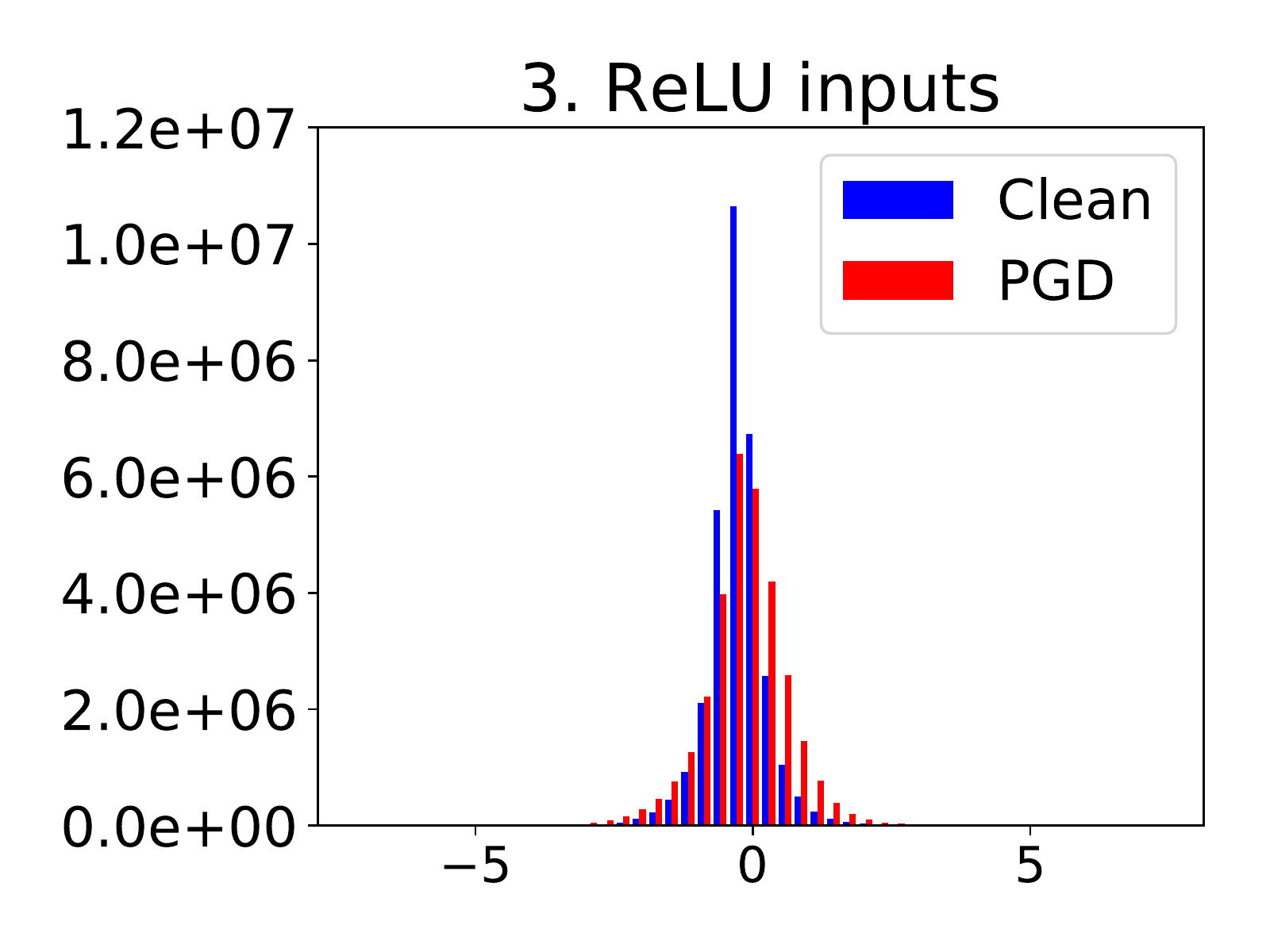} \hfill
    \includegraphics[width=.3\linewidth]{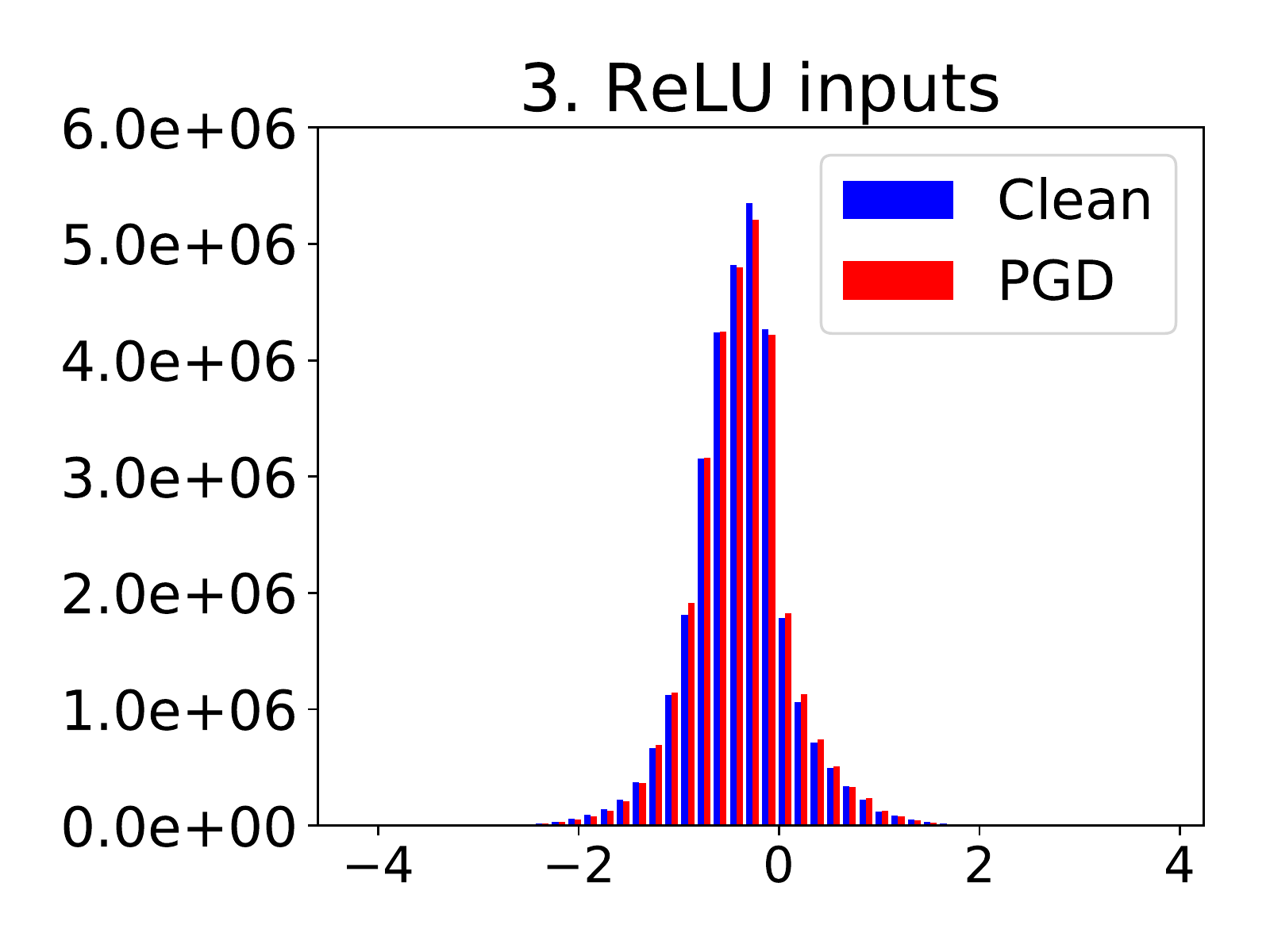} \hfill
    \includegraphics[width=.3\linewidth]{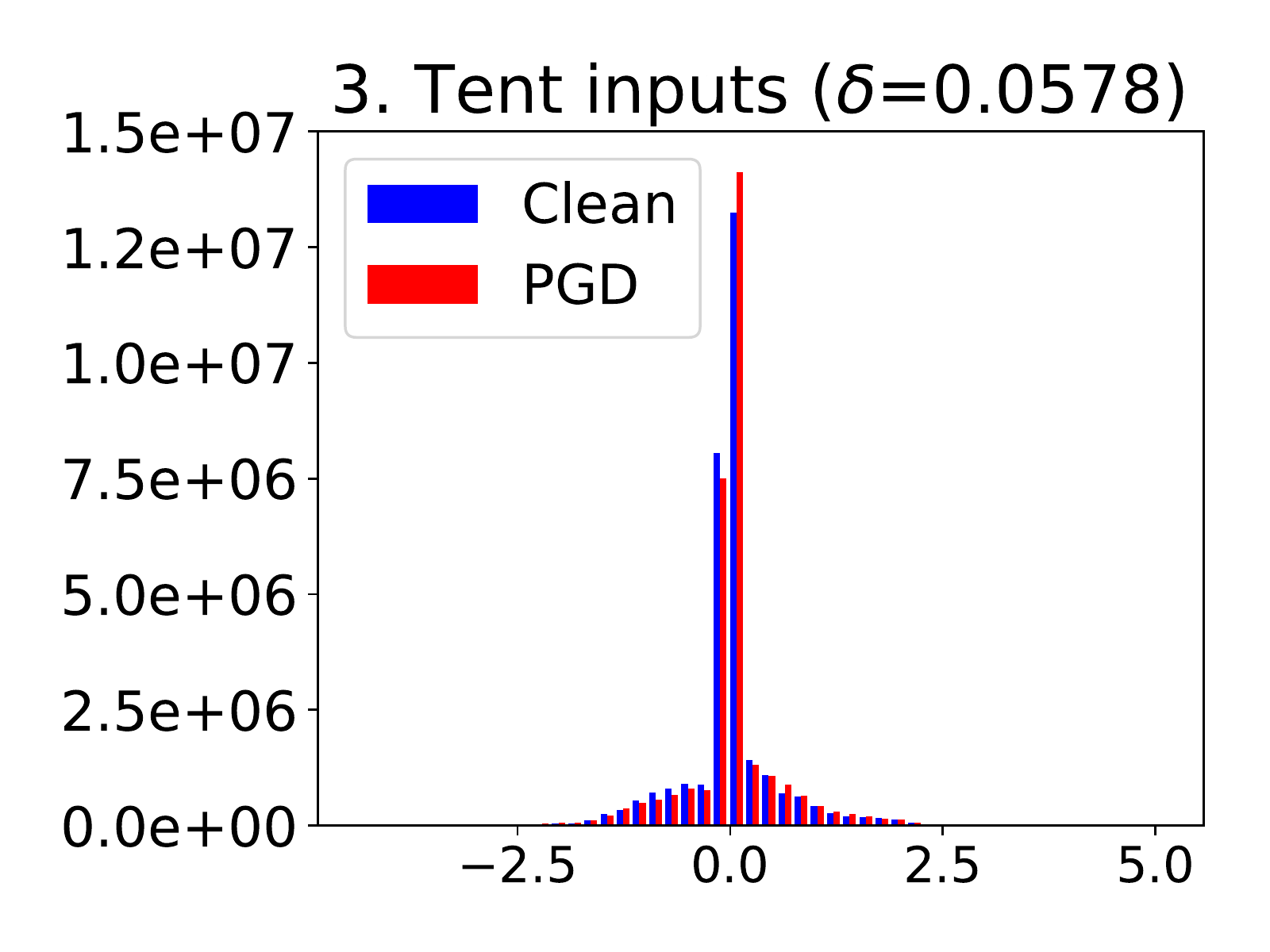} \\

    \centering
    \includegraphics[width=.3\linewidth]{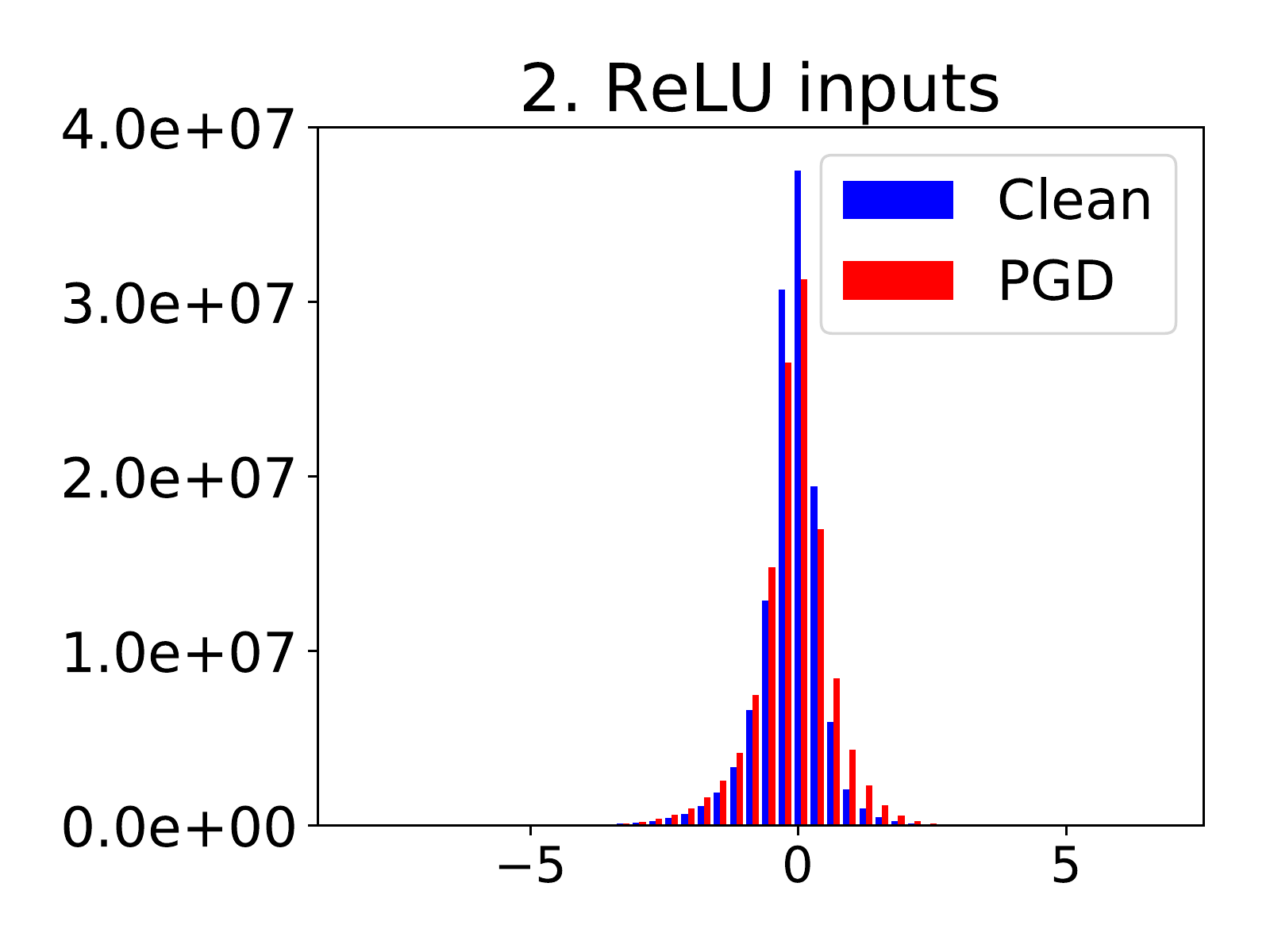} \hfill
    \includegraphics[width=.3\linewidth]{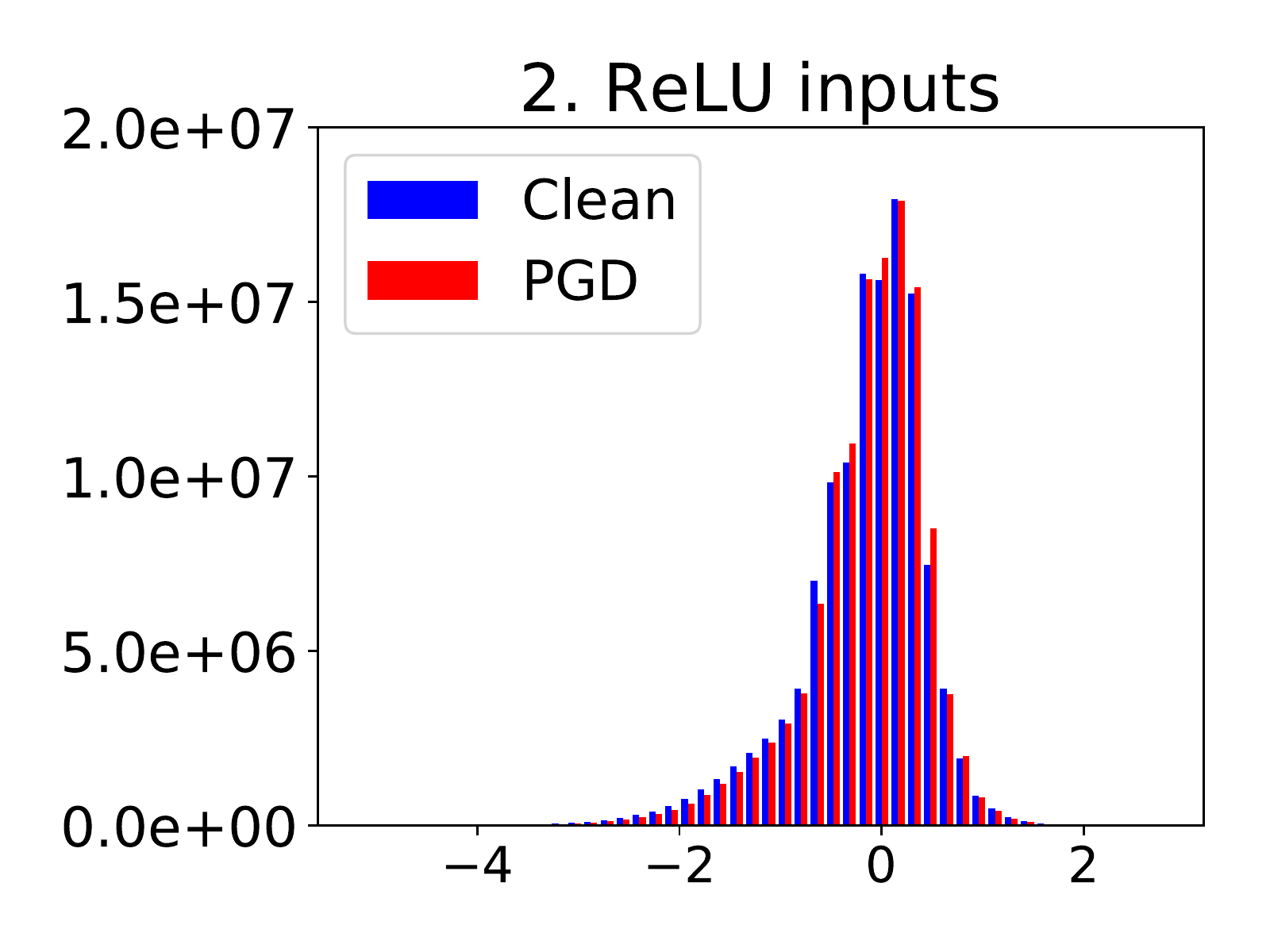} \hfill
    \includegraphics[width=.3\linewidth]{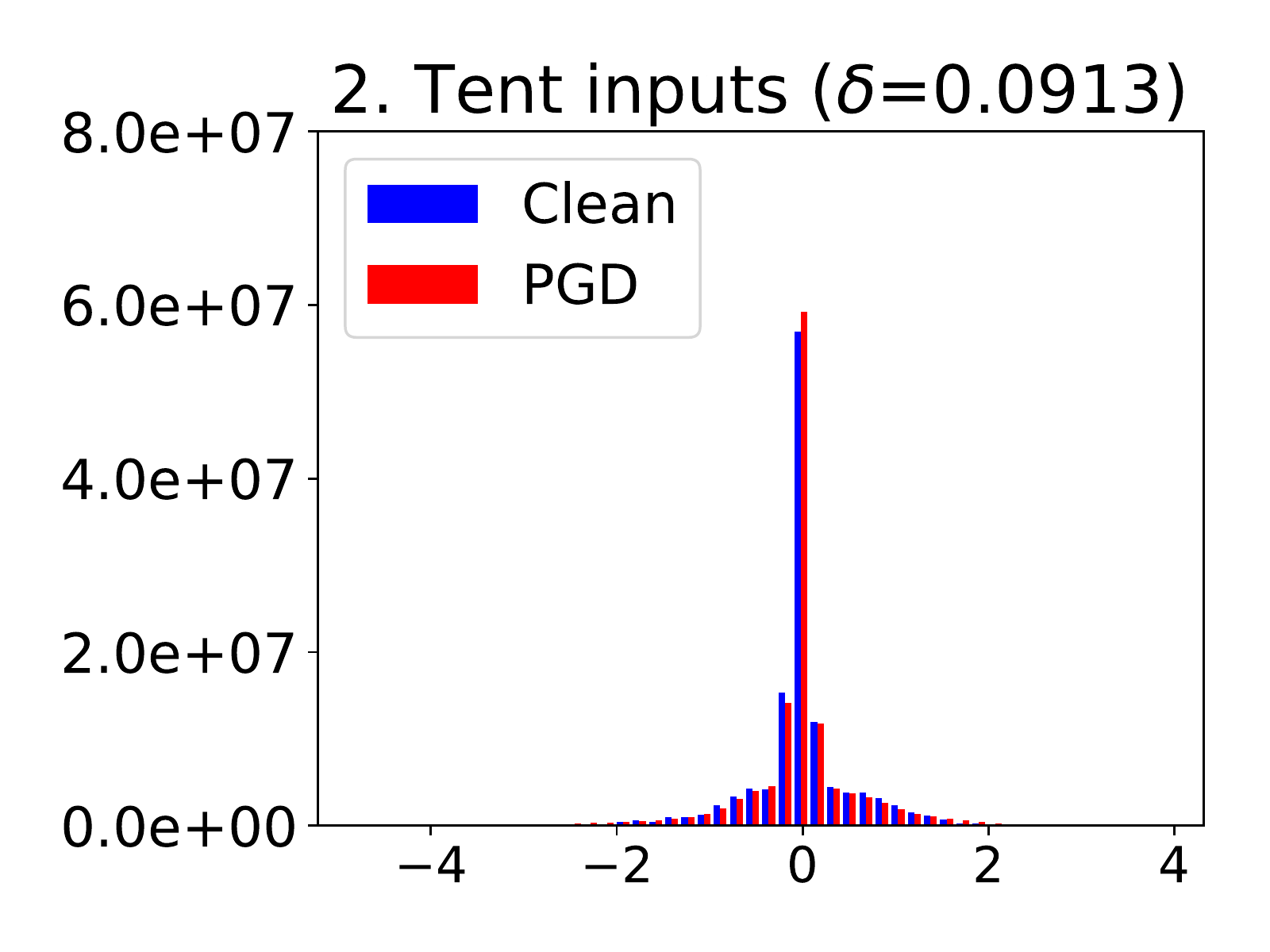} \\

    \centering
    \includegraphics[width=.3\linewidth]{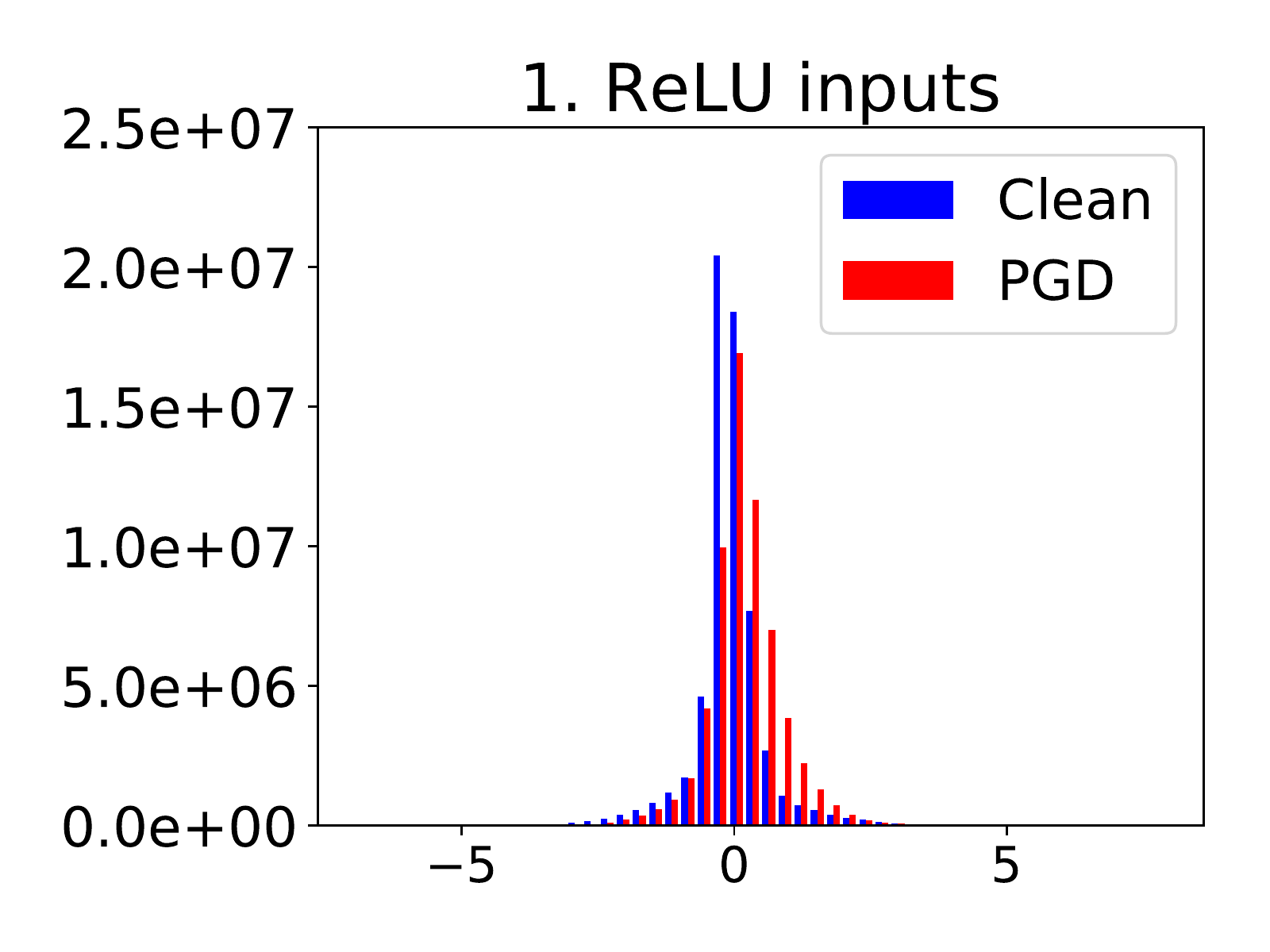} \hfill
    \includegraphics[width=.3\linewidth]{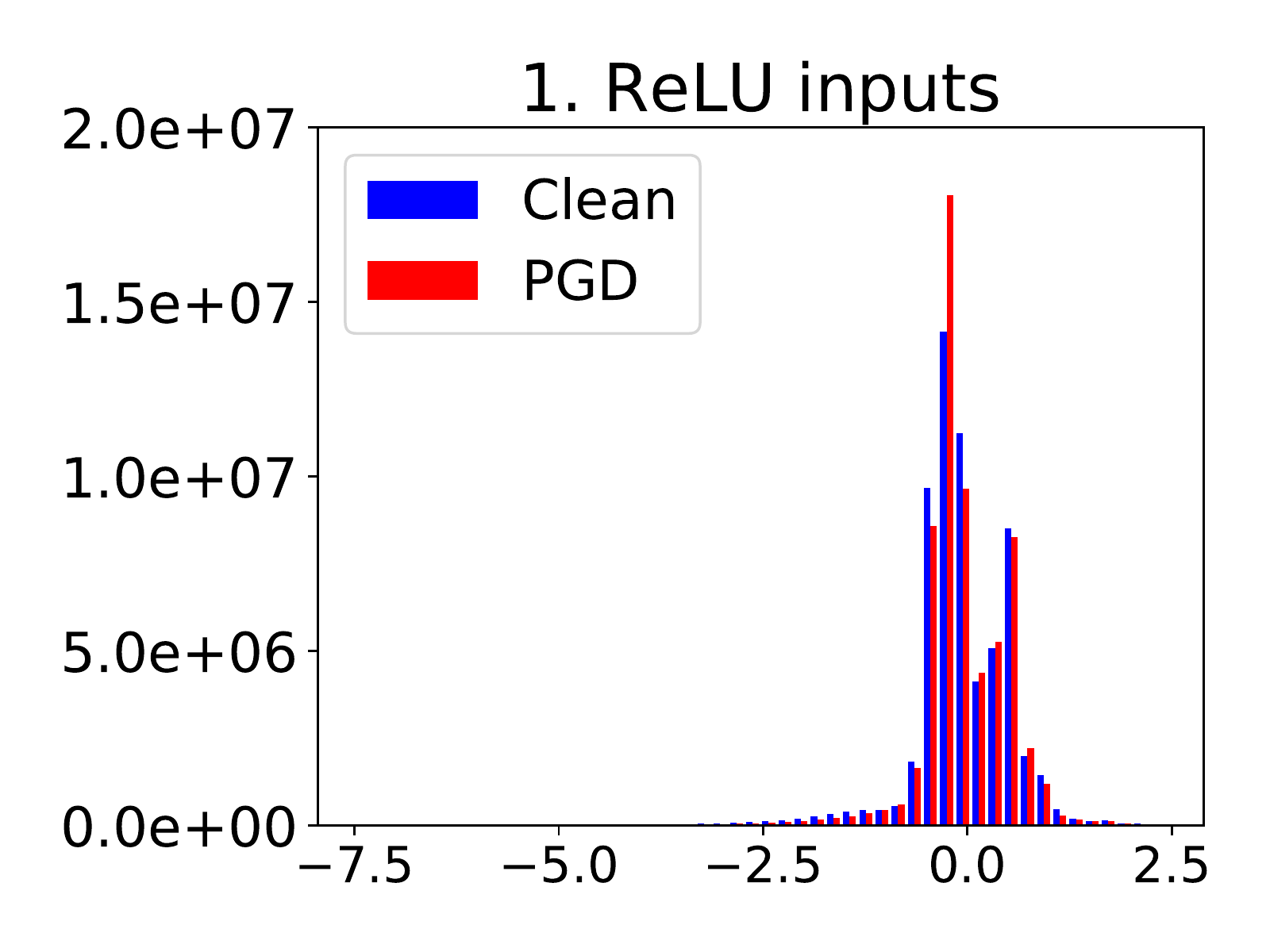} \hfill
    \includegraphics[width=.3\linewidth]{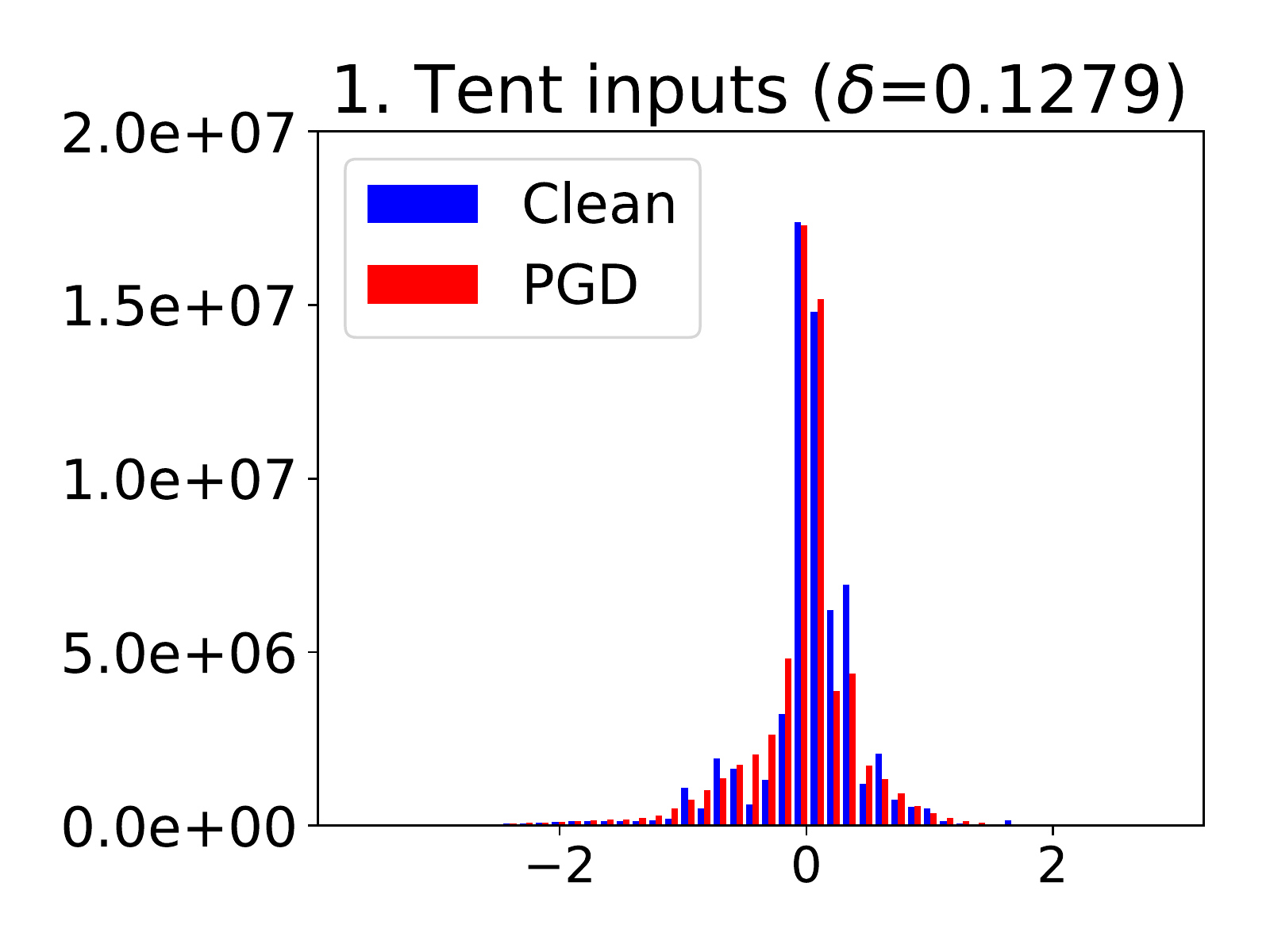} \\

    \centering
    \includegraphics[width=.3\linewidth]{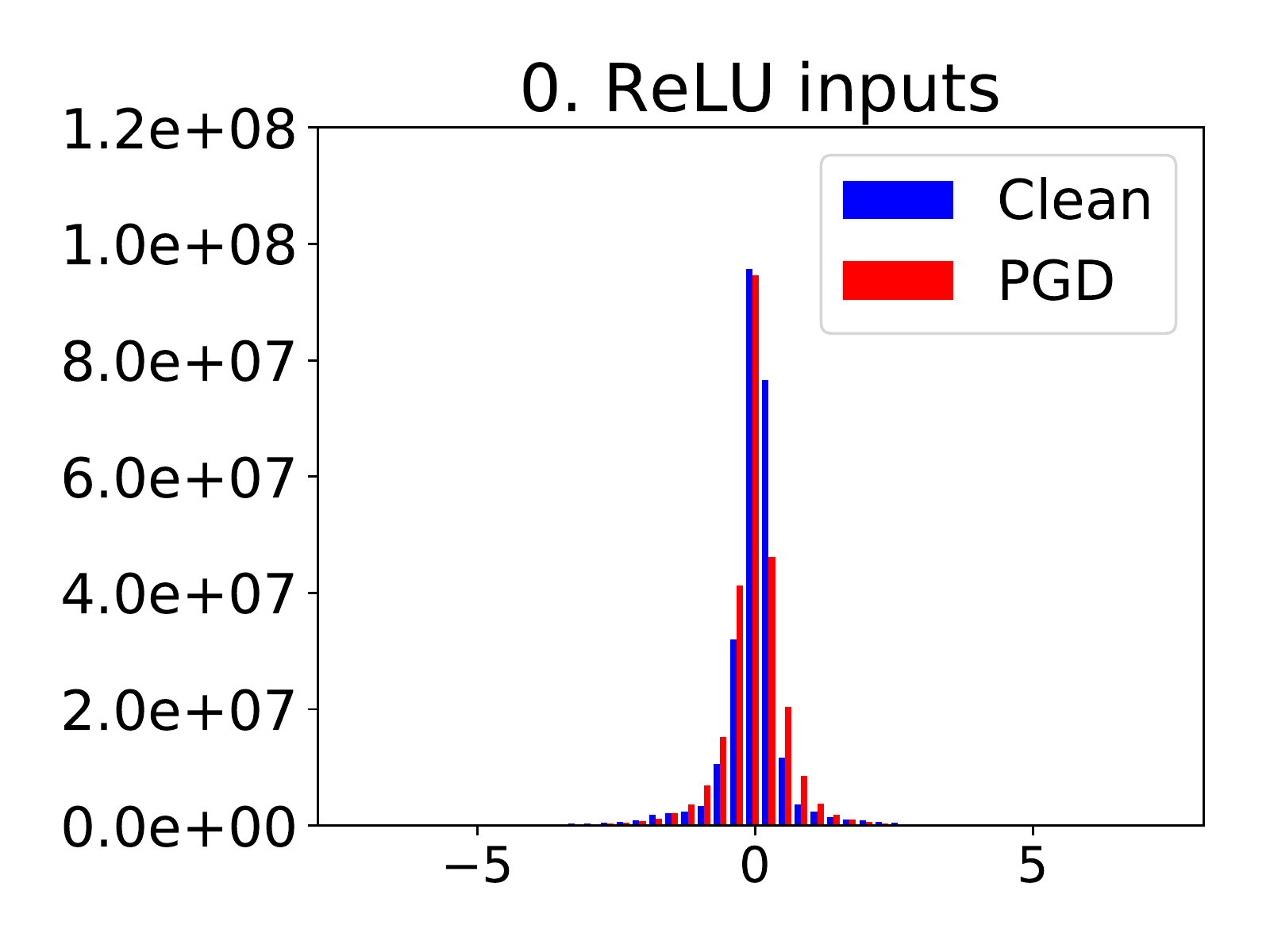} \hfill
    \includegraphics[width=.3\linewidth]{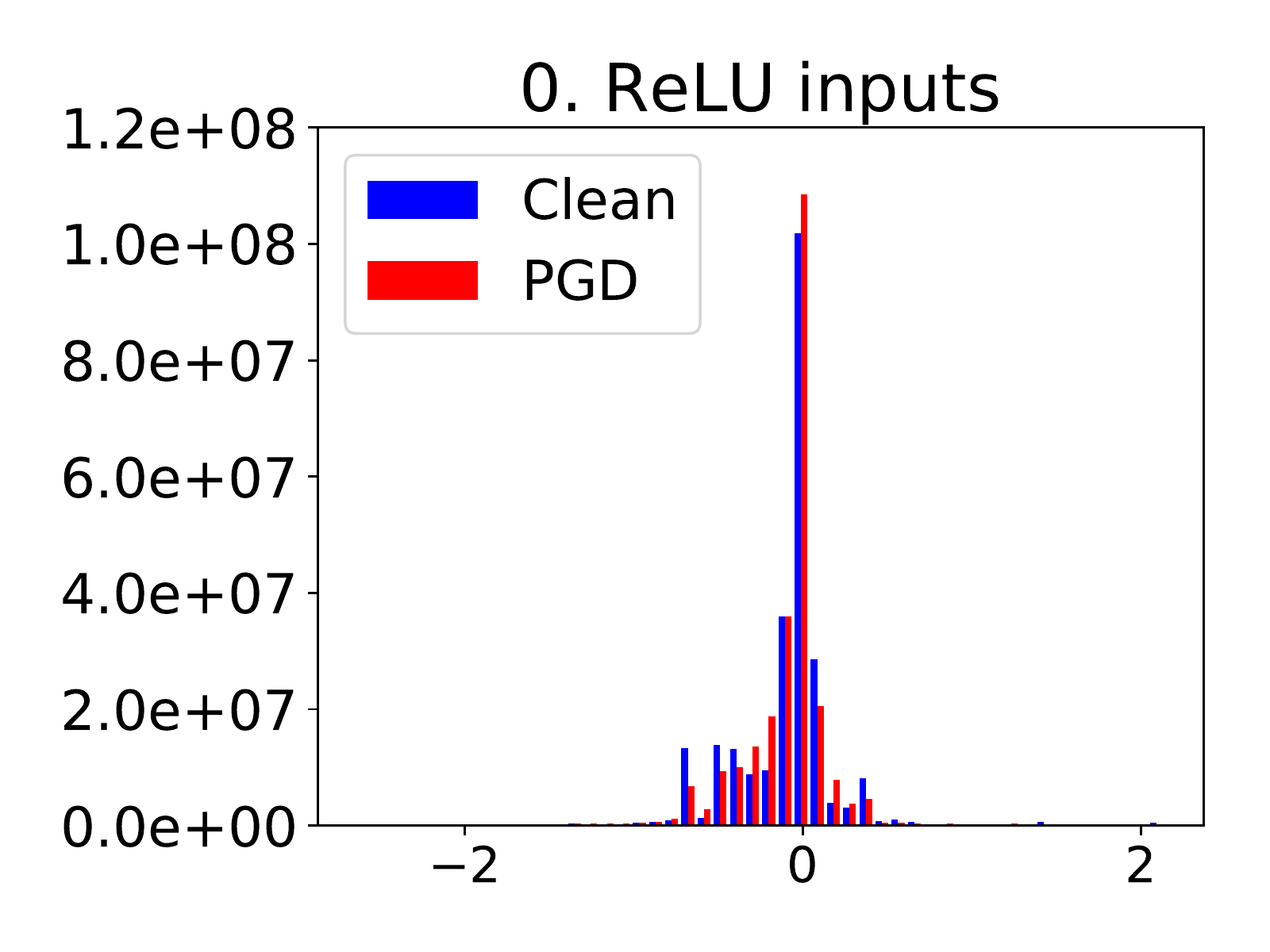} \hfill
    \includegraphics[width=.3\linewidth]{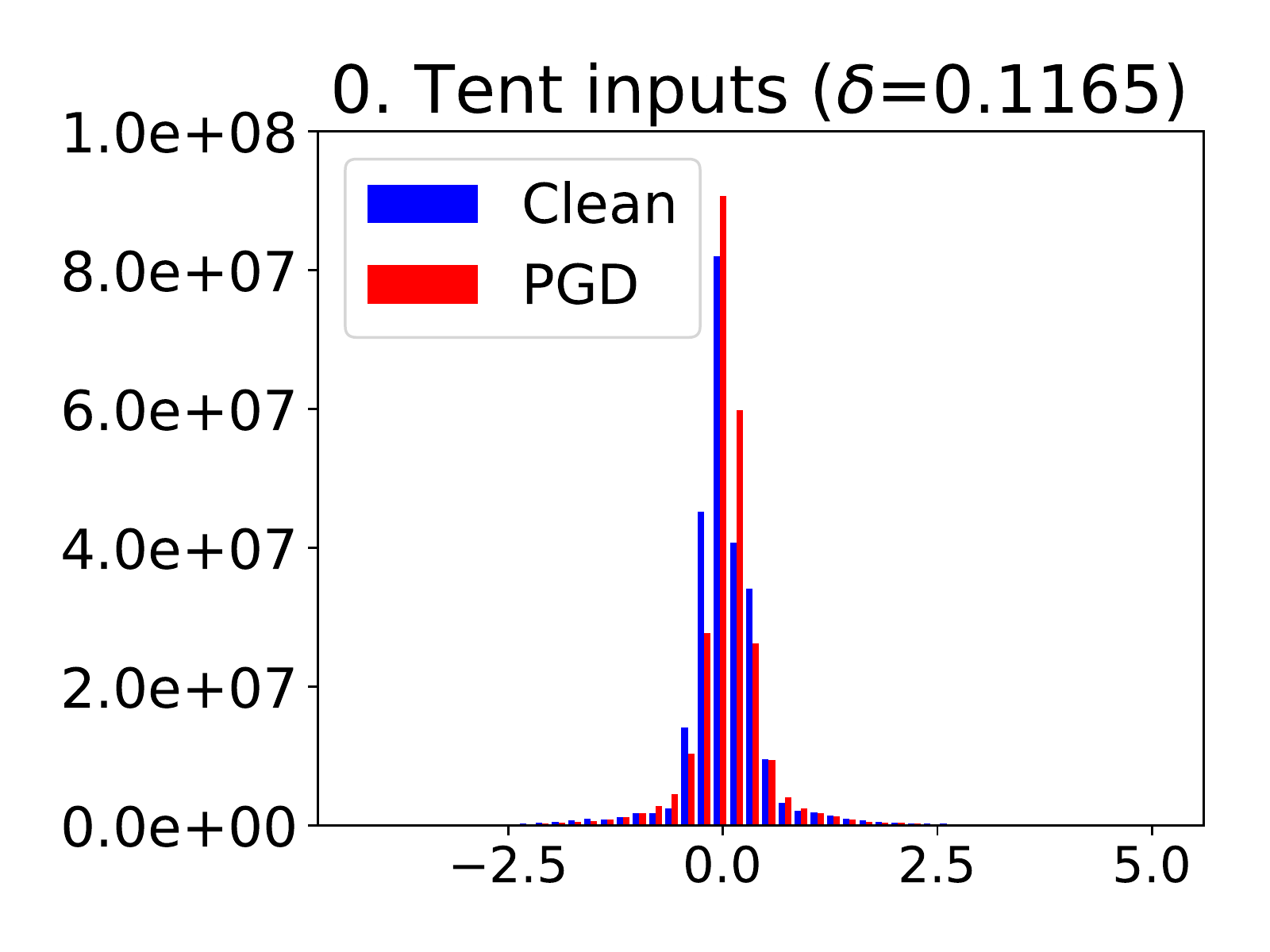} \\

  \caption{Activation Inputs: This figure shows the distributions of activation's inputs per layer for the 10k clean test samples (blue) and their adversarially perturbed counterparts via PGD (red) on three differently obtained networks: first, in the left, the regularly trained network denoted as MNIST-Net bn achieving 99.50\% accuracy on clean samples and 0.39\% on images perturbed via PGD, second, in the middle, the model obtained with adversarial training via PGD named as MNIST-Net adv having 99.36\% and 91.46\% accuracies on clean and PGD examples, and third, in the right column, the network containing tent activation functions denoted as MNIST-Net tent (0.12) delivering 99.20\% and 88.37\% accuracies on clean and PGD images, respectively. At the top of applicable plots, we list parameter $\delta$ of the particular tent.}
  \label{fig:activ_in}
\end{figure}

%% file: neurips_2019.bbl
\begin{thebibliography}{10}

\bibitem{szegedy2013intriguing}
Christian~J. Szegedy, Wojciech Zaremba, Ilya Sutskever, Joan Bruna, Dumitru
  Erhan, Ian Goodfellow, and Rob Fergus.
\newblock Intriguing properties of neural networks.
\newblock In {\em International Conference on Learning Representation (ICLR)},
  2014.

\bibitem{goodfellow2014explaining}
Ian~J. Goodfellow, Jonathon Shlens, and Christian Szegedy.
\newblock Explaining and harnessing adversarial examples.
\newblock In {\em International Conference on Learning Representation (ICLR)},
  2015.

\bibitem{moosavi2016deepfool}
Seyed-Mohsen Moosavi-Dezfooli, Alhussein Fawzi, and Pascal Frossard.
\newblock Deepfool: a simple and accurate method to fool deep neural networks.
\newblock In {\em Conference on Computer Vision and Pattern Recognition
  (CVPR)}. IEEE, 2016.

\bibitem{carlini2017towards}
Nicholas Carlini and David Wagner.
\newblock Towards evaluating the robustness of neural networks.
\newblock In {\em Symposium on Security and Privacy (SP)}. IEEE, 2017.

\bibitem{kurakin2017adversarial}
Alexey Kurakin, Ian~J. Goodfellow, and Samy Bengio.
\newblock Adversarial machine learning at scale.
\newblock In {\em International Conference on Learning Representation (ICLR)},
  2017.

\bibitem{madry2018towards}
Aleksander Madry, Aleksandar Makelov, Ludwig Schmidt, Dimitris Tsipras, and
  Adrian Vladu.
\newblock Towards deep learning models resistant to adversarial attacks.
\newblock In {\em International Conference on Learning Representation (ICLR)},
  2018.

\bibitem{tao2018attacks}
Guanhong Tao, Shiqing Ma, Yingqi Liu, and Xiangyu Zhang.
\newblock Attacks meet interpretability: Attribute-steered detection of
  adversarial samples.
\newblock In {\em Advances in Neural Information Processing Systems (NIPS)},
  2018.

\bibitem{openset14}
W.J. Scheirer, L.P. Jain, and T.E. Boult.
\newblock Probability models for open set recognition.
\newblock {\em Transactions on Pattern Analysis and Machine Intelligence
  (TPAMI)}, 36, 2014.

\bibitem{vyas2018out}
Apoorv Vyas, Nataraj Jammalamadaka, Xia Zhu, Dipankar Das, Bharat Kaul, and
  Theodore~L Willke.
\newblock Out-of-distribution detection using an ensemble of self supervised
  leave-out classifiers.
\newblock In {\em Proceedings of the European Conference on Computer Vision
  (ECCV)}, 2018.

\bibitem{lee2018simple}
Kimin Lee, Kibok Lee, Honglak Lee, and Jinwoo Shin.
\newblock A simple unified framework for detecting out-of-distribution samples
  and adversarial attacks.
\newblock In {\em Advances in Neural Information Processing Systems (NIPS)},
  2018.

\bibitem{shalev2018out}
Gabi Shalev, Yossi Adi, and Joseph Keshet.
\newblock Out-of-distribution detection using multiple semantic label
  representations.
\newblock In {\em Advances in Neural Information Processing Systems (NIPS)},
  2018.

\bibitem{racah2017extremeweather}
Evan Racah, Christopher Beckham, Tegan Maharaj, Samira~Ebrahimi Kahou,
  Mr~Prabhat, and Chris Pal.
\newblock Extremeweather: A large-scale climate dataset for semi-supervised
  detection, localization, and understanding of extreme weather events.
\newblock In {\em Advances in Neural Information Processing Systems (NIPS)},
  2017.

\bibitem{menon2018loss}
Aditya~Krishna Menon and Robert~C Williamson.
\newblock A loss framework for calibrated anomaly detection.
\newblock In {\em Proceedings of the 32nd International Conference on Neural
  Information Processing Systems}. Curran Associates Inc., 2018.

\bibitem{pidhorskyi2018generative}
Stanislav Pidhorskyi, Ranya Almohsen, and Gianfranco Doretto.
\newblock Generative probabilistic novelty detection with adversarial
  autoencoders.
\newblock In {\em Advances in Neural Information Processing Systems (NIPS)},
  2018.

\bibitem{gal2016dropout}
Yarin Gal, Jiri Hron, and Alex Kendall.
\newblock Concrete dropout.
\newblock In {\em Advances in Neural Information Processing Systems (NIPS)},
  2017.

\bibitem{lakshminarayanan2017simple}
Balaji Lakshminarayanan, Alexander Pritzel, and Charles Blundell.
\newblock Simple and scalable predictive uncertainty estimation using deep
  ensembles.
\newblock In {\em Advances in Neural Information Processing Systems (NIPS)},
  2017.

\bibitem{sensoy2018evidential}
Murat Sensoy, Lance Kaplan, and Melih Kandemir.
\newblock Evidential deep learning to quantify classification uncertainty.
\newblock In {\em Advances in Neural Information Processing Systems (NIPS)},
  2018.

\bibitem{malinin2018predictive}
Andrey Malinin and Mark Gales.
\newblock Predictive uncertainty estimation via prior networks.
\newblock In {\em Advances in Neural Information Processing Systems (NIPS)},
  2018.

\bibitem{rattani2015open}
Ajita Rattani, Walter~J Scheirer, and Arun Ross.
\newblock Open set fingerprint spoof detection across novel fabrication
  materials.
\newblock {\em IEEE Transactions on Information Forensics and Security (TIFS)},
  10(11), 2015.

\bibitem{zhang2016sparse}
He~Zhang and Vishal~M Patel.
\newblock Sparse representation-based open set recognition.
\newblock {\em IEEE Transactions on Pattern Analysis and Machine Intelligence
  (TPAMI)}, 39(8):1690--1696, 2016.

\bibitem{scherreik2016open}
Matthew~D Scherreik and Brian~D Rigling.
\newblock Open set recognition for automatic target classification with
  rejection.
\newblock {\em IEEE Transactions on Aerospace and Electronic Systems}, 52(2),
  2016.

\bibitem{shi2018odn}
Yemin Shi, Yaowei Wang, Yixiong Zou, Qingsheng Yuan, Yonghong Tian, and Yu~Shu.
\newblock Odn: Opening the deep network for open-set action recognition.
\newblock In {\em International Conference on Multimedia and Expo (ICME)}.
  IEEE, 2018.

\bibitem{dang2019open}
Sihang Dang, Zongjie Cao, Zongyong Cui, Yiming Pi, and Nengyuan Liu.
\newblock Open set incremental learning for automatic target recognition.
\newblock {\em Transactions on Geoscience and Remote Sensing}, 2019.

\bibitem{coletta2019combining}
Luiz~FS Coletta, Moacir Ponti, Eduardo~R Hruschka, Ayan Acharya, and Joydeep
  Ghosh.
\newblock Combining clustering and active learning for the detection and
  learning of new image classes.
\newblock {\em Neurocomputing}, 2019.

\bibitem{nair2010rectified}
Vinod Nair and Geoffrey~E Hinton.
\newblock Rectified linear units improve restricted {Boltzmann} machines.
\newblock In {\em International Conference on Machine Learning (ICML)}, 2010.

\bibitem{ioffe2015batch}
Sergey Ioffe and Christian Szegedy.
\newblock Batch normalization: Accelerating deep network training by reducing
  internal covariate shift.
\newblock In {\em International Conference on Machine Learning (ICML)}, 2015.

\bibitem{gulcehre2016noisy}
Caglar Gulcehre, Marcin Moczulski, Misha Denil, and Yoshua Bengio.
\newblock Noisy activation functions.
\newblock In {\em International Conference on Machine Learning (ICML)}, 2016.

\bibitem{maas2013rectifier}
Andrew~L Maas, Awni~Y Hannun, and Andrew~Y Ng.
\newblock Rectifier nonlinearities improve neural network acoustic models.
\newblock In {\em International Conference on Machine Learning (ICML)}, 2013.

\bibitem{he2015deep}
Kaiming He, Xiangyu Zhang, Shaoqing Ren, and Jian Sun.
\newblock Deep residual learning for image recognition.
\newblock In {\em Conference on Computer Vision and Pattern Recognition
  (CVPR)}. IEEE, 2016.

\bibitem{klambauer2017self}
G{\"u}nter Klambauer, Thomas Unterthiner, Andreas Mayr, and Sepp Hochreiter.
\newblock Self-normalizing neural networks.
\newblock In {\em Advances in Neural Information Processing Systems (NIPS)},
  2017.

\bibitem{karlik2011performance}
Bekir Karlik and A~Vehbi Olgac.
\newblock Performance analysis of various activation functions in generalized
  mlp architectures of neural networks.
\newblock {\em International Journal of Artificial Intelligence and Expert
  Systems}, 1(4), 2011.

\bibitem{bircanouglu2018comparison}
Cenk Bircano{\u{g}}lu and Nafiz Ar{\i}ca.
\newblock A comparison of activation functions in artificial neural networks.
\newblock In {\em 26th Signal Processing and Communications Applications
  Conference (SIU)}. IEEE, 2018.

\bibitem{wang2018efficient}
Shiqi Wang, Kexin Pei, Justin Whitehouse, Junfeng Yang, and Suman Jana.
\newblock Efficient formal safety analysis of neural networks.
\newblock In {\em Advances in Neural Information Processing Systems (NIPS)},
  2018.

\bibitem{lecun1998mnist}
Yann LeCun, Corinna Cortes, and Christopher~JC Burges.
\newblock The {MNIST} database of handwritten digits, 1998.

\bibitem{krizhevsky2009learning}
Alex Krizhevsky.
\newblock Learning multiple layers of features from tiny images, 2009.
\newblock Technical report. University of Toronto.

\bibitem{art2018}
Maria-Irina Nicolae, Mathieu Sinn, Minh~Ngoc Tran, Beat Buesser, Ambrish Rawat,
  Martin Wistuba, Valentina Zantedeschi, Nathalie Baracaldo, Bryant Chen, Heiko
  Ludwig, Ian Molloy, and Ben Edwards.
\newblock Adversarial robustness toolbox v0.10.0.
\newblock {\em CoRR}, 1807.01069, 2018.

\bibitem{papernot2016distillation}
Nicolas Papernot, Patrick McDaniel, Xi~Wu, Somesh Jha, and Ananthram Swami.
\newblock Distillation as a defense to adversarial perturbations against deep
  neural networks.
\newblock In {\em Symposium on Security and Privacy (SP)}. IEEE, 2016.

\bibitem{zagoruyko2016wide}
Sergey Zagoruyko and Nikos Komodakis.
\newblock Wide residual networks.
\newblock In {\em British Machine Vision Conference (BMVC)}, 2016.

\bibitem{sabour2016adversarial}
Sara Sabour, Yanshuai Cao, Fartash Faghri, and David~J. Fleet.
\newblock Adversarial manipulation of deep representations.
\newblock In {\em International Conference on Learning Representation (ICLR)},
  2016.

\end{thebibliography}
